\documentclass{article}

\usepackage{amssymb,amsmath,amsthm,bbm}
\usepackage{verbatim,float,url,dsfont}
\usepackage{graphicx,subcaption,psfrag}
\usepackage{algorithm}
\usepackage[noend]{algorithmic}
\usepackage{mathtools,enumitem}
\usepackage{multirow}
\usepackage{ragged2e}
\usepackage{xr-hyper}
\usepackage{caption}
\usepackage{array}

\usepackage[colorlinks=true,citecolor=blue,urlcolor=blue,linkcolor=blue]{hyperref}
\usepackage[margin=1in]{geometry}

\usepackage[utf8]{inputenc} % allow utf-8 input
\usepackage[T1]{fontenc}    % use 8-bit T1 fonts
\usepackage{booktabs}       % professional-quality tables
\usepackage{nicefrac}         % compact symbols for 1/2, etc.
\usepackage{microtype}      % microtypography

\usepackage[dvipsnames]{xcolor} % Colored text
\usepackage{pifont} % For symbols like checkmarks, see http://ctan.org/pkg/pifont
\usepackage{etoc} % For local table of contents within sections "\localtableofcontents"
\usepackage{tikz} % for diagrams
\usepackage{pgfplots}  % for diagrams
\pgfplotsset{compat=1.16}  % for diagrams 

\ifdefined\TimesFont 
\usepackage{times} % use times font
\fi

\ifdefined\ParSkip 
\usepackage{parskip} % use par skip
\fi

\newtheorem{lemma}{Lemma}
\newtheorem{corollary}{Corollary}
\newtheorem{proposition}{Proposition}
\theoremstyle{definition}
\newtheorem{remark}{Remark}
\newtheorem{definition}{Definition}

\newtheorem*{assumption*}{\assumptionnumber}
\providecommand{\assumptionnumber}{}


\makeatletter
\newcommand*\rel@kern[1]{\kern#1\dimexpr\macc@kerna}
\newcommand*\widebar[1]{%
  \begingroup
  \def\mathaccent##1##2{%
    \rel@kern{0.8}%
    \overline{\rel@kern{-0.8}\macc@nucleus\rel@kern{0.2}}%
    \rel@kern{-0.2}%
  }%
  \macc@depth\@ne
  \let\math@bgroup\@empty \let\math@egroup\macc@set@skewchar
  \mathsurround\z@ \frozen@everymath{\mathgroup\macc@group\relax}%
  \macc@set@skewchar\relax
  \let\mathaccentV\macc@nested@a
  \macc@nested@a\relax111{#1}%
  \endgroup
}
\makeatother

\DeclareMathOperator*{\argmin}{argmin}

\DeclareMathOperator{\nul}{null}

\DeclareMathOperator{\sign}{sign}

\DeclareMathOperator{\aff}{aff}

\def\R{\mathbb{R}}

\def\T{\mathsf{T}}

\def\th{^{\textnormal{th}}}
\def\Id{\mathrm{Id}}

\def\cF{\mathcal{F}}

\def\cX{\mathcal{X}}
\def\cY{\mathcal{Y}}

\usepackage{centernot} % Center the "\not" over symbols 
\def\notimplies{\centernot\implies} 
\def\Regret{\mathrm{Regret}}

\newcommand{\jupyter}[1]{\href[pdfnewwindow=true]{#1}{\smash{\begingroup
\setbox0=\hbox{\includegraphics[height=1.5em]{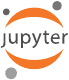}}%
\parbox{\wd0}{\box0}\endgroup}}}

\usepackage{tikz}

\title{Gradient Equilibrium in Online Learning: \\ Theory and Applications}    
\author{Anastasios N.\ Angelopoulos$^*$ \and
Michael I.\ Jordan$^{*\dagger}$ \and
Ryan J.\ Tibshirani$^*$} 
\date{$^*$University of California, Berkeley \qquad $^\dagger$Inria, Paris} 

\begin{document}
\maketitle

\begin{abstract}
We present a new perspective on online learning that we refer to as
\emph{gradient equilibrium}: a sequence of iterates achieves gradient
equilibrium if the average of gradients of losses along the sequence converges
to zero. In general, this condition is not implied by, nor implies, sublinear
regret. It turns out that gradient equilibrium is achievable by standard online
learning methods such as gradient descent and mirror descent with 
\emph{constant} step sizes (rather than decaying step sizes, as is usually   
required for no regret). Further, as we show through examples, gradient
equilibrium translates into an interpretable and meaningful property in online
prediction problems spanning regression, classification, quantile estimation,
and others. Notably, we show that the gradient equilibrium framework can be used
to develop a debiasing scheme for black-box predictions under arbitrary 
distribution shift, based on simple post hoc online descent updates.  
We also show that post hoc gradient updates can be used to calibrate  
predicted quantiles under distribution shift, and that the framework leads to
unbiased Elo scores for pairwise preference prediction.             
\end{abstract}

\tableofcontents
\clearpage

\section{Introduction}

Online learning is a powerful paradigm for the analysis of sequential data, with
applications in diverse areas such as forecasting, calibration, and control.
Distinct from classical methods in statistics for analyzing time series and
stochastic processes, the online learning framework makes no stochastic or
generative assumptions about the sequence of data points. That is, the
guarantees from typical online learning analyses apply to an individual sequence
of data points that can be entirely arbitrary. As such, online learning offers
algorithmic and analytic strategies to cope with arbitrary distribution drift
and adversarial behavior---issues which are central in real deployments of
machine learning systems. 

Most online learning analyses focus on \emph{regret} as a key metric---given an
algorithm that is presented with one data point at a time, and must produce
predictions accordingly, how small is the total loss relative to an oracle who
can wait until the end, seeing the entire batch of data, and outputting the
single best prediction? Although regret provides a theoretical handle on the
challenging problem of optimizing sums of losses with respect to arbitrary
sequences, it is not always apparent how regret aligns with various goals of
learning. It is often the case that certain properties of a predictive algorithm
are of interest, and not just its loss relative to a fixed reference; for
example, we might seek to understand its bias, coverage, or other statistical
properties. While extensions of regret may be able to target such behaviors,
they can become complex.

In the current paper, we offer a simple, alternative perspective on online
learning which directly targets various properties that we might want an online
algorithm to exhibit. Such properties tend to have a statistical flavor, but in
the spirit of online learning, we achieve them without making any stochastic
assumptions. For example, our framework can guarantee a set of predictions is
unbiased over a completely arbitrary sequence of data points; here, bias refers
to the difference between the average prediction and average response over the
sequence, not in expectation. We refer to our online learning framework as
\emph{gradient equilibrium}.

Next we introduce the basic idea of gradient equilibrium, and we
provide three examples of applications that highlight the simplicity and
scalability of gradient equilibrium. In Section \ref{sec:regret}, we return to a
discussion of regret, demonstrating that our framework is not the same as
regret---sublinear regret does not subsume gradient equilibrium nor vice
versa. In Section \ref{sec:grad_descent}, we provide comprehensive theory 
on gradient equilibrium in online gradient descent. Section
\ref{sec:regularization} extends the analysis to incorporate
regularization. Although our main focus is constant step sizes, in Section
\ref{sec:arbitrary_steps} we treat the case of arbitrary step sizes. Finally, in
Section \ref{sec:experiments}, we give several examples of applications, and 
in Section \ref{sec:discussion}, our conclusions. Throughout, each experimental
figure has a clickable Jupyter logo
\jupyter{https://github.com/aangelopoulos/gradient-equilibrium} in its caption,
which links to the notebook for reproducing that figure.  

\subsection{Gradient equilibrium}

To fix notation, let $\ell_t$, $t = 1,2,3,\dots$ be a sequence of loss
functions, with each $\ell_t : \R^d \to (-\infty, \infty]$ assumed to be finite
and subdifferentiable, in a generalized sense made precise below, on a common
domain $D \subseteq \R^d$. Each loss $\ell_t$ will typically depend on a data
point $(x_t, y_t) \in \cX \times \cY$, e.g., we could have
\smash{$\ell_t(\theta) = \frac{1}{2} (y_t - x_t^\T \theta)^2$} with $\cX =
\R^d$ and $\cY = \R$, though we suppress this dependence for notational
simplicity.  

Let $\theta_t$, $t = 1,2,3,\dots$ denote a sequence of iterates produced by some  
online optimization scheme applied to $\ell_t$, $t = 1,2,3,\dots$, such as
gradient descent. Next we define our main condition of study. 

\begin{definition}
A sequence of iterates $\theta_t \in D \subseteq \R^d$, $t = 1,2,3,\dots$ satisfies
\emph{gradient equilibrium} with respect to a sequence of loss functions
$\ell_t$, $t = 1,2,3,\dots$ provided that 
\begin{equation}
\label{eq:grad_eq}
\frac{1}{T} \sum_{t=1}^T g_t(\theta_t) \to 0, \quad \text{as $T \to \infty$},
\end{equation}
where each $g_t(\theta_t)$ is a subgradient of $\ell_t$ at $\theta_t$.   
\end{definition}

In this definition and in general, we use $g_t(\theta) \in \R^d$ to
denote a generalized subgradient of $\ell_t$ at $\theta \in D$ as defined in 
Appendix \ref{app:gen_subgrad}. For our purposes in what follows, the precise
definition of a generalized subgradient is not crucial, and the most important
point to convey is that this generalization fluidly encapsulates both
subgradients of convex functions, and gradients of differentiable (and possibly
nonconvex) functions. We will often call  $g_t(\theta)$ a ``gradient,'' for
simplicity, even though it technically denotes a generalized subgradient.  

Gradient equilibrium, simply put, says that the average of gradients along the
iterate sequence tends to zero. This condition turns out to be interpretable in
various problem settings: for example, for squared losses, it translates into a
sequential notion of unbiasedness; for quantile losses, it reduces to one-sided
coverage; for generalized linear model (GLM) losses, it becomes uncorrelatedness
of the GLM residuals and the covariates. Generally, one can understand these
conditions as sequential analogs of the first-order optimality condition in
M-estimation problems. Table \ref{tab:grad_eq} gives a summary. In Section
\ref{sec:examples}, we will examine these and other examples in more
detail, also including a derivation of the equilibrium conditions. 

\renewcommand{\arraystretch}{1.4}
\begin{table}[H]
\centering
\begin{tabular}{ccc}
\toprule
Loss & Equilibrium condition & Interpretation \\
\midrule
$\ell_t(\theta) = \frac{1}{2} (y_t - \theta)^2$ & 
$\frac{1}{T} \sum_{t=1}^T \theta_t \asymp \frac{1}{T} \sum_{t=1}^T y_t$ &    
$\theta_t$ is unbiased for $y_t$ \\  
\midrule
$\ell_t(\theta) = \rho_\tau (y_t - \theta)$ &  
$\frac{1}{T} \sum_{t=1}^T 1\{ y_t \leq \theta_t \} \asymp \tau$ &
$\theta_t$ lies above $y_t$ with frequency $\tau$ \\
\midrule
$\ell_t(\theta) = -y_t x_t^\T \theta + \psi(x_t^\T \theta)$ &  
$\frac{1}{T} \sum_{t=1}^T (\psi'(x_t^\T \theta_t) - y_t) x_t \asymp 0$ &  
\parbox[c]{2in}{\centering\smallskip\smallskip
GLM residuals from $\theta_t$ are \\ uncorrelated with features
\smallskip\smallskip} \\  
\bottomrule
\end{tabular}
\captionsetup{singlelinecheck=off}
\caption[.]{\it Examples of loss functions and gradient equilibrium
  conditions. For sequences $a_T$ and $b_T$, we write $a_T \asymp b_T$ to mean
  $a_T - b_T \to 0$ as $T \to \infty$. In the second row, $\rho_\tau$ denotes
  the quantile loss at level $\tau \in [0,1]$, i.e., $\rho_\tau(u) = \tau |u|$
  for $u \geq 0$ and $\rho_\tau(u) = (1-\tau) |u|$ for $u < 0$. In the third
  row, $\psi$ is the cumulant generating function for the GLM (e.g.,
  \smash{$\psi(u) = \frac{1}{2} u^2$} for linear regression, $\psi(u) = \log(1 +
  e^u)$ for logistic regression, and $\psi(u) = e^u$ for Poisson regression),
  and $\psi'$ is its derivative.}
\label{tab:grad_eq}
\end{table}
\renewcommand{\arraystretch}{1}

\begin{figure}[p]
\centering
\includegraphics[width=\textwidth]{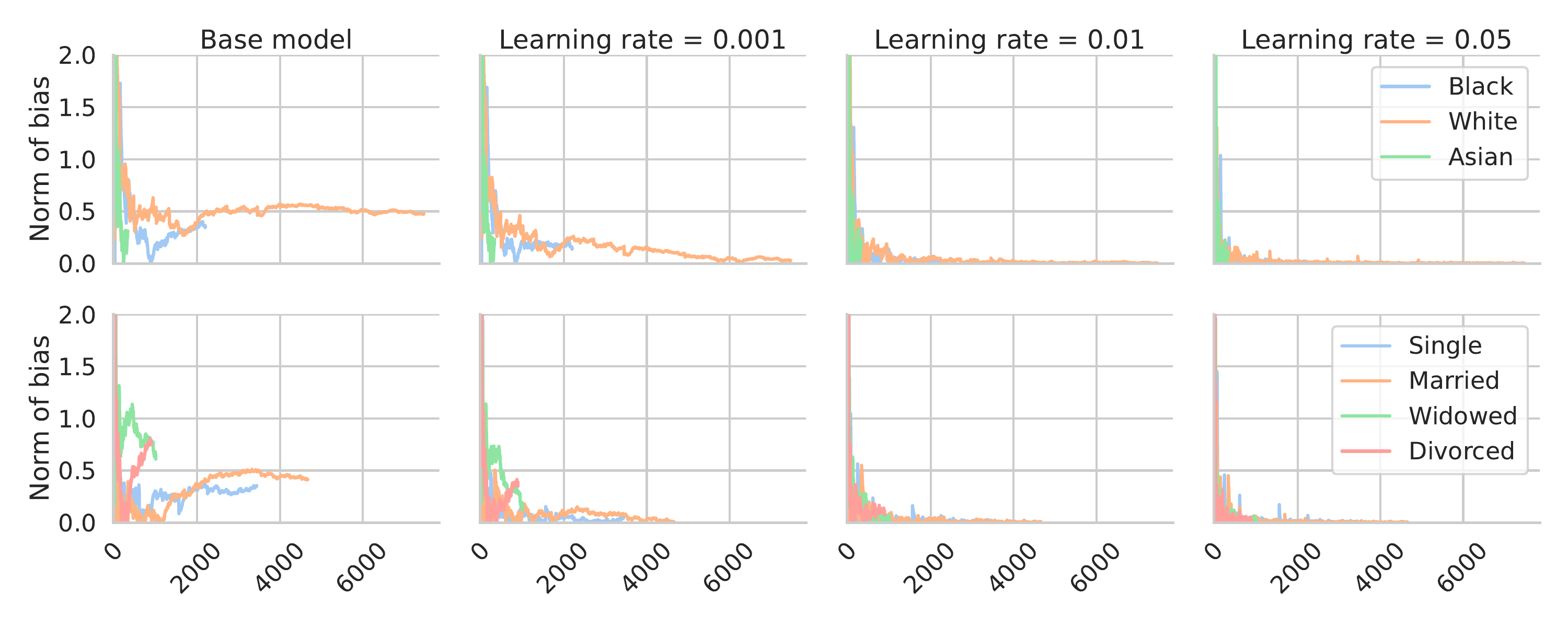} 
\caption{\it Multigroup debiasing results on the MIMIC dataset, on predicting
  length-of-stay of patients in a hospital system in Boston, Massachusetts. We
  train an XGBoost model on a large number of features from this dataset and run
  our multigroup debiasing procedure with respect to ethnicity (top row) and
  marital status (bottom row), with each column showing a different learning
  rate. Gradient equilibrium (third row of Table \ref{tab:grad_eq} where the
  features are group indicators) for this problem says that we  achieve zero
  bias for each ethnicity and marital status, in the long run. 
  \jupyter{https://github.com/aangelopoulos/gradient-equilibrium/blob/main/mimic_stay/multigroup.ipynb}}
\label{fig:multigroup_debiasing_mimic}

\bigskip\bigskip
\includegraphics[width=\linewidth]{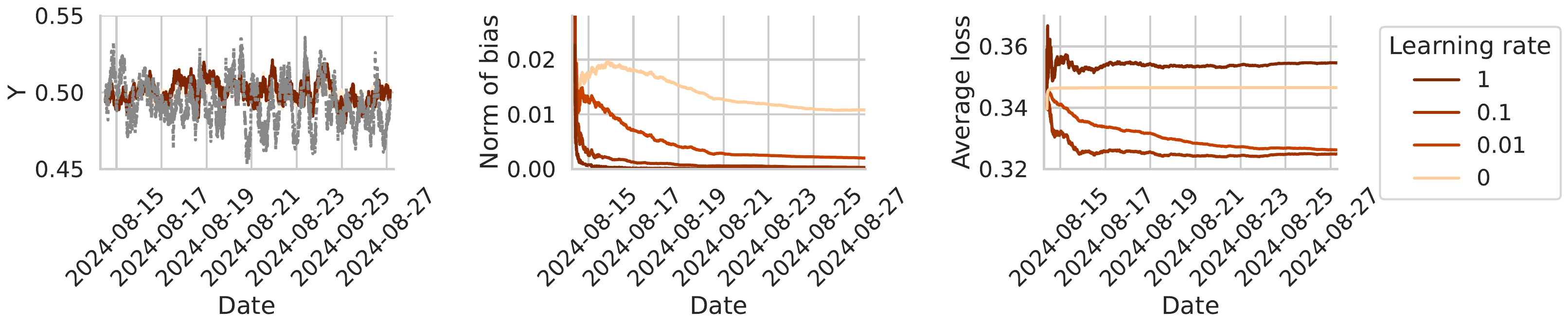}
\includegraphics[width=\linewidth]{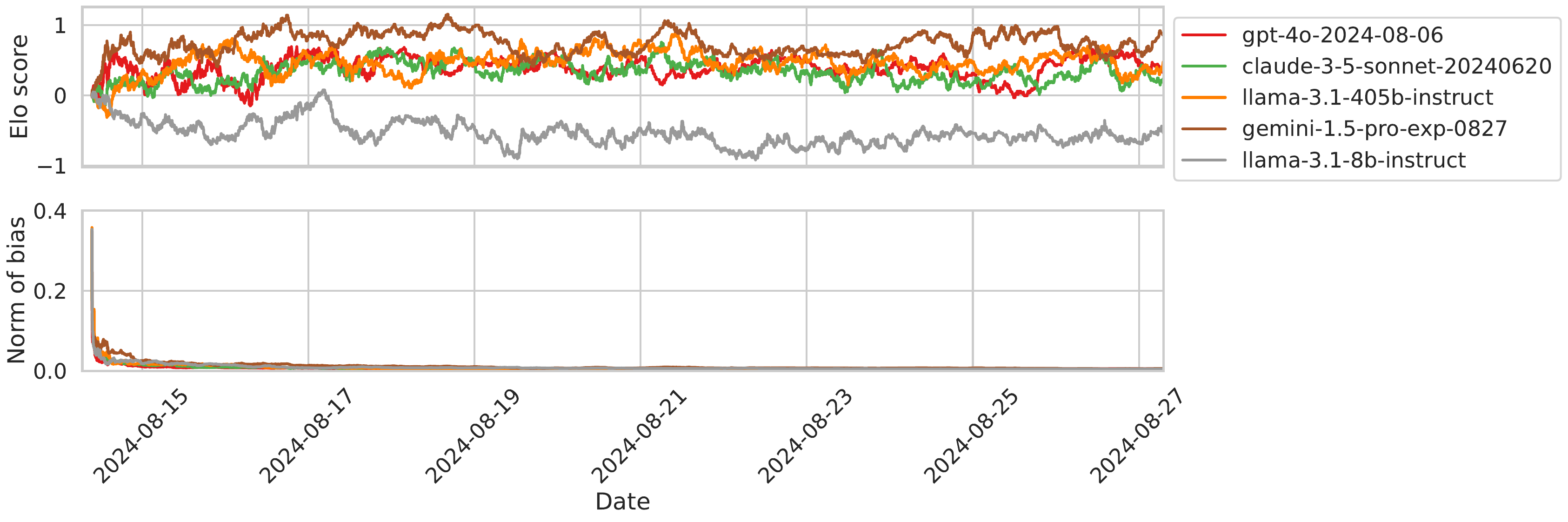}
\caption{\it Results on the Chatbot Arena dataset, comprised of predictions of
  human preferences between pairs of large language models (LLMs). The top row
  shows statistics of the predictions over time, with different learning
  rates. The left plot shows a rolling average of the predicted win rate along
  with a rolling average of the ground-truth labels in gray over time. The
  middle plot gives the absolute bias of the predictions over the sequence. On
  the right is the average binary cross-entropy loss of predictions. The bottom
  two plots show the sequence of Elo scores, and the bias of the Elo win-rate
  predictions on a per-model level. Gradient equilibrium for this problem says
  that we achieve zero bias per model, in the long run.  
  \jupyter{https://github.com/aangelopoulos/gradient-equilibrium/blob/main/arena/online_elo.ipynb}}   
\label{fig:elo_scores_chatbot}
\end{figure}

To further motivate our study, we briefly review a few applications of 
gradient equilibrium conditions in machine learning. An important application is
the debiasing of a black-box prediction model under arbitrary distribution
shift, a problem of significant interest across a wide range of machine learning
deployments. To give a flavor of how this works, and how simple it is, suppose
that $f_t(x_t)$ is a prediction of a response $y_t$ at $t$. We form an adjusted
prediction via 
\[
\tilde{f}_t(x_t) = f_t(x_t) + \theta_t, \quad \text{where} \quad
\theta_t = \theta_{t-1} + \eta (y_{t-1} - f_{t-1}(x_{t-1})).
\]
Here $\eta > 0$ is a step size (i.e., learning rate). The parameter $\theta_t$
can be seen as the result of an online gradient descent update applied to the 
loss \smash{$\ell_{t-1}(\theta) = \frac{1}{2} (y_{t-1} - f_{t-1}(x_{t-1}) -
  \theta)^2$}. Gradient equilibrium translates into
\[
\frac{1}{T} \sum_{t=1}^T \tilde{f}_t(x_t) - \frac{1}{T} \sum_{t=1}^T y_t \to 0,
\quad \text{as $T \to \infty$},
\]
the unbiasedness of the adjusted predictions \smash{$\tilde{f}_t(x_t)$}, $t =
1,2,3,\dots$ along the sequence. The same idea can be applied to carry out
debiasing based on features, which may or may not be a part of the feature
$x_t$ used by the original predictor $f_t$. For example, we can achieve
multigroup (i.e., groupwise) debiasing guarantees by setting $\ell_{t-1}$ to be
the loss which regresses the residual $y_{t-1} - f_{t-1}(x_{t-1})$ onto a
feature that indicates group membership. Figure
\ref{fig:multigroup_debiasing_mimic} showcases multigroup debiasing on a medical
length-of-stay prediction dataset. This setting is studied formally in Section
\ref{sec:examples} and empirically in Sections \ref{sec:simple_debiasing} and 
\ref{sec:multigroup_debiasing}. 

A second example we highlight is the learning of Elo scores for pairwise
preference prediction. Elo scores are a widely-used rating system for
competitive two-player games, like chess. Given $M$ competitors, we first
initialize a coefficient vector of all zeros, $\theta_1 = 0 \in \R^d$. During a
sequence $t = 1,2,3,\dots$ of battles between two distinct competitors $a_t,b_t
\in \{1,\dots,M\}$, we define $y_t = 1$ if $b_t$ wins and $y_t = 0$
otherwise. After the battle at each $t$, we perform the following update  
(abbreviating $a = a_t$ and $b = b_t$): 
\begin{align*}
\theta_{ta} &= \theta_{t-1,a} - \eta(y_t - p_t) \\
\theta_{tb} &= \theta_{t-1,b} - \eta(p_t - y_t).
\end{align*}
Here $\eta > 0$ is a learning rate, and $p_t = \sigma(\theta_{t-1,b} -
\theta_{t-1,a})$ where $\sigma(x) = e^x/(1+e^x)$ is the sigmoid function. All  
other coefficients are carried forward, $\theta_{tm} = \theta_{t-1,m}$ for $m
\not= a,b$. Gradient equilibrum allows us to prove an unbiasedness guarantee for
this algorithm: namely, for each player $m \in \{1,\dots,M\}$,    
\[
\frac{1}{|I_m|}\sum_{t \in I_m} p_t - \frac{1}{|I_m|}\sum_{t \in I_m} y_t \to 0, 
\quad \text{as $T \to \infty$, where} \; I_m = \{ t \leq T : \text{$a_t = m$ or 
      $b_t = m$} \}.
\]
That is, for every player, our win-rate predictions are unbiased in the long
run, over the sequence. Figure \ref{fig:elo_scores_chatbot} showcases this
capability in the domain of pairwise human preference evaluations of large  
language models (LLMs). We return to a more detailed discussion of this problem
in Section \ref{sec:pairwise_preference}.   

A third and last application is quantile calibration: we can use simple post hoc 
online gradient updates (analogous to those described above for debiasing) to
adjust predicted quantiles, so that they achieve finite-sample coverage
guarantees under arbitrary distribution shift. This was already studied in
\cite{gibbs2021adaptive, angelopoulos2023conformal}, albeit not under the
framework of gradient equilibrium. We show formally how this line of work
relates to gradient equilibrium in Section \ref{sec:examples}, and revisit it
through empirical examples in Section \ref{sec:quantile_tracking}.

\subsection{Comparison to regret}
\label{sec:regret}

Regret (including its various generalizations) is the most common metric used to  
analyze algorithms in the literature on online learning. See the related work
discussion in Section \ref{sec:related_work}. The \emph{regret} associated with
the sequence $\theta_t$, $t = 1,2,3,\dots$ at iteration $T$ is defined by:     
\[
\Regret_T = \sum_{t=1}^T \ell_t(\theta_t) - \inf_\theta \, \sum_{t=1}^T
\ell_t(\theta).     
\]
This measures the difference in loss incurred by the iterate sequence to the
loss of the best fixed parameter choice ``in hindsight.''  In this paper, we make
use of the following definition.

\begin{definition}
A sequence of iterates $\theta_t$, $t = 1,2,3,\dots$ satisfies \emph{no regret}
with respect to a sequence of loss functions $\ell_t$, $t = 1,2,3,\dots$
provided that 
\begin{equation}
\label{eq:no_regret}
\frac{\Regret_T }{T} = \frac{1}{T} \sum_{t=1}^T \ell_t(\theta_t) - \inf_\theta
\, \frac{1}{T} \sum_{t=1}^T \ell_t(\theta) \to 0, \quad \text{as $T \to
  \infty$}. 
\end{equation}
\end{definition}

The condition in \eqref{eq:no_regret} is usually called ``sublinear regret'' in
the online learning literature, as much of this literature is focused on 
quantifying the precise growth rate of \smash{$\Regret_T$} as a function of $T$
for certain online algorithms under certain conditions on the sequence $\ell_t$
(e.g., \smash{$\sqrt{T}$} for gradient descent on convex functions, or $\log{T}$
for strongly convex functions, and so on). We refer to condition
\eqref{eq:no_regret} as ``no regret'' for simplicity.    

\begin{figure}[b!]
\centering
\includegraphics[width=\textwidth]{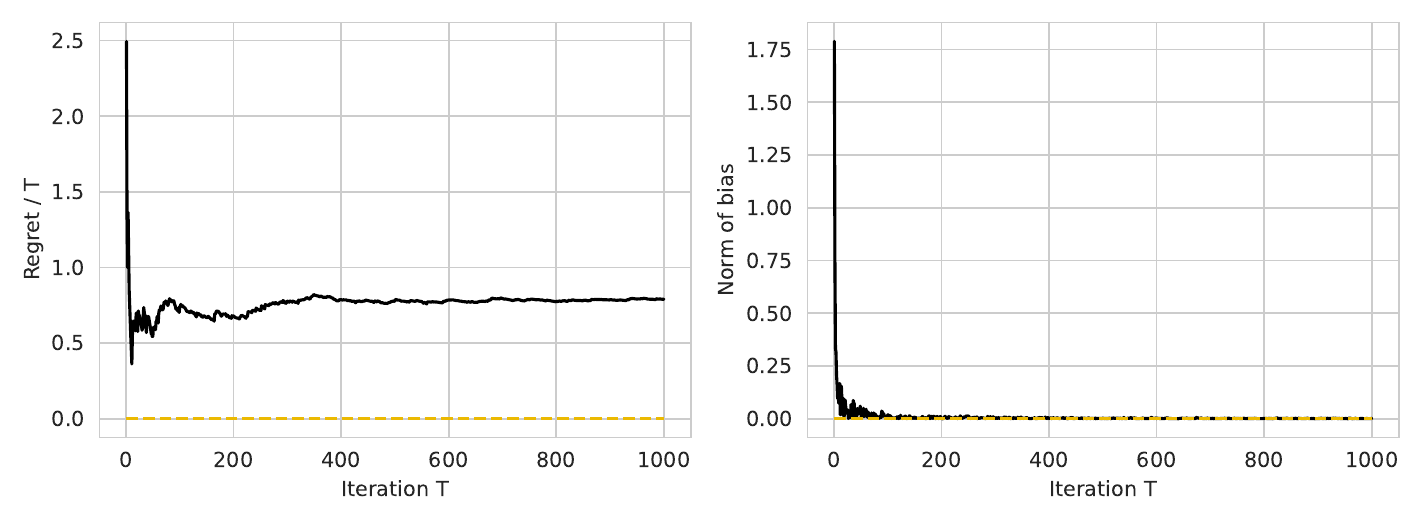}
\caption{\it Regret and bias for gradient descent on squared losses, with
  constant step sizes.}   
\label{fig:regret_and_bias}
\end{figure}

How do regret and gradient equilibrium compare? To glean some insight, 
Figure \ref{fig:regret_and_bias} displays the result of a simple experiment with 
squared losses, $\ell_t(\theta) = \frac{1}{2} (y_t - \theta)^2$, $t =
1,2,3,\dots$, where the data $y_t$, $t = 1,2,3,\dots$ are i.i.d.\ standard
Gaussian samples. We use gradient descent, initialized at $\theta_1 = 0$, to  
define the sequence $\theta_t$, $t = 2,3,\dots$. We use a constant step
size, $\eta = 0.2$, which is not typical for online optimization, and does not 
imply a no-regret guarantee in general; for strongly convex functions like
the squared losses in the current setting, a decaying step size schedule such as
$\eta_t = 1/(\mu t)$ would be typical (where $\mu$ is the strong convexity
constant), since this would lead to a sharp no-regret guarantee.      

The iterate sequence from gradient descent here does not appear to have no
regret, as we can see in the left panel of Figure \ref{fig:regret_and_bias} that
$\Regret_T / T$ does not vanish as $T \to \infty$. However, something
interesting happens to the magnitude of the average bias,
\smash{$|\frac{1}{T} \sum_{t=1}^T (\theta_t-y_t)|$}, plotted in the right panel
of the figure: this is driven to zero as $T$ grows. What this shows, in fact, is
that along the gradient descent sequence the average of the gradients 
\smash{$\frac{1}{T} \sum_{t=1}^T g_t(\theta_t)$} can be made small even if the
regret is large.  

It is not hard to see that, in general, gradient equilibrium
\eqref{eq:grad_eq} does not imply no regret \eqref{eq:no_regret}. For
example, the iterates $\theta_t$, $t = 1,2,3,\dots$ could ``bounce around'' the   
minimizer of \smash{$\bar\ell_T = \frac{1}{T} \sum_{t=1}^T \ell_t$} in such a
way that the average gradient tends to zero but the average loss is not minimal;  
just as in Figure \ref{fig:lrs_notimplies_nr} below. Further, if
\smash{$\bar\ell_T$} is nonconvex, then it could have stationary points (with
zero gradient) that are strict local minima or strict saddle points, and
$\theta_t$ could be bouncing around (or even converging to) such points, and in
doing so it could satisfy gradient equilibrium without attaining anywhere close
to minimal average loss. Conversely, no regret does not imply gradient
equilibrium in general, as we study next through examples. 

\subsubsection{Two illustrative examples}

To help elucidate the relationship between no regret (NR) and gradient
equilibrium (GEQ), we work through two illustrative examples. 

\paragraph{Example 1: absolute value losses.}

For the first example, let $\ell_t(\theta) = |\theta|$, for all $t =
1,2,3,\dots$. To see that NR does not imply GEQ, consider any sequence 
such that $\theta_t > 0$ for all $t$, and $\theta_t \to 0$. Then we have
$\Regret_T / T \to 0$ but \smash{$\frac{1}{T} \sum_{t=1}^T g_t(\theta_t) = 1$}
for all $T$, thus GEQ does not hold. This is depicted in Figure
\ref{fig:nr_notimplies_lrs}.     

For the opposite direction, to see that GEQ does not imply NR, consider any
sequence with $\theta_t > 0$ for odd $t$, and $\theta_t < 0$ for even $t$. (This 
behavior of ``bouncing around'' the minimum was discussed earlier, and the
current setting provides a concrete picture.) Then \smash{$\frac{1}{T}
|\sum_{t=1}^T g_t(\theta_t)| \leq \frac{1}{T} \to 0$}, but as long as
$|\theta_t|$ remains bounded away from zero, NR does not hold. This is depicted
in Figure \ref{fig:lrs_notimplies_nr}. 

\begin{figure}[b!]
\bigskip % Spacing hack, fig lands too close to an equation 
\centering
\begin{subfigure}{0.49\textwidth}
\includegraphics[width=\linewidth]{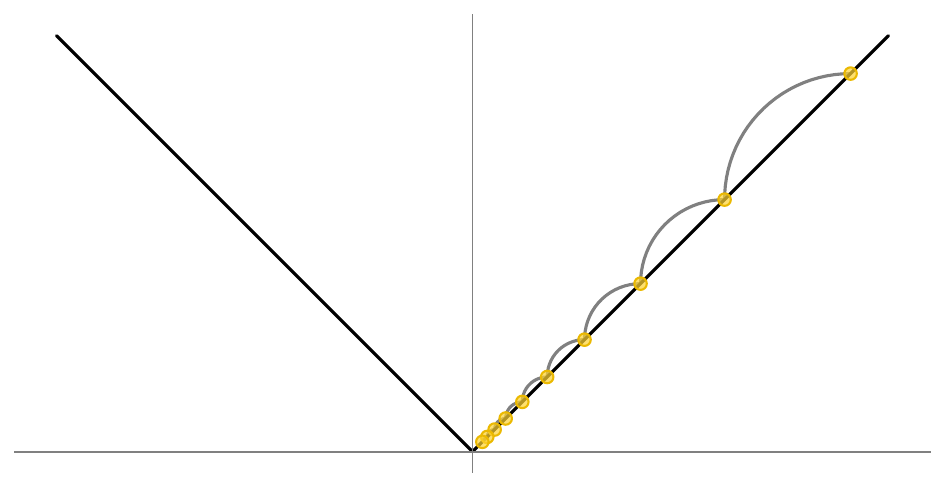}
\caption{NR $\notimplies$ GEQ}
\label{fig:nr_notimplies_lrs}
\end{subfigure} \hfill
\begin{subfigure}{0.49\textwidth}
\includegraphics[width=\linewidth]{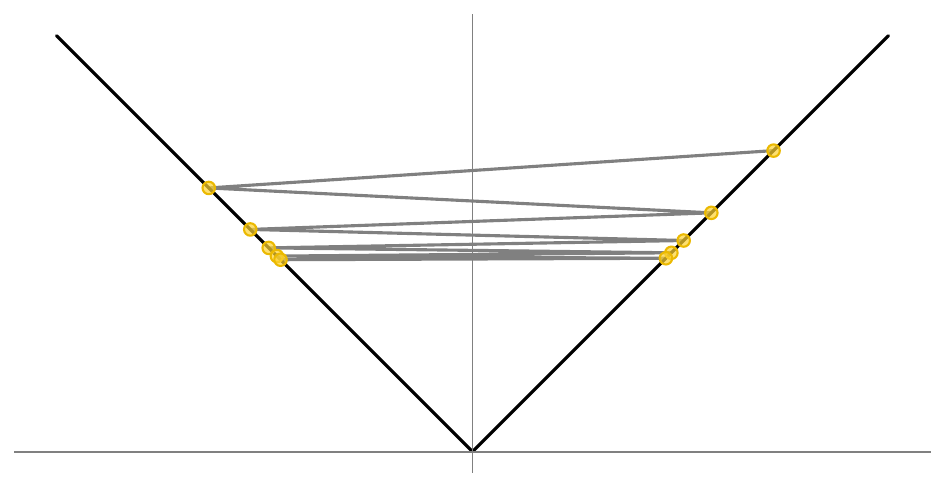}
\caption{GEQ $\notimplies$ NR}
\label{fig:lrs_notimplies_nr}
\end{subfigure}
\caption{\it Two examples with $\ell_t(\theta) = |\theta|$ which show that 
  neither NR nor GEQ necessarily implies the other. In each panel, the iterates
  start at  the upper right-most point, and the thin gray lines are simply as
  visual aid to demonstrate the order of the sequence.} 
\label{fig:abs_value}
\end{figure}

\paragraph{Example 2: squared losses.}

After the last example, illustrated in Figure \ref{fig:nr_notimplies_lrs}, it is
natural to ask whether this kind of behavior, where iterates can satisfy NR but
not GEQ, is limited to settings where the gradients are nonsmooth, as in the
absolute value loss. The answer is ``no'': as the next example shows, a
sequence of iterates can still satisfy NR for smooth and even strongly convex
loss functions, without attaining GEQ. Let \smash{$\ell_t(\theta) = \frac{1}{2}
  (y_t - \theta)^2$}, for $t = 1,2,3,\dots$, where  
\[
(y_1, y_2, y_3, \dots) = ( 
\underbrace{a, a, \dots, a}_{\text{$n$ times}},
\underbrace{b, b, \dots, b}_{\text{$m$ times}},
\underbrace{a, a, \dots, a}_{\text{$n$ times}}, 
\underbrace{b, b, \dots, b}_{\text{$m$ times}},
\dots),
\]
that is, a repeating pattern of $a$ for $n$ elements, $b$ for $m$ elements, $a$
for $n$ elements, and so on. Suppose that $a < 0 < b$ and $n > m$ are such that 
$na + mb = 0$. Suppose that $T$ is a multiple of $n+m$. Then 
\[
\inf_\theta \, \frac{1}{T} \sum_{t=1}^T (y_t - \theta)^2 = \frac{na^2 +
  mb^2}{n+m} = \alpha a^2 + \beta b^2,
\]
where $\alpha = n/(n+m)$ and $\beta = m/(n+m)$. Meanwhile, define 
\[
(\theta_1, \theta_2, \theta_3, \dots) = ( 
\underbrace{u, u, \dots, u}_{\text{$n$ times}}, 
\underbrace{v, v, \dots, v}_{\text{$m$ times}}, 
\underbrace{u, u, \dots, u}_{\text{$n$ times}}, 
\underbrace{v, v, \dots, v}_{\text{$n$ times}}, 
\dots),
\]
that is, a repeating pattern of $u$ for $n$ elements, $v$ for $m$ elements, $u$
for $n$ elements, and so on. The basic idea is to choose $u,v$ in such a way
that average loss of the sequence $\theta_t$ equals $\alpha a^2 + \beta b^2$,
but the average gradient remains positive and bounded away from zero (possible
due to the square growth of the loss functions versus the linear growth of their
gradients). To this end, we set $u = a$ and then choose $v > b$ by equating  
$\alpha a^2 + \beta b^2$ and \smash{$\frac{1}{T} \sum_{t=1}^T (y_t - \theta_t)^2
  = \beta (b - v)^2$}, which results in \smash{$v = b + \sqrt{\alpha a^2 / \beta
    + b^2}$}. This makes the average gradient:
\[
\frac{1}{T} \sum_{t=1}^T \theta_t = \alpha a + \beta v = 
\sqrt{\alpha \beta a^2 + \beta^2 b^2}.
\]
Hence, we have constructed a sequence with zero regret but with average
gradient, or equivalently average bias, equal to \smash{$\sqrt{\alpha \beta a^2
  + \beta^2b^2}$}, for arbitrarily large $T$.   

Zooming out from this particular construction, our next result shows that for
\emph{any} data sequence $y_t$ with nontrivial sample variance, we can always 
find an iterate sequence $\theta_t$ that has zero regret and average bias
bounded away from zero. The proof is very simple, and deferred until Appendix 
\ref{app:squared_loss_nr_lrs}. 

\begin{proposition}
\label{prop:squared_loss_nr_lrs}
For any sequence $y_t$, $t = 1,\dots,T$, denote its sample mean and variance
by \smash{$\bar{y}_T = \frac{1}{T} \sum_{t=1}^T y_t$} and \smash{$s_T^2 =
  \frac{1}{T} \sum_{t=1}^T (y_t - \bar{y}_T)^2$}. There exists a sequence
$\theta_t$, $t = 1,\dots,T$ such that    
\[
\frac{1}{T} \sum_{t=1}^T (y_t - \theta_t)^2 =s_T^2, \quad \text{and} \quad  
\bigg| \frac{1}{T} \sum_{t=1}^T \theta_t - \frac{1}{T} \sum_{t=1}^T y_t \bigg|
\geq s_T. 
\]
\end{proposition}

\subsubsection{The broader picture}

What did we learn from the last two examples? The main takeaway:
\begin{quote}\it
No regret does not imply gradient equilibrium, even for smooth and
strongly convex functions.
\end{quote}
This is an important point, and---together with the converse fact, that gradient
equilibrium does not imply no regret, even with strong assumptions on the loss
functions in question---supports the study of gradient equilibrium as a
standalone property of interest in online learning, beyond no regret. 

A careful look back at the last subsection actually suggests a second takeaway:      
\begin{quote}\it
Iterate convergence does not imply gradient equilibrium if the gradients are
nonsmooth.    
\end{quote}
Here by \emph{iterate convergence} we mean convergence of the iterate sequence 
to a minimizer of the average loss:
\begin{equation}
\label{eq:iterate_conv}
\inf \bigg\{ \|\theta_T - \theta^\star_T\|_2 : \text{$\theta^\star_T$ minimizes}
\; \bar\ell_T = \frac{1}{T} \sum_{t=1}^T \ell_t \bigg\} \to 0, \quad \text{as $T
  \to \infty$}. 
\end{equation}
This is a very strong condition, especially in the online setting, and as such
it is not typically studied in the online optimization literature. The point of
raising it here is not to suggest that it is interesting, but instead to
emphasize that even under a strong condition on the iterates such as
\eqref{eq:iterate_conv}, the GEQ property \eqref{eq:grad_eq} is not guaranteed 
to hold if the gradients are nonsmooth, as the example in Figure
\ref{fig:nr_notimplies_lrs} demonstrates.\footnote{It is not hard to see that
  some level of nonsmoothness is needed for iterate convergence to fail to imply
  gradient equilibrium. If the loss functions $\ell_t$ are smooth enough such
  that, say, \smash{$\nabla \bar\ell_T$} is $L$-Lipschitz (where the constant $L
  > 0$ does not depend on $T$), then we have \smash{$\|\nabla
    \bar\ell_T(\theta_T) \|_2 = \|\nabla \bar\ell_T(\theta_T) - \nabla
    \bar\ell_T(\theta^\star_T) \|_2 \leq L \|\theta_T - \theta^\star_T\|_2 \to
    0$.}}          

The discussion of regret-type properties in relation to gradient equilibrium 
in continued in Appendix \ref{app:no_move_regret}, where we introduce a 
condition we call \emph{no move regret}, which can be viewed as a conceptual 
stepping stone between NR and GEQ. 

% The discussion of iterate convergence is continued in Appendix
% \ref{app:avg_iterate_conv}, where we discuss \emph{average iterate
% convergence}, which does not imply GEQ in general, even under smoothness
% assumptions.       

\subsection{Related work}
\label{sec:related_work}

Online learning studies algorithms that process data sequentially, aiming to
maximize a cumulative measure of performance while making no assumptions
about putative mechanisms that generated the sequence. We begin by overviewing
the historical roots of the adversarial sequence model, which underlie both the
classical regret-based perspective on online learning and our gradient
equilibrium perspective. 

\subsubsection{Historical roots}

The study of algorithms for analyzing arbitrary sequences under no statistical
assumptions arose historically in separate threads in various fields over
several decades. One important thread was game theory. Blackwell, inspired by
von Neumann's minimax theorem for two-player, zero-sum games, asked whether a
similar result could be obtained for vector-valued payoffs \cite{Blackwell56}.
After realizing that no such theorem could be obtained for single-move games, he
asked instead what could be achieved in a repeated, non-zero-sum game, where the
evaluation of performance was expressed in terms of a time-averaged
vector-valued payoff.  Specifically, one player was conceived of as trying to
force a time-averaged payoff to approach a given set, while the competing player
was viewed as trying to prevent the time-averaged payoff from approaching that
set.  Blackwell derived conditions under which such ``approachability'' can be
achieved, and developed a simple randomized algorithm that a player could employ
to ensure approachability, even in the face of an adversarial sequence of
actions from the opposing player.  This paper developed time-averaged
performance measures, adversarial sequences, and vector-valued payoffs, all
elements which have formed the basis of various further developments in game
theory and (subsequently) online learning. Concurrent work by Hannan, on
repeated games, further refined the notion of a time-averaged performance
measure, by introducing the ``regret'' of a player via a comparison to an
oracle who could see the entire sequence of play \cite{Hannan57}.  Recent work
has established the equivalence of the approachability framework and the
regret-minimizing framework \cite{AbernethyEtal11}.

A different thread came from information theory and finance, as methods for
investment decisions and compression were developed that eschewed traditional
stochastic assumptions, and ideas such as hedging and loss minimization began to
take shape \cite{Kelly56, Ziv-Lempel, Cover91}.  

Online learning initially developed independently of this earlier literature;
it was focused on computational issues and on specific gradient-based and
multiplicative algorithms for prediction and decision-making~\cite{Vovk1990,
littlestone-warmuth, cesa-bianchi-etal}. The performance measures were of a
statistical flavor, with the canonical examples being squared error and zero-one
loss, but the focus was on theoretical guarantees on predicting labels or
selecting models that were obtained without making classical statistical
assumptions. Given this stance, connections to the earlier literature on
arbitrary, possibly adversarial data soon became apparent. In fact, when
Cesa-Bianchi and Lugosi wrote the first comprehensive book on online learning
\cite{cesa-bianchi-lugosi}, connections to Blackwell, Cover, Hannan, and Kelly
were explicit and well-developed.  

An additional crucial ingredient in the melting pot of online learning was
optimization theory, specifically the branch of optimization theory where the
focus was iterative, gradient-based algorithms \cite{kiefer-wolfowitz,
  nemirovski-yudin}.  An important unifying step was made by Zinkevich, who
considered optimization problems involving sums of convex functions, doing so
within the regret perspective of online learning \cite{zinkevich2003online}.
This opened the door to many further connections. Indeed, it turned out that the
algorithms studied in gradient-based optimization theory---most notably, the 
mirror descent algorithm of Nemirovski and Yudin---not only exhibit favorable
convergence in optimization-theoretic analysis but also have favorable regret
guarantees in an online-learning analysis.  Thus, recent textbooks have
emphasized the powerful connections between optimization and online learning
\cite{hazan2016introduction, Orabona}. The story has been enriched by
regularization; the two main algorithmic success stories in online
learning---mirror descent and follow the regularized leader
\cite{mcmahan2013ad}---both make essential use of regularized updates, 
which are needed to handle the adversarial nature of the problem.

Further connections emerged from statistics and economics in the form
of calibration \cite{foster-vohra}. Calibration can be viewed as a weak but
useful desideratum for a learning algorithm. A well-known example is weather
prediction: when a meteorologist says ``the chance of rain today is 70\%,'' how
do we evaluate such a prediction given that it either rains or it doesn't?  A
standard calibration method would define a grid on the interval $(0,1)$,
consider the subset of days in which the prediction falls in a particular bin,
and ask that the empirical proportion of rainy days for that subset is close to
the number at the center of the bin (say 0.7). It is relatively straightforward
to calibrate predictions in batch mode, at the end of a sequence of predictions,
but it is more useful in many problems to calibrate as the data
arrive. Moreover, the goal is generally to calibrate for the particular sequence
at hand. This motivates studying calibration within an online, regret-based
framework. An example of such a framework was presented by Foster
\cite{foster99}, based on a connection to approachability.

% In summary, regret has played an important historical role in providing a
% rigorous foundation for online learning, allowing connections to emerge to
% numerous related topics. But it is important to not equate online learning with
% regret-based analysis, and the next subsection discusses recently developed
% approaches to the adversarial sequence model that go beyond standard notions of
% regret. 

\subsubsection{Recent developments}

A closely related line of work to this manuscript is the literature on online
conformal prediction, introduced in \cite{gibbs2021adaptive}. Building on the
literature on conformal prediction \cite{vovk2005algorithmic}, in online
conformal we seek to construct a sequence of prediction sets that satisfy a
coverage guarantee on average over the sequence: 
\[
\frac{1}{T} \sum_{t=1}^T 1\{y_t \in C_t\} \to 1-\alpha, \quad \text{as $T \to
  \infty$}, 
\]
where $y_t$, $t = 1,2,3,\dots$ are ground-truth labels, and $1-\alpha$ is a
desired coverage level. It is well-known that the adaptive conformal inference
algorithm proposed by \cite{gibbs2021adaptive} is an instance of online gradient
descent with respect to the quantile loss. Many papers have built on this
algorithm, and built on its characterization as an online optimization method,
to prove coverage and uncertainty quantification guarantees
\cite{zaffran2022adaptive, feldman2022achieving, bhatnagar2023improved,
  angelopoulos2023conformal, lekeufack2024conformal}, regret-type guarantees
\cite{chen2023ipoc, bhatnagar2023improved, ge2024stochastic, gibbs2024conformal,  
  podkopaev2024adaptive, zhang2024importance}, convergence guarantees
\cite{angelopoulos2024online, zecchin2024localized}, and more.  

In a way, our paper builds on the online conformal prediction literature and
extends these ideas to general optimization problems. As we will show in Section 
\ref{sec:quantile_losses}, online conformal prediction is equivalent to finding
a sequence of quantiles that satisfy gradient equilibrium with respect to the
quantile loss. However, our paper takes this idea much farther, defining
gradient equilibrium as a general condition, analyzing it in general, and
deriving its implications for other families of losses beyond the quantile
loss. 
% Though some previous work has extended the online conformal ideas beyond the
% pinball loss \cite{feldman2022achieving, lekeufack2024conformal}, the approach
% we take in this paper is substantially more general and systematic. 

Another closely related line of work is that on multivalid conformal
\cite{gupta2021online, Bastani22, jung2023batch, deng2023happymap,
  gibbs2023conformal, blot2024automatically, noarov2023high}, and especially its
online variant as described in \cite{Bastani22}. In this line of work, it is
shown that conformal prediction is an intercept-only quantile regression 
problem, and that simultaneous coverage guarantees over overlapping groups can
be obtained by running a quantile regression on the vector of group indicators,
i.e., (possibly overlapping) subsets of the domain $\cX$. The literature on
multivalidity was inspired by multicalibration \cite{hebert-johnson-etal} and
multiaccuracy \cite{Kim19},  which provide technical tools towards mitigating
systematic biases in learning systems and have strong connections to the
literature on fairness in machine learning \cite{Barocas-Hardt-Narayanan}. 
The guarantees given by multivalid conformal algorithms parallel the
multigroup debiasing guarantees in Sections \ref{sec:glms_ortho_features},
\ref{sec:logistic_lasso}, and \ref{sec:squared_ridge}, and the multiaccuracy
guarantees in Section \ref{sec:multiaccuracy}. We view our work is
complementary: in some sense, it is broader in scope (since we study gradient
equilibrium in general), while our guarantees may be less developed and
coarser. We also focus on applying online gradient descent as our main workhorse 
algorithm, rather than creating bespoke iterative procedures.    

Farther afield, we remark that the concept of gradient equilibrium can be seen
as an online variant of an estimating equation; see \cite{godambe1991estimating,
  qin1994empirical} for classical references. Recall, our condition is of the form  
\[
\frac{1}{T} \sum_{t=1}^T g_t(\theta_t) \to 0, \quad \text{as $T \to \infty$}.
\]
which---viewing the average in time as an expectation---resembles the form of a
generic estimating equation. The gradient equilibrium implications that we
examine in this paper are often interpretable as online analogs of standard
first-order optimality guarantees in an M-estimation problem. The main
difference is that we are averaging over a sequence in a nonstochastic
adversarial setting, and in comparison the standard guarantees in M-estimation
or estmating equations are given in a stochastic setting in
expectation. Revisiting Table \ref{tab:grad_eq} will give the flavor of this
parallel---the squared loss leads to an unbiasedness statement, the quantile
loss leads to coverage, and so on.

Finally, in modern optimization-based online learning analyses, we note that it
is common to analyze the convergence of the average gradient norm (for example, 
Theorem 2.1 of \cite{ghadimi2013stochastic} or Proposition 2.6 of
\cite{hazan2017efficient}). Although this problem seems relevant at first
glance, there is an important difference: these references study the
\emph{average norm of the gradient}, while we study the \emph{norm of the 
  average gradient}. These are very different objects. In gradient equilibrium,
the individual gradients need not converge to zero. This allows our theory, as
we will see in the next section, to handle a much broader class of functions
than typical analyses. 

\section{Gradient descent}
\label{sec:grad_descent}

We now study how to achieve gradient equilibrium via \emph{online gradient
  descent}, which, given an initial point $\theta_1 \in \R^d$, produces iterates 
according to:      
\begin{equation}
\label{eq:grad_descent}
\theta_{t+1} = \theta_t - \eta_t g_t(\theta_t), \quad t = 1,2,3,\dots, 
\end{equation}
for a sequence $\eta_t > 0$, $t = 1,2,3,\dots$ of step sizes. Technically,
this is the online subgradient method, since, recall, we use $g_t(\theta_t)$ to
denote a generalized subgradient of $\ell_t$ at $\theta_t$. For simplicity, we
refer to the algorithm in \eqref{eq:grad_descent} as gradient descent (GD), and 
$g_t(\theta_t)$ as a gradient.   

In what follows, we assume that each loss $\ell_t$ is finite and
subdifferentiable on all of $\R^d$. Importantly, we do not assume
convexity. We also focus on constant step sizes, $\eta_t = \eta$, for all
$t$. Later, in Section \ref{sec:arbitrary_steps}, we allow for arbitrary step
sizes.    

\subsection{Bounded or slowly growing iterates}
\label{sec:slowly_growing}

We begin by deriving a simple but useful bound on the average gradient.

\begin{proposition}
\label{prop:gd_simple}
Consider gradient descent \eqref{eq:grad_descent}, with arbitrary initialization
$\theta_1 \in \R^d$, and constant step sizes $\eta_t = \eta > 0$, for all $t$.
The average gradient satisfies
\begin{equation}
\label{eq:gd_avg_grad_identity}
\frac{1}{T} \sum_{t=1}^T g_t(\theta_t) = \frac{\theta_1 - \theta_{T+1}}{\eta T}, 
\end{equation}
and therefore
\begin{equation}
\label{eq:gd_avg_grad_bound}
\bigg\| \frac{1}{T} \sum_{t=1}^T g_t(\theta_t) \bigg\|_2 \leq
\frac{\|\theta_1\|_2 + \|\theta_{T+1}\|_2}{\eta T}.   
\end{equation}
\end{proposition}

\begin{proof}
Rewrite the iteration \eqref{eq:grad_descent}, with $\eta_t = \eta$, as
$\theta_t - \theta_{t+1} = \eta g_t(\theta_t)$. Adding this up over $t =
1,\dots,T$, the left-hand side telescopes, yielding   
\[
\theta_1 - \theta_{T+1} = \eta \sum_{t=1}^T g_t(\theta_t).
\]
Dividing both sides by $\eta T$ proves \eqref{eq:gd_avg_grad_identity}. The
bound in \eqref{eq:gd_avg_grad_bound} follows by simply taking the norm of both
sides, and then applying the triangle inequality.
\end{proof}

Proposition \ref{prop:gd_simple} is entirely elementary, yet the end result
reveals an important property of gradient descent with constant step sizes:
this algorithm achieves gradient equilibrium whenever the iterates remain
bounded, since, if \smash{$\sup_{t \geq 1} \|\theta_t\|_2 \leq b$} for a
constant $b > 0$ (not depending on $T$), then by \eqref{eq:gd_avg_grad_bound},      
\[
\bigg\| \frac{1}{T} \sum_{t=1}^T g_t(\theta_t) \bigg\|_2 \leq \frac{2b}{T} \to
0, \quad \text{as $T \to \infty$}. 
\]
It is worth noting that the bound $b$ does not need to be known in order to 
run GD in the first place. It is also worth reiterating the contrast to regret:
in general, GD with constant step sizes does not deliver a no-regret guarantee
(as the example in Figure \ref{fig:regret_and_bias} illustrates).   

It turns out that boundedness of the GD iterates arises naturally with some loss
functions, and we will discuss this shortly. Meanwhile, a key realization is
that boundedness is not actually necessary for gradient equilibrium: from
\eqref{eq:gd_avg_grad_bound}, we see that we only need $\|\theta_{T+1}\|_2$ to
grow sublinearly in $T$. We use the term \emph{slowly growing} to refer to an
iterate sequence which satisfies this sublinearity condition, writing it as
\smash{$\|\theta_t\|_2 = o(t)$}. Notice that from
\eqref{eq:gd_avg_grad_identity}, gradient equilibrium itself implies $\|\theta_t
- \theta_1\|_2$, hence $\|\theta_t\|_2$, is slowly growing. The next result
summarizes, for completeness.

\begin{proposition}
\label{prop:gd_sublinear}
For gradient descent \eqref{eq:grad_descent} with constant step sizes, gradient
equilibrium \eqref{eq:grad_eq} holds if and only if the iterates are slowly
growing.  
\end{proposition}

While sublinearity of the iterate norm $\|\theta_t\|_2$ clearly arises from
\eqref{eq:gd_avg_grad_identity} and \eqref{eq:gd_avg_grad_bound} as a necessary and
sufficient condition for gradient equilibrium, it would be dissatisfying for the
analysis to end here. This is because we do not yet have an idea of which
problem settings (i.e., which sequences of loss functions $\ell_t$, $t =
1,2,3,\dots$) lead to such sublinear growth once we apply GD. To address this,
we now present a general condition on the loss sequence that controls the growth
of GD iterates.

\subsection{Restorative loss functions}
\label{sec:restorative}

For a parameter $h \geq 0$, and nonnegative function $\phi : \R^d \to \R_+$,   
we will say that a loss function $\ell : \R^d \to \R$ admits a
\emph{$(h,\phi)$-restorative negative gradient field} provided that   
\begin{equation}
\label{eq:restorative}
g(\theta)^\T \theta \geq \phi(\theta), \quad \text{for all $\|\theta\|_2 > h$,
  and all generalized subgradients $g(\theta)$ of $\ell$ at $\theta$}.   
\end{equation}
Informally, this condition says that if $\theta$ is far enough away from the
origin (a distance of at least $h$), then the negative gradient field of $\ell$
is antialigned with directions of growth, because $-g(\theta)^\T \theta \leq
-\phi(\theta) \leq 0$. This field hence imparts something of a restorative
``force,'' which stops the parameter from ``escaping to infinity'' too
quickly. The parameter $h$, which we call the \emph{horizon}, is the distance at
which this force activates; $\phi$, which we call the \emph{curvature function}, 
controls the strength of the restorative force. From here on, we will refer to a
loss $\ell$ as being $(h,\phi)$-restorative, with the understanding that this is
really a property of its negative gradient field. See Figure
\ref{fig:restorative} for an illustration.    

\begin{figure}[p]
\vspace{-20pt} % Spacing hack, to land this more vertically centered
\centering
\includegraphics[width=0.6\textwidth]{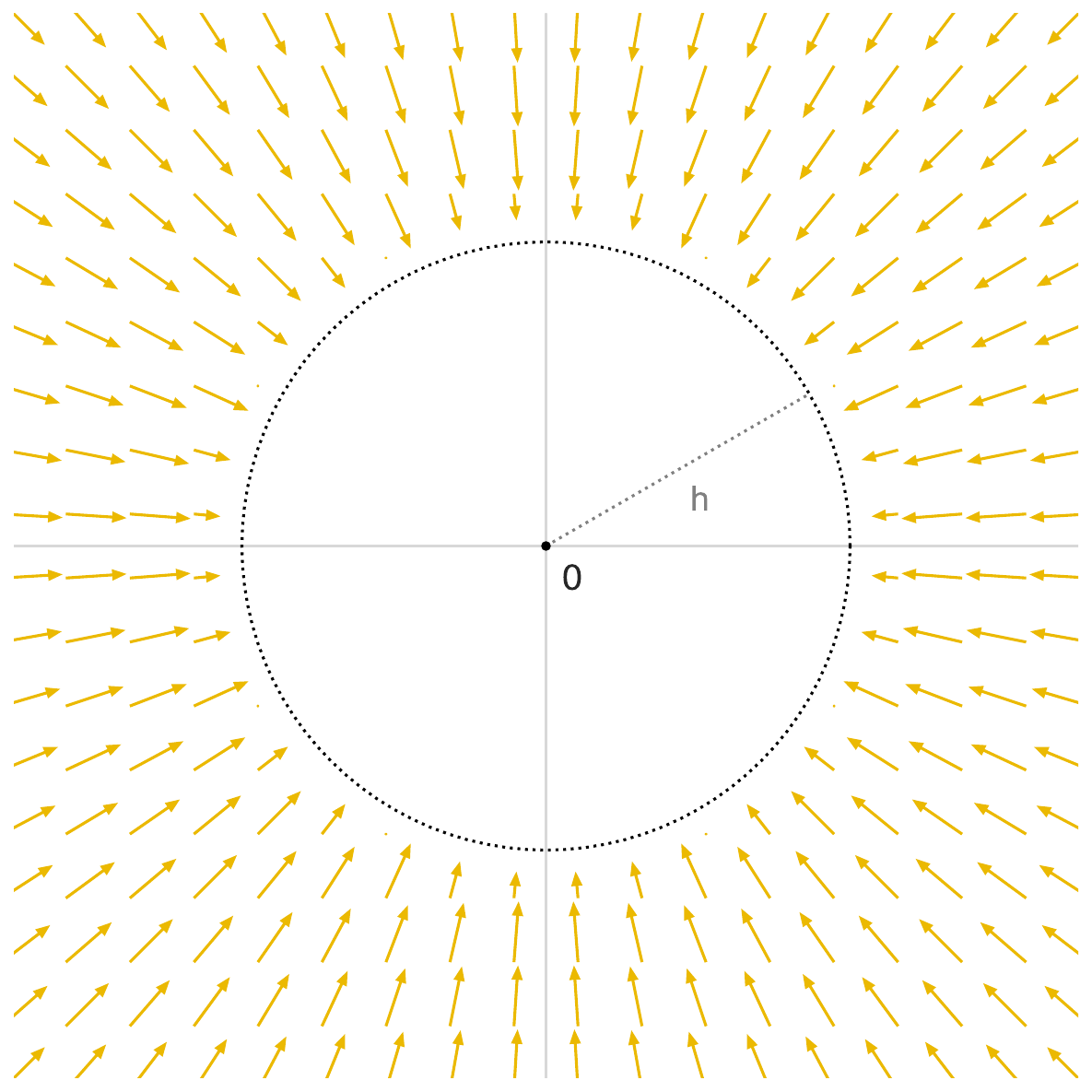} 
\caption{\it Illustration of a restorative field, where each gold arrow
  represents the negative gradient of $\ell_t$ at a particular point. Note that
  if we take this point to be $\theta_t$, then the gradient descent update would
  move $\theta_{t+1}$ in the direction of arrow. The gold arrows need to point 
  inwards outside of a radius $h$; within the radius, the field can be
  arbitrary, so it is not drawn. }   
\label{fig:restorative}

\bigskip\bigskip
\begin{subfigure}{0.475\textwidth}
\centering
\includegraphics[width=\textwidth]{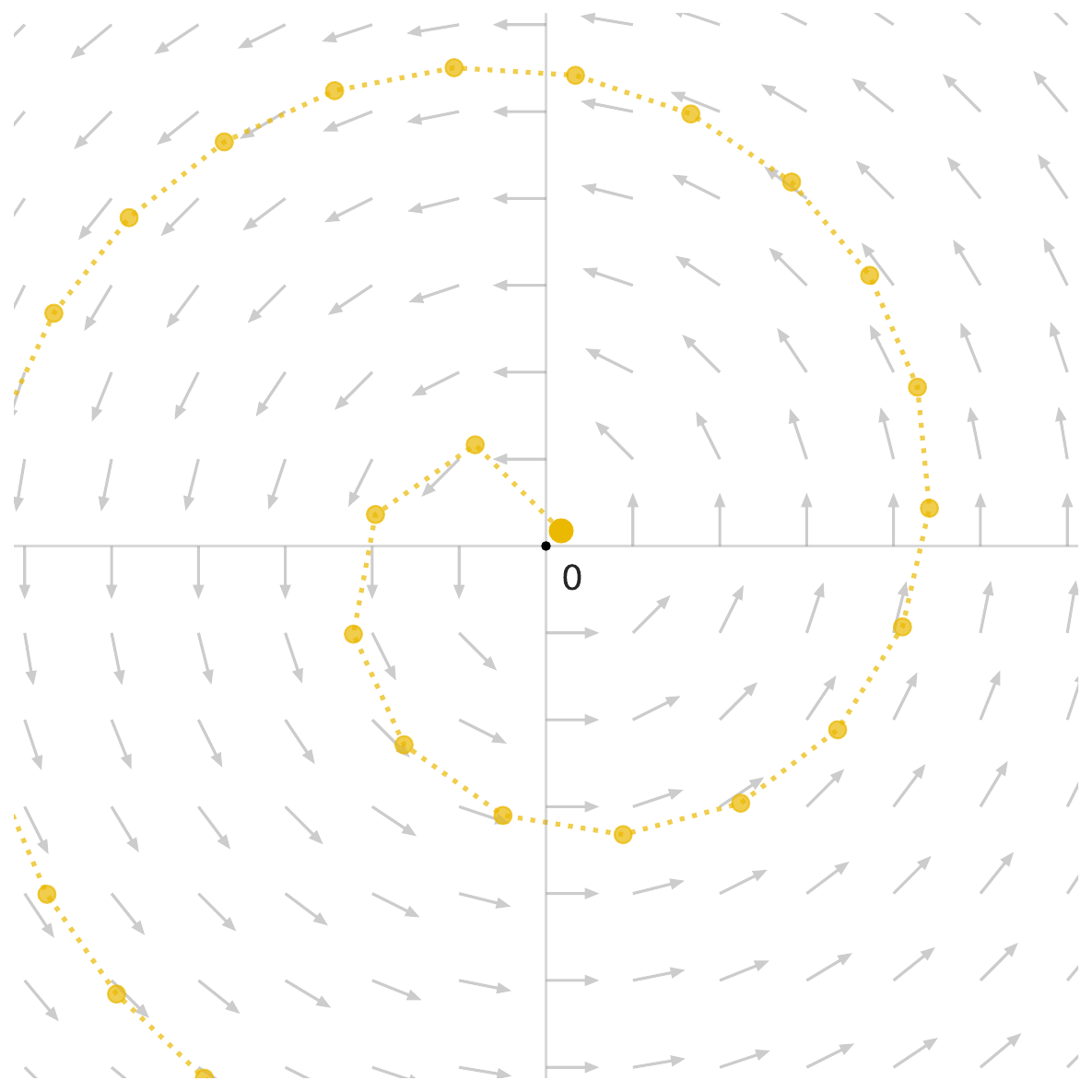}
\caption{Zero curvature}
\label{fig:zero_curvature}
\end{subfigure} \hfill
\begin{subfigure}{0.475\textwidth}
\centering
\includegraphics[width=\textwidth]{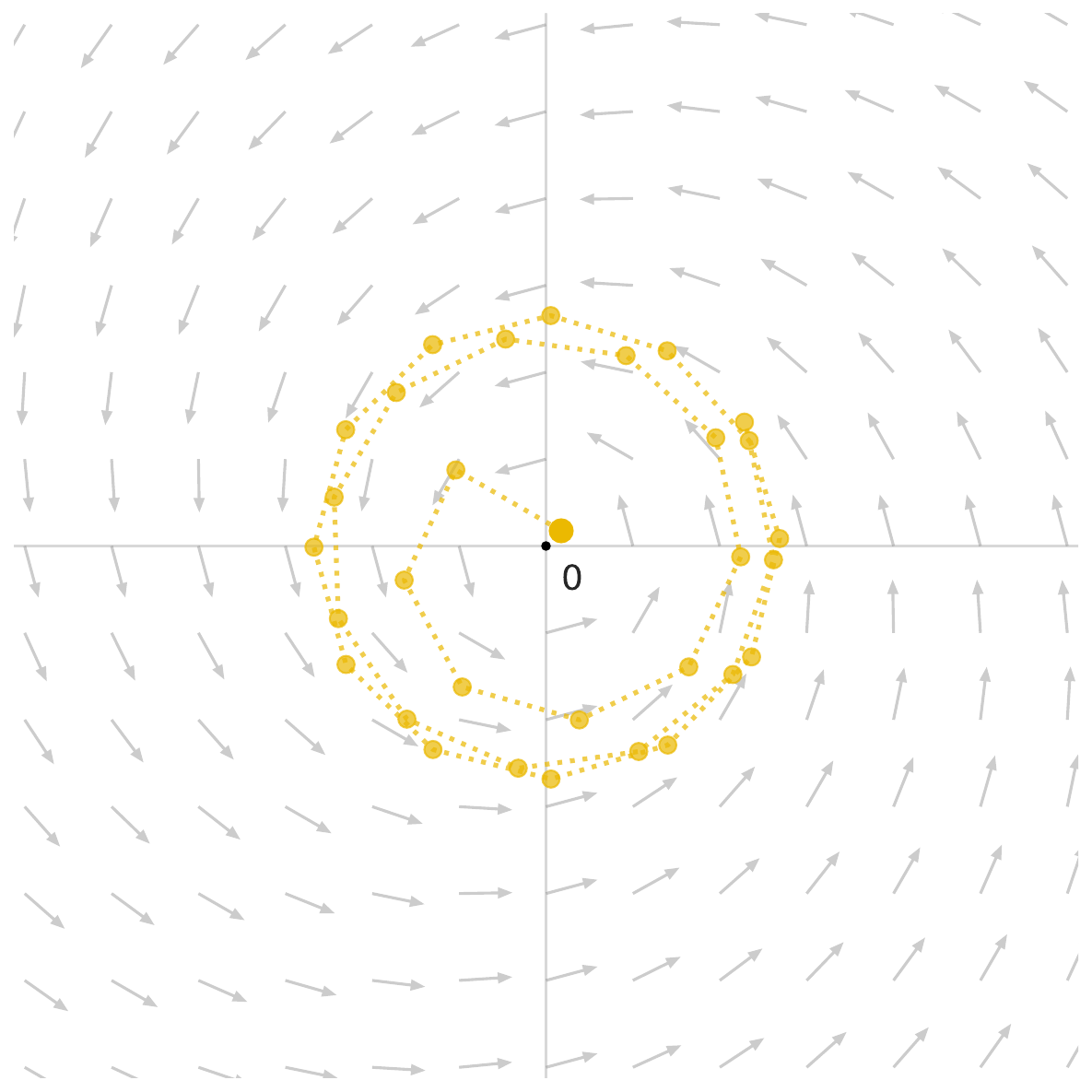}
\caption{Positive curvature}
\label{fig:pos_curvature}
\end{subfigure} 
\caption{\it Two example gradient descent trajectories, in gold, resulting from
  different negative gradient fields. In the left plot, the field has zero
  curvature, pointing at a $90^{\circ}$ angle from each line to the origin,
  causing the gradient descent iterates to spiral outwards. In the right plot,
  the field has a positive curvature of $15^{\circ}$, and this causes the
  iterates to remain bounded.}    
\label{fig:curvature}
\end{figure}

In what follows, we examine the growth of GD iterates for what we call a 
\emph{restorative loss sequence}. This is a sequence $\ell_t$, $t =
1,2,3,\dots$, such that
\begin{itemize}
\item each $\ell_t$ is $(h_t,\phi_t)$-restorative, for some horizon $h_t$
  and curvature function $\phi_t$;
\item the horizon sequence is sublinear, $h_t = o(t)$, and nondecreasing.       
\end{itemize}
It turns out that the univariate and multivariate cases are fundamentally 
different, in terms of the required curvature in the negative gradient field. We
therefore separate our investigation into cases. The proofs of all results below
are elementary, but deferred until the appendix to preserve the flow of ideas.   

\subsubsection{Univariate losses with weak curvature}

Our first result studies restorative univariate loss sequences with zero
curvature. Note that in the univariate case, $d = 1$, the restorative condition
\eqref{eq:restorative} with $\phi = 0$ reduces to    
\begin{equation} 
\label{eq:restorative_zero_curv_1d}
\sign(g(\theta)) = \sign(\theta), \quad \text{for all $|\theta| > h$, and all
  generalized subgradients $g(\theta)$ of $\ell$ at $\theta$}.    
\end{equation}
In other words, the negative gradient field points back toward the origin
beyond a horizon $h$. This property, along with a Lipschitz assumption on the
loss, ensures strong control of univariate GD iterates even under the weakest
possible curvature assumption, with $\phi = 0$. The proof of the next result is
in Appendix \ref{app:restorative_zero_curv_1d}. 

\begin{proposition}
\label{prop:restorative_zero_curv_1d} 
Let $d = 1$ and assume that each $\ell_t$ is $L$-Lipschitz, and
$(h_t,0)$-restorative. Assume also that $h_t$ is nondecreasing. Then gradient
descent \eqref{eq:grad_descent}, under arbitrary initialization $\theta_1 \in
\R$, and constant step sizes $\eta_t = \eta > 0$, for all $t$, satisfies     
\begin{equation}
\label{eq:gd_iterate_bound_zero_curv_1d}
|\theta_{T+1}| \leq \max\{ |\theta_1|, h_T \} + \eta L.
\end{equation}
In particular, if $h_t$ is sublinear, then the iterates are slowly growing. By
\eqref{eq:gd_avg_grad_bound}, they satisfy gradient equilibrium:       
\begin{equation}
\label{eq:gd_avg_grad_bound_zero_curv_1d}
\bigg| \frac{1}{T} \sum_{t=1}^T g_t(\theta_t) \bigg| \leq \frac{2
  |\theta_1|}{\eta T} + \frac{L}{T} + \frac{h_T}{\eta T} \to 0, \quad
\text{as $T \to \infty$}.  
\end{equation}
\end{proposition}

When the horizon sequence is constant, $h_t = h$ for all $t$, the result in
\eqref{eq:gd_iterate_bound_zero_curv_1d} shows that the iterates remain bounded,
and \eqref{eq:gd_avg_grad_bound_zero_curv_1d} shows that the average gradient
diminishes at the rate $1/T$. We note that this effectively generalizes results
from \cite{gibbs2021adaptive, angelopoulos2023conformal} on quantile losses,
which, under a boundedness assumption on the data, satisfy the property
\eqref{eq:restorative_zero_curv_1d} for a global constant $h > 0$. We return to
this setting in Section \ref{sec:quantile_losses}.

\subsubsection{Multivariate losses with weak curvature}

When $d \geq 2$, the story is quite different under zero curvature in a
restorative negative gradient field, as the next result shows. Its proof is
given in Appendix \ref{app:restorative_zero_curv}.  

\begin{proposition}
\label{prop:restorative_zero_curv} 
Assume that each $\ell_t$ is $L$-Lipschitz and $(h_t,0)$-restorative. Then
gradient descent \eqref{eq:grad_descent}, with arbitrary initialization
$\theta_1 \in \R^d$, and constant step sizes $\eta_t = \eta > 0$, for all $t$,
satisfies        
\begin{equation}
\label{eq:gd_iterate_bound_zero_curv}
\|\theta_{T+1}\|_2 \leq \sqrt{\|\theta_1\|_2^2 + \eta^2 L^2 T + 2 \eta L
  \sum_{t=1}^T h_t}.  
\end{equation}
In particular, if $h_t$ is sublinear and nondecreasing, then the iterates are
slowly growing. By \eqref{eq:gd_avg_grad_bound}, they satisfy gradient
equilibrium:        
\begin{equation}
\label{eq:gd_avg_grad_bound_zero_curv}
\bigg\| \frac{1}{T} \sum_{t=1}^T g_t(\theta_t) \bigg\|_2 \leq \frac{2
  \|\theta_1\|_2}{\eta T} + \sqrt{\frac{L^2}{T} + \frac{2 L h_T}{\eta T}} \to 0,
\quad \text{as $T \to \infty$}.  
\end{equation}
\end{proposition}

Note that Proposition \ref{prop:restorative_zero_curv} can never guarantee
bounded iterates. The upper bound \eqref{eq:gd_iterate_bound_zero_curv} on
$\|\theta_{T+1}\|_2$ grows at best at the rate \smash{$\sqrt{T}$}, and hence the
bound on the average gradient in \eqref{eq:gd_avg_grad_bound_zero_curv} grows at
best at the rate \smash{$1/\sqrt{T}$}. These rates are achieved by a constant
horizon sequence, $h_t = h$ for all $t$.

Meanwhile, it is not difficult to argue that Proposition
\ref{prop:restorative_zero_curv} cannot be improved in terms of its dependence
on $T$ when $d \geq 2$, under the stated assumptions. Suppose that
\smash{$g_t(\theta_t)^\T \theta_t = 0$} and $\|g_t(\theta_t)\|_2 = L$, for all
$t$ (this can be satisfied when $d \geq 2$, but not when $d =  1$). Then we   
have  
\begin{align*}
\|\theta_{T+1}\|_2^2 
&= \|\theta_T\|_2^2 + \eta^2 \|g_T(\theta_T)\|_2^2 - 2 \eta 
  g_T(\theta_T)^\T \theta_T \\ 
&= \|\theta_T\|_2^2 + \eta^2 L^2 \\
&=  \|\theta_1\|_2^2 + \eta^2 L^2 T,
\end{align*}
thus $\|\theta_{T+1}\|_2$ grows precisely at the rate \smash{$\sqrt{T}$}. 
Figure \ref{fig:zero_curvature} gives an illustration. Hence we see that, in
order to maintain bounded iterates in multiple dimensions, we need to strengthen
the curvature assumption.   

\subsubsection{Multivariate losses with strong curvature}

Our next result considers restorative loss sequences under a positive but
constant curvature assumption. Its proof is given in Appendix
\ref{app:restorative_pos_curv}.    

\begin{proposition}
\label{prop:restorative_pos_curv} 
Assume that each $\ell_t$ is $L$-Lipschitz and $(h_t,\phi_t)$-restorative, 
where the curvature is a sufficiently large positive constant:
\begin{equation}
\label{eq:pos_curvature}
\phi_t(\theta) \geq \frac{\eta L^2}{2}, \quad \text{for all $\|\theta\|_2 >
  h_t$}.   
\end{equation}
Assume also that $h_t$ is nondecreasing. Then gradient descent
\eqref{eq:grad_descent}, with arbitrary initialization $\theta_1 \in \R^d$, and
constant step sizes $\eta_t = \eta > 0$, for all $t$, satisfies   
\begin{equation}
\label{eq:gd_iterate_bound_pos_curv}
\|\theta_{T+1}\|_2 \leq \max\{ \|\theta_1\|_2, h_T \} + \eta L. 
\end{equation}
Thus if $h_t$ is sublinear, then the iterates are slowly growing. By
\eqref{eq:gd_avg_grad_bound}, they satisfy gradient equilibrium:     
\begin{equation}
\label{eq:gd_avg_grad_bound_pos_curv}
\bigg\| \frac{1}{T} \sum_{t=1}^T g_t(\theta_t) \bigg\|_2 \leq \frac{2
  \|\theta_1\|_2}{\eta T} + \frac{L}{T} + \frac{h_T}{\eta T} \to 0, \quad
\text{as $T \to \infty$}.  
\end{equation}
\end{proposition}

We note that Proposition \ref{prop:restorative_pos_curv} produces bounds of
the same form as in Proposition \ref{prop:restorative_zero_curv_1d}, but it does
so under the minimum positive curvature assumption \eqref{eq:pos_curvature}. For
a constant horizon sequence, $h_t = h$ for all $t$, the upper bound on 
$\|\theta_{T+1}\|_2$ in \eqref{eq:gd_iterate_bound_pos_curv} is of constant
order, and the average gradient bound \eqref{eq:gd_avg_grad_bound_pos_curv}
scales at the rate $1/T$. Figure \ref{fig:pos_curvature} gives an illustration.

Our next result weakens the Lipschitz condition on the loss, while assuming  
even stronger curvature.  Its proof is given in Appendix
\ref{app:restorative_quad_curv}.

\begin{proposition}
\label{prop:restorative_quad_curv} 
Assume that each $\ell_t$ is $L_t$-Lipschitz on $\{ \theta \in \R^d :
\|\theta\|_2 \leq h_t \}$, equivalently,  
\begin{equation}
\label{eq:local_lipschitz_bound}
\|g_t(\theta)\|_2 \leq L_t, \quad \text{for all $\|\theta\|_2 \leq h_t$},  
\end{equation}
% The subgradient bound holds under local lower semicontinuity; equivalence
% additionally assumes regularity, see Theorem 9.13 and Corollary 8.11 of
% Rockafellar and Wets (2009).
and assume that $\ell_t$ is $(h_t,\phi_t)$-restorative, where the curvature is a
sufficiently large quadratic in the gradient:   
\begin{equation}
\label{eq:quad_curvature}
\phi_t(\theta) \geq \frac{\eta}{2} \|g_t(\theta)\|_2^2, \quad \text{for all 
  $\|\theta\|_2 > h_t$}.  
\end{equation}
Assume also that $h_t,L_t$ are nondecreasing. Then gradient descent
\eqref{eq:grad_descent}, with arbitrary initialization $\theta_1 \in \R^d$, and
constant step sizes $\eta_t = \eta > 0$, for all $t$, satisfies      
\begin{equation}
\label{eq:gd_iterate_bound_quad_curv}
\|\theta_{T+1}\|_2 \leq \max\{ \|\theta_1\|_2, h_T \} + \eta L_T. 
\end{equation}
Thus if $h_t,L_t$ are sublinear, then the iterates are slowly growing. By
\eqref{eq:gd_avg_grad_bound}, they satisfy gradient equilibrium:       
\begin{equation}
\label{eq:gd_avg_grad_bound_quad_curv}
\bigg\| \frac{1}{T} \sum_{t=1}^T g_t(\theta_t) \bigg\|_2 \leq \frac{2
  \|\theta_1\|_2}{\eta T} + \frac{L_T}{T} + \frac{h_T}{\eta T} \to 0, \quad
\text{as $T \to \infty$}.  
\end{equation}
\end{proposition}

Like Proposition \ref{prop:restorative_pos_curv}, this result gives bounded
iterates \eqref{eq:gd_iterate_bound_quad_curv}, and the average gradient
\eqref{eq:gd_avg_grad_bound_quad_curv} vanishes at the rate $1/T$, for a
constant horizon sequence, $h_t = h$ for all $t$. The difference is that we have
only assumed local Lipschitzness \eqref{eq:local_lipschitz_bound}, along with
stronger curvature \eqref{eq:quad_curvature} in the negative gradient field.

\subsection{Examples of restorative losses}
\label{sec:examples}

We work through several examples of restorative losses. 

\subsubsection{Quantile losses}
\label{sec:quantile_losses}

First, consider $\ell_t(\theta) = \rho_\tau(y_t - \theta)$, where $\rho_\tau$
denotes the quantile loss at level $\tau \in [0,1]$, i.e., $\rho_\tau(u) = \tau 
|u|$ for $u \geq 0$ and $(1-\tau) |u|$ for $u < 0$. We will show that this loss  
exhibits the restorative condition \eqref{eq:restorative_zero_curv_1d} with zero
curvature. For $\theta \not= y_t$, the loss $\ell_t$ is differentiable at
$\theta$ with gradient   
\[
g_t(\theta) = 
\begin{cases}
-\tau & \text{if $\theta < y_t$}, \\
1-\tau & \text{if $\theta > y_t$}.
\end{cases}
\]
We can see that \eqref{eq:restorative_zero_curv_1d} is met, provided that we
choose a horizon $h \geq |y_t|$. Note also that $\ell_t$ is Lipschitz with
constant $L = \max\{\tau, 1-\tau\} \leq 1$. The next result summarizes the
conclusion for quantile loss. Its proof is in Appendix
\ref{app:quantile_loss_zero_curv}.  

\begin{corollary}
\label{cor:quantile_loss_zero_curv}
Let $\ell_t(\theta) = \rho_\tau(y_t - \theta)$. Then $\ell_t$ satisfies the
$(h_t,0)$-restorative condition \eqref{eq:restorative_zero_curv_1d} for any $h_t
\geq |y_t|$. Hence supposing each $|y_t| \leq b_t$, where $b_t$ is sublinear and 
nondecreasing, we can set $h_t = b_t$, and by
\eqref{eq:gd_avg_grad_bound_zero_curv_1d} (and assuming we take $g_t(\theta_t) =
-\tau$ whenever $\theta_t = y_t$), 
\begin{equation}
\label{eq:quantile_loss_avg_grad_bound}
\bigg| \frac{1}{T} \sum_{t=1}^T 1\{y_t \leq \theta_t\} - \tau \bigg|  
\leq \frac{2 |\theta_1| + \eta + b_T}{\eta T}.
\end{equation}
\end{corollary}

The result in Corollary \ref{cor:quantile_loss_zero_curv} reproduces an online
conformal prediction result from \cite{angelopoulos2023conformal}. These authors
consider a setting in which each $y_t$ is replaced by a score, denoted $s_t = 
s(x_t,y_t)$, where $s$ is a score function, $x_t$ is a feature, and $y_t$ is a
response. The learned parameter $\theta_t$ is used to build a prediction set via      
\[
C_t(x_t) = \{ y : s(x_t, y) \leq \theta_t \},
\]
for the (unseen) response $y_t$ at $t$. Because $y_t \in C_t(x_t) \iff s_t \leq
\theta_t$, note that \eqref{eq:quantile_loss_avg_grad_bound} applied to the
scores $s_t$, $t = 1,2,3,\dots$ (which we assume are bounded in magnitude by
$b$) translates into the following guarantee:
\begin{equation}
\label{eq:quantile_loss_coverage_bound}
\bigg| \frac{1}{T} \sum_{t=1}^T 1\{y_t \in C_t(x_t) \} - \tau \bigg| \leq
\frac{\eta + b}{\eta T},
\end{equation}
where we set $\theta_1 = 0$ for simplicity. This problem setting will be
studied empirically in Section \ref{sec:quantile_tracking}.  

\subsubsection{Squared losses}
\label{sec:squared_losses}

Moving on to squared loss, \smash{$\ell_t(\theta) = \frac{1}{2} (y_t -
  \theta)^2$}, we will show that this loss satisfies the restorative property 
\eqref{eq:restorative} with quadratic curvature \eqref{eq:quad_curvature}, the
strongest kind of curvature. For such $\ell_t$, this condition becomes           
\[
\theta (\theta - y_t) \geq \frac{\eta}{2} (y_t - \theta)^2, \quad \text{for
  $|\theta| > h_t$}. 
\]
Roughly speaking, we can always satisfy this condition for large enough
$|\theta|$, and $\eta \leq 2$: if (say) $\theta_t = 2y_t > 0$, then the
left-hand side is $2y_t \cdot y_t = 2y_t^2$ but the right-hand side is only
$(\eta / 2) y_t^2 \leq y_t^2$. The next result makes this idea precise. Its
proof is given in Appendix \ref{app:squared_loss_quad_curv}.    

\begin{corollary}
\label{cor:squared_loss_quad_curv}
Let \smash{$\ell_t(\theta) = \frac{1}{2} (y_t - \theta)^2$}. Fixing any $\delta
\in (0,1)$, $\ell_t$ satisfies the $(h_t,\phi_t)$-restorative condition
\eqref{eq:restorative} with quadratic curvature \eqref{eq:quad_curvature}, for  
any horizon $h_t \geq |y_t|/\delta$ and step size $\eta \leq 2(1-\delta) /
(1+\delta)^2$. Furthermore, suppose each $|y_t| \leq b_t$. Then the local 
Lipschitz condition \eqref{eq:local_lipschitz_bound} is met with $L_t =
h_t+b_t$. If $b_t$ is sublinear and nondecreasing, then for any $\delta \in
(0,1)$, we can choose $\eta \leq 2(1-\delta)/(1+\delta)^2$ and $h_t =
b_t/\delta$, and by \eqref{eq:gd_avg_grad_bound_quad_curv},
\begin{equation}
\label{eq:squared_loss_avg_grad_bound}
\bigg| \frac{1}{T} \sum_{t=1}^T \theta_t - \frac{1}{T} \sum_{t=1}^T y_t \bigg| 
\leq \frac{2 |\theta_1| + b_T \eta (1 + 1/\delta) + b_T/\delta}{\eta T}.
\end{equation}
\end{corollary}

Corollary \ref{cor:squared_loss_quad_curv} leads to guarantees for a scheme
where we use GD to debias the output of a prediction model. In this setting,
each $y_t$ is replaced by a residual $y_t - f_t(x_t)$, where $f_t$ is a
black-box predictor, $x_t$ is a feature, and $y_t$ is a response. We use the
parameter $\theta_t$ (learned by gradient descent) to augment each prediction,  
via $f_t(x_t) + \theta_t$. Note that \eqref{eq:squared_loss_avg_grad_bound} 
applied to the residuals $y_t - f_t(x_t)$, $t = 1,2,3,\dots$ (which we assume
are bounded in magnitude by $b$) translates into the following guarantee: 
\begin{equation}
\label{eq:squared_loss_bias_bound}
\bigg| \frac{1}{T} \sum_{t=1}^T (f_t(x_t) + \theta_t) - \frac{1}{T} \sum_{t=1}^T
y_t \bigg| \leq \frac{b \eta (1 + 1/\delta) + b/\delta}{\eta T},
\end{equation}
where we set $\theta_1 = 0$ for simplicity. This problem setting will be studied
empirically in Section \ref{sec:simple_debiasing}.  

\subsubsection{Logistic losses}
\label{sec:logistic_losses}

Next we study a generalized logistic loss, $\ell_t(\theta) = -y_t \theta +
(b-a) \log(1 + e^\theta) + a \theta$, for arbitrary values $a<b$. This
generalization accommodates responses $y_t$ lying in $[a,b]$ (for the usual
logistic loss, we would set $a = 0$ and $b = 1$). We will show that this loss
satisfies the restorative condition \eqref{eq:restorative_zero_curv_1d} with
zero curvature. Note that $\ell_t$ has a gradient at $\theta$ of 
\[
g_t(\theta) = -y_t + (b-a) \frac{e^\theta}{1 + e^\theta} + a.
\]
We can see that $\ell_t$ is Lipschitz with constant $L = b-a$, and yet, it is
clear that condition \eqref{eq:restorative_zero_curv_1d} cannot be met when $y_t
= a$ or $y_t = b$. On the other hand, if $y_t$ is bounded away from $a$ and $b$,
then \eqref{eq:restorative_zero_curv_1d} can be satisfied, as the next
proposition shows. Its proof is in Appendix \ref{app:logistic_loss_zero_curv}.     

\begin{corollary}
\label{cor:logistic_loss_zero_curv}
Let $\ell_t(\theta) = -y_t \theta + (b-a) \log(1 + e^\theta) + a \theta$, where
$a<b$. Then $\ell_t$ satisfies the $(h_t,0)$-restorative condition
\eqref{eq:restorative_zero_curv_1d} for any   
\[
h_t \geq \max\bigg\{ \log \frac{y_t-a}{b-y_t}, \, \log \frac{b-y_t}{y_t-a}
\bigg\}. 
\]
(Note $h_t$ is infinite when $y_t = a$ or $y_t = b$.) Thus, supposing each $y_t
\in [a+\epsilon_t, b-\epsilon_t]$, where $\epsilon_t \in (0, (b-a)/2)$, is
subexponentially vanishing in the sense that $1/\epsilon_t = \exp(o(t))$, and is 
nonincreasing, we can set        
\[
h_t = \log \frac{b-a}{2\epsilon_t}.
\]
This will be sublinear and nondecreasing, and by
\eqref{eq:gd_avg_grad_bound_zero_curv_1d},     
\begin{equation}
\label{eq:logistic_loss_avg_grad_bound}
\bigg| \frac{1}{T} \sum_{t=1}^T \bigg( (b-a) \frac{e^{\theta_t}}{1 +
  e^{\theta_t}} + a \bigg) - \frac{1}{T} \sum_{t=1}^T y_t\bigg| \leq \frac{2 
  |\theta_1| + (b-a) \eta  + \log((b-a)/(2\epsilon_t))}{\eta T}. 
\end{equation}
\end{corollary}

Analogous to the case of squared loss, Corollary
\ref{cor:logistic_loss_zero_curv} leads to guarantees for a scheme where we use 
GD to debias the output of a probabilistic classifer. In this setting, we
replace $y_t$ in the corollary by $y_t - p_t(x_t)$, for a black-box
probabilistic classifer $p_t$, feature $x_t$, and binary response $y_t \in
\{0,1\}$. The learned parameter $\theta_t$ is used to augment the prediction via
\smash{$p_t(x_t) + 2 e^{\theta_t} / (1 + e^{\theta_t}) - 1$}. Notice now that
\eqref{eq:logistic_loss_avg_grad_bound} applied to the residuals $y_t -
p_t(x_t)$, $t = 1,2,3,\dots$ (we take $a = -1$ and $b = 1$, and assume the
residuals lie in $[-1+\epsilon, 1-\epsilon]$, which can be accomplished by
restricting $p_t(x_t)$ to lie in $[\epsilon, 1-\epsilon]$) translates into the
following guarantee:
\begin{equation}
\label{eq:logistic_loss_bias_bound}
\bigg| \frac{1}{T} \sum_{t=1}^T \bigg( p_t(x_t) + \frac{2 e^{\theta_t}}{1 + 
  e^{\theta_t}} -1 \bigg) - \frac{1}{T} \sum_{t=1}^T y_t \bigg| \leq \frac{2
  \eta + \log(1/\epsilon)}{\eta T},
\end{equation}
where we set $\theta_1 = 0$ for simplicity. This problem setting will be
studied empirically in Section \ref{sec:simple_debiasing}.

\subsubsection{GLMs with orthogonal features}
\label{sec:glms_ortho_features}

Lastly, we study a generalized linear model (GLM) loss, $\ell_t(\theta) = -y_t
x_t^\T \theta + \psi(x_t^\T \theta)$, where $x_t$ is a feature, $y_t$ is a
response, and $\psi$ is the cumulant generating function in the underlying
exponential family, e.g., \smash{$\psi(u) = \frac{1}{2} u^2$} for linear
regression and $\psi(u) = (b-a) \log(1 + e^u) + a u$ for generalized logistic
regression. The GD updates \eqref{eq:grad_descent} for a GLM, assuming as usual
constant step sizes $\eta_t = \eta > 0$, for all $t$, are    
\begin{equation}
\label{eq:grad_descent_glm}
\theta_{t+1} = \theta_t + \eta x_t (y_t - \psi'(x_t^\T \theta_t)), \quad t =
1,2,3,\dots,  
\end{equation}
where $\psi'$ denotes the derivative of $\psi$. Something interesting occurs in
\eqref{eq:grad_descent_glm} when the features are orthogonal. Let $u_j$, $j =
1,\dots,d$ be an orthonormal basis for $\R^d$, and suppose that  
\[
x_t \in \{u_j\}_{j=1}^d, \quad t = 1,2,3,\dots,
\] 
i.e., the features only take values in this orthonormal set. We can then
reparametrize the GD iterates by 
\[
\theta_t = \sum_{j=1}^d \vartheta_{tj} u_j, \quad t = 1,2,3,\dots,
\]
where $\vartheta_{tj} = u_j^\T \theta_t$, for $j = 1,\dots,d$. Taking an inner
product on each side of \eqref{eq:grad_descent_glm} with $x_t$ gives     
\begin{equation}
\label{eq:grad_descent_decoupled}
\vartheta_{t+1, j_t} = \vartheta_{t, j_t} + \eta (y_t - \psi'(\vartheta_{t,
  j_t})), \quad t = 1,2,3,\dots,  
\end{equation}
where $j_t \in \{1,\dots,d\}$ is the index such that \smash{$x_t =
  u_{j_t}$}. Thus, gradient descent for GLM losses in the current setting
actually \emph{decouples} into $d$ separate GD processes, one per coordinate in
the new parametrization.    

These processes are truly decoupled in the sense that each coordinate
$\vartheta_{tj}$ is only updated when $x_t = u_j$, and not at any other
$t$. Hence to analyze GD in the current setting, we can simply analyze the GD
iterations for each coordinate separately, effectively reducing the analysis to
that of univariate GLMs, which was given in Corollaries
\ref{cor:squared_loss_quad_curv} and \ref{cor:logistic_loss_zero_curv}. The next  
result provides details, and its proof is given in Appendix
\ref{app:glm_loss_ortho_features}.   

\begin{corollary}
\label{cor:glm_loss_ortho_features} 
Let $\ell_t(\theta) = -y_t x_t^\T \theta + \psi(x_t^\T \theta)$. Let $u_j$,  
$j = 1,\dots,d$ be an orthonormal basis $\R^d$, and suppose that each
\smash{$x_t \in \{u_j\}_{j=1}^d$}. Denote $I_j(T) = \{t \leq T : j_t = j \}$,
and $T_j = |I_j(T)|$, for $j = 1,\dots,d$.  

\begin{enumerate}[label=(\alph*)] 
\item If \smash{$\psi(u) = \frac{1}{2} u^2$} (linear regression), and $|y_t|
  \leq b_t$, where $b_t$ is sublinear and nondecreasing, then for any $\delta
  \in (0,1)$, and $\eta \leq 2(1-\delta)/(1+\delta)^2$, we have for each $j = 
  1,\dots,d$,   
  \begin{equation}
  \label{eq:glm_loss_linear_avg_grad_bound}
  \bigg| \frac{1}{T_j} \sum_{t \in I_j(T)} \vartheta_{tj} - \frac{1}{T_j}
  \sum_{t \in I_j(T)} y_t \bigg| \leq \frac{2 |\vartheta_{1j}| + b_T \eta (1 +
    1/\delta) + b_T/\delta}{\eta T_j}.
  \end{equation} 

\item If $\psi(u) = (b-a) \log(1 + e^u) + au$ (generalized logistic),
  and $y_t \in [a+\epsilon_t, b-\epsilon_t]$, where $\epsilon_t \in (0,
  (b-a)/2)$, is subexponentially vanishing in the sense that $1/\epsilon_t =
  \exp(o(t))$, and is nonincreasing, then we have for each $j = 1,\dots,d$,    
  \begin{equation}
  \label{eq:glm_loss_logistic_avg_grad_bound}
  \bigg| \frac{1}{T_j} \sum_{t \in I_j(T)} \bigg( (b-a)
  \frac{e^{\vartheta_{tj}}}{1 + e^{\vartheta_{tj}}} + a \bigg) - \frac{1}{T_j}
  \sum_{t \in I_j(T)} y_t \bigg| \leq \frac{2 |\vartheta_{1j}| + (b-a) \eta  +
    \log((b-a)/(2\epsilon_t))}{\eta T_j}.  
  \end{equation} 
\end{enumerate}
\end{corollary}

An important special case of Corollary \ref{cor:glm_loss_ortho_features} is when
the features are indicators of group membership. This leads to guarantees
precisely as in \eqref{eq:squared_loss_bias_bound} and
\eqref{eq:logistic_loss_bias_bound} for each group, which we call
\emph{multigroup debiasing}. Given a black-box predictor $f_t$, feature $x_t$,
and response $y_t$, we can implement the GD iterations
\eqref{eq:grad_descent_glm} with \smash{$\psi(u) = \frac{1}{2} u^2$}, the 
residual $y_t - f_t(x_t)$ in place of $y_t$, and a group indicator vector $z_t
\in \R^d$ in place of $x_t$. This has $z_{tj} = 1$ if observation $t$ belongs to
group $j$ and 0 otherwise. Assuming that the groups are disjoint
(nonoverlapping), the desired orthogonality condition is met, and
\eqref{eq:glm_loss_linear_avg_grad_bound} implies (assuming also 
that each residual $y_t - f_t(x_t)$ is bounded in magnitude by $b$):  
\begin{equation}
\label{eq:squared_loss_groupwise_bias_bound}
\bigg| \frac{1}{T_j} \sum_{t \in I_j(T)}  (f_t(x_t) + z_t^\T \theta_t) -
\frac{1}{T_j} \sum_{t \in I_j(T)} y_t \bigg| \leq \frac{b \eta (1 + 1/\delta) +
  b/\delta}{\eta T_j},  
\end{equation} 
for each $j = 1,\dots,d$, where we set $\theta_1 = 0$ for
simplicity. Similarly, given a black-box probabilistic classifier $p_t$,
applying \eqref{eq:grad_descent_glm} with $\psi(u) = (b-a) \log(1 + e^u) + au$,
the residual $y_t - p_t(x_t)$ in place of $y_t$, and $z_t$ in place of $x_t$, 
implies (taking $a = -1$ and $b = 1$, and assuming each residual $y_t -
p_t(x_t)$ lies in $[-1+\epsilon, 1-\epsilon]$):
\begin{equation}
\label{eq:logistic_loss_groupwise_bias_bound}
\bigg| \frac{1}{T_j} \sum_{t \in I_j(T)} \bigg( p_t(x_t) +
\frac{2 e^{z_t^\T \theta}}{1 + e^{z_t^\T \theta}} -1 \bigg) - \frac{1}{T_j}  
\sum_{t \in I_j(T)} y_t \bigg| \leq \frac{2 \eta  + \log(1/\epsilon)}{\eta T_j},    
\end{equation} 
for each $j = 1,\dots,d$, where again we set $\theta_1 = 0$ for 
simplicity. The multigroup debiasing problem setting will be studied empirically
in Section \ref{sec:multigroup_debiasing}.     

\subsection{Connections to traditional conditions}

As a final part of this section on gradient descent theory, we investigate
connections between the restorative conditions used in the propositions and more
traditional conditions used in optimization and adjacent fields. Copied here for 
convenience, the key assumptions are as follows: 
\begin{alignat}{3}
\label{eq:zero_curvature2}
&\text{Proposition \ref{prop:restorative_zero_curv}}: \quad 
&&g_t(\theta)^\T \theta \geq 0, \quad 
&&\text{for all $\|\theta\|_2 > h_t$}, \\ 
\label{eq:pos_curvature2}
&\text{Proposition \ref{prop:restorative_pos_curv}}: \quad 
&&g_t(\theta)^\T \theta \geq \tfrac{\eta L^2}{2}, \quad &&
\text{for all $\|\theta\|_2 > h_t$}, \\  
\label{eq:quad_curvature2}
&\text{Proposition \ref{prop:restorative_quad_curv}}: \quad 
&&g_t(\theta)^\T \theta \geq \tfrac{\eta}{2} \|g_t(\theta)\|_2^2, \quad
&&\text{for all $\|\theta\|_2 > h_t$},   
\end{alignat}
These are reminiscent of the following conditions arising in convex analysis and 
monotone operator theory:        
\begin{alignat}{3}
\label{eq:monotone}
&\text{Monotonicity}: \quad
&&(g_t(\theta) - g_t(\theta'))^\T (\theta - \theta') \geq 0, 
\quad &&\text{for all $\theta,\theta'$}, \\
\label{eq:strongly_monotone}
&\text{Strong monotonicity}: \quad
&&(g_t(\theta) - g_t(\theta'))^\T (\theta - \theta') \geq
\alpha \|\theta - \theta'\|_2^2, \quad &&\text{for all $\theta,\theta'$}, \\
\label{eq:co_coercive}
&\text{Co-coercivity}: \quad
&&(g_t(\theta) - g_t(\theta'))^\T (\theta - \theta') \geq
\mu \|g_t(\theta) - g_t(\theta')\|_2^2, \quad &&\text{for
  all $\theta,\theta'$},    
\end{alignat}
where $\alpha,\mu > 0$ are constants. A standard result in convex analysis is
that, for differentiable $\ell_t$, monotonicity of the gradient
\eqref{eq:monotone} is equivalent to convexity of $\ell_t$, whereas strong
monotonicity \eqref{eq:strongly_monotone} is equivalent to strong convexity of   
$\ell_t$. Another important result is the following: for convex $\ell_t$,
co-coercivity of the gradient \eqref{eq:co_coercive} is equivalent to
Lipschitzness of $g_t$ \cite{baillon1977quelques}. This result is often referred
to as the \emph{Baillon-Haddad theorem} in the literature on monotone operator
theory \cite{bauschke2009baillon}.               

Immediately, we can see that conditions
\eqref{eq:zero_curvature2}--\eqref{eq:quad_curvature2} only assume lower
bounds on gradient inner products when $\|\theta\|_2 > h_t$, and not on the 
whole space, as in \eqref{eq:monotone}--\eqref{eq:co_coercive}. Thus in no way 
do Propositions
\ref{prop:restorative_zero_curv}--\ref{prop:restorative_quad_curv} require,  
e.g., convexity or strong convexity of $\ell_t$. On the other hand, the
curvature conditions \eqref{eq:zero_curvature2}--\eqref{eq:quad_curvature2} can
be seen as weaker, \emph{restricted} versions of
\eqref{eq:monotone}--\eqref{eq:co_coercive}, where ``restricted'' refers to the
fact that the conditions are only assumed to hold when $\|\theta\|_2 > h_t$. The 
next proposition, whose proof is given in Appendix
\ref{app:restricted_strongly_monotone}, makes this connection precise.  

\begin{proposition}
\label{prop:restricted_strongly_monotone}
Assume the restricted $\alpha_t$-strongly monotone condition:  
\begin{equation}
\label{eq:restricted_strongly_monotone}
(g_t(\theta) - g_t(0))^\T \theta \geq \alpha_t \|\theta\|_2^2, \quad \text{for 
  all $\|\theta\|_2 > h_t$, and all generalized subgradients $g_t(\theta)$}.    
\end{equation}
Note that this condition is met when $\ell_t$ is $\alpha_t$-strongly convex; but
\eqref{eq:restricted_strongly_monotone} is much weaker than strong convexity in
general. Let $b_t \geq \|g_t(0)\|_2$. Then the following holds.      

\begin{enumerate}[label=(\alph*)]
\item For any $h_t \geq b_t / \alpha_t$, $\ell_t$ satisfies the restorative
  condition with zero curvature \eqref{eq:zero_curvature2}.
  % from Proposition \ref{prop:restorative_zero_curv}.

\item For any $h_t \geq b_t / \alpha_t + \sqrt{\eta/2} \cdot L / \alpha_t$,
  $\ell_t$ satisfies the restorative condition with positive curvature
  \eqref{eq:pos_curvature2}.
  % from Proposition \ref{prop:restorative_pos_curv}.  
  
\item Let us additionally assume restricted $\beta_t$-co-coercivity and
  restricted $\beta_t$-smoothness conditions: 
  \begin{alignat}{2}
  \label{eq:restricted_co_coercive}
  \hspace{-5pt}
  \beta_t (g_t(\theta) - g_t(0))^\T \theta &\geq \|g_t(\theta) - g_t(0)\|_2^2,
  \quad &&\text{for all $\|\theta\|_2 > h_t$, and all generalized subgradients
    $g_t(\theta)$}, \\   
  \label{eq:restricted_grad_lipschitz}
  \|g_t(\theta) - g_t(0)\|_2 &\leq \beta_t \|\theta\|_2, \quad &&\text{for all  
  $\|\theta\|_2 > h_t$, and all generalized subgradients $g_t(\theta)$}. 
  \end{alignat}
  Note that these conditions are simultaneously met when $\ell_t$ is
  differentiable and convex with $\beta_t$-Lipschitz gradient; but
  \eqref{eq:restricted_co_coercive} and \eqref{eq:restricted_grad_lipschitz} are
  much weaker conditions in general. Then for any $\eta \leq 2/\beta_t$, and
  for any    
  \[
  h_t \geq \frac{b_t (1 + \eta\beta_t+ \sqrt{\eta/2})}{\alpha_t (1 -
    \eta\beta_t/2)},   
  \]
  $\ell_t$ satisfies the restorative condition with quadratic curvature
  \eqref{eq:quad_curvature2}.
  % from Proposition \ref{prop:restorative_quad_curv}.  
\end{enumerate}
\end{proposition}

We reiterate that the conditions used in Proposition
\ref{prop:restricted_strongly_monotone} are implied by standard conditions in
optimization: \eqref{eq:restricted_strongly_monotone} is implied by strong
convexity and \eqref{eq:restricted_co_coercive},
\eqref{eq:restricted_grad_lipschitz} are implied by Lipschitzness of the
gradient; however, the conditions in the proposition are weaker in general.  
% since they are only constrain how the gradient acts in between the origin and
% points $\theta$ far enough away.  

\begin{remark}
As a sanity check (and investigation into the sharpness of the results in
Proposition \ref{prop:restricted_strongly_monotone} in terms of their
dependence on problem parameters, such as $b_t$ and $\eta$), we can ask what the 
proposition says about squared loss, \smash{$\ell_t(\theta) = \frac{1}{2} (y_t -  
  \theta)^2$}. First, note that the gradient at the origin is $g_t(0) = y_t$, so 
we can take $b_t$ to be a bound on $|y_t|$. Second, it is easy to check that
$\ell_t$ is strongly convex with parameter $\alpha_t = 1$, and has a Lipschitz
gradient with parameter $\beta_t = 1$. Thus we can apply part (c) of
Proposition \ref{prop:restricted_strongly_monotone}, which says that the
restorative condition with quadratic curvature holds when $\eta \leq 2$ and 
\[
h_t \geq \frac{b_t (1 + \eta + \sqrt{\eta/2})}{1-\eta/2}.  
\]
Fix any $\delta \in (0,1)$; we can take $\eta \leq 2(1-\delta)/(1+\delta)^2$, 
and calculate $1-\eta/2 = \delta(3+\delta) / (1+\delta)^2 \geq \delta /
(1+\delta)$, which means $1/(1-\eta/2) \leq (1+\delta) / \delta$. Using this 
along with the simple upper bounds $\delta \leq 1$ and $\eta \leq 2$, we see   
that a sufficient condition for the above display is  
\[
h_t \geq \frac{b_t (1 + \eta + \sqrt{\eta/2}) (1+\delta)}{\delta} =
\frac{8b_t}{\delta}.  
\]
In other words, Proposition \ref{prop:restricted_strongly_monotone} part (c) 
implies that if $h_t \geq 8b_t/\delta$, then $\ell_t$ satisfies the restorative
condition with quadratic curvature. We can see that this is less sharp than the
direct analysis for squared loss from Corollary \ref{cor:squared_loss_quad_curv} 
by a constant factor of 8, but admits the same precise dependence on $b_t$ and
$\eta$ (via $\delta$).  
\end{remark}

\begin{remark}
Meanwhile, for quantile and logistic losses the results in Proposition
\ref{prop:restricted_strongly_monotone} cannot be applied. This is because
neither loss satisfies the restricted strong monotonicity property in 
\eqref{eq:restricted_strongly_monotone}. Nevertheless, as we saw via direct  
analyses in Corollaries \ref{cor:quantile_loss_zero_curv} and
\ref{cor:logistic_loss_zero_curv}, useful results can still be derived in these
cases, which emphasizes the fact that restricted strong monotonicity (or
strong convexity, as a simpler sufficient condition) is certainly not the only
avenue for establishing restorativeness of the negative gradient field.    
\end{remark}

\section{Regularization}
\label{sec:regularization}

We consider regularization applied to the loss sequence $\ell_t$, $t =
1,2,3,\dots$. At a high level, we will show that proper regularization can be  
used to modify the loss sequence so that it satisfies restorative properties
(when the original sequence does not), and therefore is amenable to applying 
gradient descent in order to achieve gradient equilibrium. Importantly, we
also show that this translates back to a statement about the original
equilibrium condition of interest (with respect to the original sequence), since 
regularization perturbs each gradient by a controlled amount. 

To fix notation, let $r$ be a regularizer, assumed to be finite and
subdifferentiable on all of $\R^d$. As in our analysis of gradient descent, in
the last section, we assume the same of each loss $\ell_t$. Define the
regularized loss sequence \smash{$\tilde\ell_t = \ell_t + r$, $t =
  1,2,3,\dots$}. Two common choices of regularizers are:     
\begin{itemize}
\item $r(\theta) = \lambda \|\theta\|_1$, often called a \emph{lasso} penalty
  \cite{tibshirani1996regression}; and 
\item $r(\theta) = \lambda \|\theta\|_2^2$, often called a \emph{ridge} penalty
  \cite{hoerl1970ridge}. 
\end{itemize}
In either case $\lambda \geq 0$ is a tuning parameter which controls the
strength of regularization. We first examine gradient equilibrium for a general 
regularized sequence, and then we analyze the lasso and ridge cases.   

\subsection{Gradient equilibrium revisited}

We investigate gradient equilibrium for \smash{$\tilde\ell_t$}, $t =
1,2,3,\dots$, and relate it back to the original sequence. A key fact is that
subgradients of \smash{$\tilde\ell_t$} can be related to those of $\ell_t$: our
generalized notion of subgradients allows us to take, as a subgradient 
\smash{$\tilde{g}_t(\theta)$} of \smash{$\tilde\ell_t$} at $\theta$,            
\begin{equation}
\label{eq:subgrad_decomp_reg}
\tilde{g}_t(\theta) = g_t(\theta) + g_r(\theta),
\end{equation}
where $g_t(\theta)$ is a subgradient of $\ell_t$ at $\theta$, and
$g_r(\theta)$ denotes a subgradient of $r$ at $\theta$. (Appendix
\ref{app:gen_subgrad} gives a proof of this fact.) Thus, for a sequence of  
subgradients, \smash{$\tilde{g}_t(\theta)$}, $t = 1,2,3\dots$, which are chosen
to satisfy \eqref{eq:subgrad_decomp_reg}, 
\[
\frac{1}{T} \sum_{t=1}^T \tilde{g}_t(\theta_t) = \frac{1}{T} \sum_{t=1}^T
g_t(\theta_t) + \frac{1}{T} \sum_{t=1}^T g_r(\theta_t).
\]
Rearranging, and using the triangle inequality, we get
\begin{equation}
\label{eq:original_avg_grad_bound}
\bigg\| \frac{1}{T} \sum_{t=1}^T g_t(\theta_t) \bigg\|_2 \leq \bigg\|
\frac{1}{T} \sum_{t=1}^T \tilde{g}_t(\theta_t) \bigg\|_2 + \frac{1}{T}
\sum_{t=1}^T \| g_r(\theta_t) \|_2. 
\end{equation}
If we apply gradient descent to \smash{$\tilde\ell_t$}, $t = 1,2,3,\dots$ then
first term can be made small provided that $r$ endows the regularized sequence
with the appropriate restorative properties. The second term can be made small
by a judicious choice of the amount of regularization (which is controlled by
$\lambda$ in the lasso and ridge penalties). The next two subsections follow
this general template.

\subsection{Logistic losses with lasso penalties}
\label{sec:logistic_lasso}

We consider generalized logistic losses with arbitrary features, $\ell_t(\theta)
= -y_t x_t^\T \theta + (b-a) \log(1 + e^{x_t^\T \theta}) + a x_t^\T \theta$,
where $x_t$ is a feature, and $y_t$ is a response lying in $[a,b]$. For
arbitrary $x_t \in \R^d$, this loss need not satisfy the restorative condition
\eqref{eq:restorative} for any $\phi(\theta) \geq 0$, even when $y_t$ is bounded
away from the endpoints $a,b$ of the given range. This is because    
\begin{equation}
\label{eq:logistic_grad_inner_prod}
g(\theta)^\T \theta = \bigg( (b-a) \frac{e^{x_t^\T \theta}}{1 + e^{x_t^\T
    \theta}} + a - y_t \bigg) x_t^\T \theta,
\end{equation}
and the leading factor $(b-a) e^{x_t^\T \theta} / (1 + e^{x_t^\T \theta}) + a -
y_t$ can be negative when $x_t^\T \theta$ is small and positive. 

An important realization is that the gradient inner product in 
\eqref{eq:logistic_grad_inner_prod} cannot be arbitrarily small. Next we give a 
finite lower bound on its infimum. The proof is given in Appendix
\ref{app:logistic_grad_inner_prod}.   

\begin{lemma}
\label{lem:logistic_grad_inner_prod}
Let $g(u) = (b-a) e^u / (1 + e^u) + a - y$, where $y \in [a,b]$, for $a<b$. Then 
\[
\inf_u \, g(u)^\T u \geq -0.279 (b-a).
\]
\end{lemma}

The finite infimum in Lemma \ref{lem:logistic_grad_inner_prod} suggests that an
appropriate level of $\ell_1$ regularization can ``overwhelm'' a possibly
negative gradient inner product in order to ensure suitable curvature, and thus
restorativeness, for the regularized loss. The proof of the next result is given
in Appendix \ref{app:logistic_lasso_pos_curv}. 

\begin{proposition}
\label{prop:logistic_lasso_pos_curv}
Let \smash{$\tilde\ell_t(\theta) = -y_t x_t^\T \theta + (b-a) \log(1 + e^{x_t^\T
    \theta}) + a x_t^\T \theta + \lambda \|\theta\|_1$}, for arbitrary $\lambda
> 0$, and assume $\|x_t\|_2 \leq c$, and $y \in [a,b]$. Then
\smash{$\tilde\ell_t$} is Lipschitz continuous with parameter \smash{$L = (b-a) 
  c + \lambda \sqrt{d}$}, and satisfies the $(h_t,\phi_t)$-restorative condition 
\eqref{eq:restorative} with positive curvature \eqref{eq:pos_curvature} for any         
\[
h_t \geq \frac{0.279 (b-a) + \eta (c (b-a) + \lambda \sqrt{d})^2/2}{\lambda}.   
\]
Taking $h_t$ to match this lower bound, we have from
\eqref{eq:gd_avg_grad_bound_pos_curv} that gradient descent 
\eqref{eq:grad_descent}, with \smash{$\tilde{g}_t$} in place of $g_t$ and
constant step sizes $\eta_t = \eta > 0$, for all $t$, satisfies   
\begin{equation}
\label{eq:logistic_lasso_avg_grad_bound}
\bigg\| \frac{1}{T} \sum_{t=1}^T \bigg( (b-a) \frac{e^{x_t^\T\theta_t}}{1 + 
  e^{x_t^\T\theta_t}} + a - y_t \bigg) x_t \bigg\|_2 \leq \frac{C_1(\lambda)}
{\eta T} + \frac{C_2(\lambda)}{\lambda \eta T} + \lambda \sqrt{d}, 
\end{equation}
where \smash{$C_1(\lambda) = 2 \|\theta_1\|_2 + (c (b-a) + \lambda 
  \sqrt{d}) / \eta$}, and \smash{$C_2(\lambda) = 0.279 (b-a) + \eta (c
  (b-a) + \lambda \sqrt{d})^2 / 2$}.   
\end{proposition}

\begin{remark}
In this proposition, note we do not need $y_t$ to be bounded away from $a,b$,
the endpoints of its range (as we did in the unregularized, univariate logistic
theory in Corollary \ref{cor:logistic_loss_zero_curv}, or the theory for
orthogonal features in Corollary
\ref{cor:glm_loss_ortho_features}). Furthermore, a lasso penalty is chosen
(instead of, say, a ridge penalty) here so that the regularized loss retains the
Lipschitz property of the original logistic loss, which allows us to apply the
positive curvature result from Proposition \ref{prop:restorative_pos_curv} (and
allows the step size $\eta > 0$ to be arbitrary).      
\end{remark}

Proposition \ref{prop:logistic_lasso_pos_curv} can be used to derive guarantees
on a post-processing scheme for a black-box probabilistic classifier, where we
use gradient descent to decorrelate the errors with any features of
interest. Precisely, let $p_t(x_t)$ be a probabilistic prediction of a binary
response $y_t \in \{0,1\}$, based on a feature $x_t$. Replace $y_t$ in the
proposition by $y_t - p_t(x_t)$, and replace $x_t$ by another feature $z_t \in
\R^d$. To adjust predictions, we use: 
\begin{equation}
\label{eq:logistic_loss_adj_pred}
p_t(x_t) + \frac{2 e^{z_t^\T \theta_t}}{1 + e^{z_t^\T \theta_t}} - 1, \quad t =
1,2,3,\dots, 
\end{equation}
where $\theta_t$, $t = 1,2,3,\dots$ is obtained by gradient descent on the   
lasso-penalized logistic loss sequence (and we set $a = -1$ and $b = 1$ in the
language of the proposition). If each $\|z_t\|_2 \leq c$, then the result in  
\eqref{eq:logistic_lasso_avg_grad_bound} says:
\begin{equation}
\label{eq:logistic_lasso_covariance_bound}
\bigg\| \frac{1}{T} \sum_{t=1}^T \bigg( p_t(x_t) + \frac{2 e^{z_t^\T
  \theta_t}}{1 + e^{z_t^\T \theta_t}} - 1 - y_t \bigg) z_t \bigg\|_2 \leq
\frac{2c + \lambda \sqrt{d}}{T} + \frac{1.116 + \eta (2c + \lambda
  \sqrt{d})^2}{2 \lambda \eta T} + \lambda \sqrt{d},
\end{equation}
where we set $\theta_1 = 0$ for simplicity. The bound in
\eqref{eq:logistic_lasso_covariance_bound} can be made small by taking
$\lambda$ to be small and $T$ to be large. For example, taking \smash{$\lambda =
  1 / \sqrt{T}$} to approximately balance the second and third terms gives 
\[
\bigg\| \frac{1}{T} \sum_{t=1}^T \bigg( p_t(x_t) + \frac{2 e^{z_t^\T
    \theta_t}}{1 + e^{z_t^\T \theta_t}} - 1 - y_t \bigg) z_t \bigg\|_2 \leq 
O\bigg( \frac{1}{\sqrt{T}} \bigg),
\]
where we have simplified the bound to highlight the dependence on $T$ alone. 

As an example, we can apply this post-processing strategy in a setting in which
$z_t$ indicates membership with respect to a set of overlapping groups. In this
setting (with overlap), the features will not be orthogonal and so the previous
theory (Corollary \ref{cor:glm_loss_ortho_features} part (b)) does not
apply. In Sections \ref{sec:multigroup_debiasing} and
\ref{sec:pairwise_preference}, we study empirical examples with overlapping
groups and regularization. We must emphasize, however, that the post-processing
strategy described above accommodates a much broader set of applications, as
it allows the user to specify the secondary feature set $z_t$ arbitrarily. We
return to this point in Section \ref{sec:multiaccuracy}.

\subsection{Squared losses with ridge penalties}
\label{sec:squared_ridge}

We now consider squared losses with arbitrary features, \smash{$\ell_t(\theta) =  
\frac{1}{2} (y_t - x_t^\T \theta)^2$}, where $x_t$ is a feature, and $y_t$ is a
response with $|y_t| \leq b_t$. For arbitrary $x_t \in \R^d$, this loss need not
satisfy the restorative property \eqref{eq:restorative} for any $\phi(\theta)
\geq 0$. This is because   
\begin{equation}
\label{eq:squared_grad_inner_prod}
g(\theta)^\T \theta = (x_t^T \theta - y_t) x_t^\T \theta,
\end{equation}
and the leading factor $x_t^T \theta - y_t$ can be negative when $x_t^\T \theta$
is small and positive.  

Similar to the earlier logistic setting, a key realization is that the gradient
inner product \eqref{eq:squared_grad_inner_prod} cannot be arbitrarily small,
and is subject to a finite lower bound. The proof of the next result is in
Appendix \ref{app:squared_grad_inner_prod}.    

\begin{lemma}
\label{lem:squared_grad_inner_prod}
Let $g(u) = u - y$, where $|y| \leq b$. Then for arbitrary $a_1 > a_2$, 
\[
\inf_u \, \Big\{ a_1 g(u)^\T u - a_2 (u-y)^2 \Big\} \geq - \frac{(a_1 - 2 a_2)^2
  b^2}{4 (a_1 - a_2)} - a_2 b^2,  
\]
\end{lemma}

Again, in a similar vein to our earlier  analysis of logistic regression, the
finite infimum in Lemma \ref{lem:squared_grad_inner_prod} means that we can use
an appropriate level of regularization to ``overwhelm'' a potentially negative
gradient inner product, thus ensuring proper curvature and restorativeness for
the regularized loss. (The generalized form of the criterion, involving $a_1,a_2
\geq 0$, will allow us to verify quadratic curvature.) The proof of the next
result is given in Appendix \ref{app:squared_ridge_quad_curv}.      

\begin{proposition}
\label{prop:squared_ridge_quad_curv}
Let \smash{$\tilde\ell_t(\theta) = \frac{1}{2}(y_t - x_t^\T \theta)^2 +
\frac{\lambda}{2} \|\theta\|_2^2$}, for arbitrary $\lambda > 0$, and assume
$|y_t| \leq b_t$ and $\|x_t\|_2 \leq c_t$ for sublinear and nondecreasing
sequences $b_t, c_t$. Then \smash{$\tilde\ell_t$} is Lipschitz continuous on $\{
\theta \in \R^d : \|\theta\|_2 \leq h_t \}$ with parameter $L_t = (c_t^2 +
\lambda) h_t + b_t c_t$, and satisfies the $(h_t,\phi_t)$-restorative condition
\eqref{eq:restorative} with quadratic curvature \eqref{eq:quad_curvature} for
any $\eta < 1/(\lambda + c_t^2 / 2)$, and     
\[
h_t^2 \geq \frac{C_{0t}(\lambda)}{\lambda},  
\]
where $C_{0t}(\lambda) = ((1 - \lambda \eta - \eta c_t^2)^2 b_t^2 / (4 (1 -
\lambda \eta - \eta c_t^2 / 2)) + \eta c_t^2 b_t^2 / 2) / (1 - \lambda \eta /
2)$. Taking $h_t$ to match this lower bound, we have from
\eqref{eq:gd_avg_grad_bound_quad_curv} that gradient descent
\eqref{eq:grad_descent}, with \smash{$\tilde{g}_t$} in place of $g_t$ and
constant step sizes $\eta_t = \eta > 0$, for all $t$, satisfies   
\begin{equation}
\label{eq:squared_ridge_avg_grad_bound}
\bigg\| \frac{1}{T} \sum_{t=1}^T (x_t^T \theta_t - y_t) x_t \bigg\|_2 \leq 
\frac{C_{1T}(\lambda)}{\eta T} + \frac{C_{2T}(\lambda)}{\sqrt\lambda \eta T} + 
C_{3T}(\lambda), 
\end{equation}
where $C_{1T}(\lambda) = 2 \|\theta_1\|_2 + \eta b_T c_T$,
$C_{2T}(\lambda) = \sqrt{C_{0T}(\lambda)}(\eta (c_T^2 + \lambda) + 1)$, and
$C_{3T}(\lambda) = \sqrt{\lambda} C_{2T}(\lambda) + \lambda (\|\theta_1\|_2 +
\eta b_T c_T)$.  
\end{proposition}

\begin{remark}
The use of a ridge penalty (instead of, say, a lasso penalty) is important here
since it leads to quadratic curvature, which allows us to apply the result from
Proposition \ref{prop:restorative_quad_curv}. This type of curvature is needed
because the squared loss is not Lipschitz, only locally Lipschitz, so other
results (in particular, Proposition \ref{prop:restorative_pos_curv}) cannot be
applied. Moreover, we note that relying on strong convexity alone to establish 
restorativeness (as in Proposition \ref{prop:restricted_strongly_monotone})
would be too blunt to deliver the result in
\eqref{eq:squared_ridge_avg_grad_bound}. Deriving this requires a more
specialized analysis, which relies on the fact that the gradient inner product  
\eqref{eq:squared_grad_inner_prod} for squared loss has a finite infimum. It is  
possible that a similar analysis could be extended to cover all GLMs, but we do
not pursue this.  
\end{remark}

Proposition \ref{prop:squared_ridge_quad_curv} can be used to derive guarantees 
on a scheme which post-processes a black-box prediction model, using gradient
descent to decorrelate the errors with any given features of interest. Let
$f_t(x_t)$ be the prediction of a response $y_t$, based on a feature $x_t$. 
Replace $y_t$ in the proposition by the residual $y_t - f_t(x_t)$, and $x_t$ by
a secondary feature $z_t \in \R^d$. We adjust predictions using:   
\begin{equation}
\label{eq:squared_loss_adj_pred}
f_t(x_t) + \theta_t, \quad t = 1,2,3,\dots,
\end{equation}
where $\theta_t$, $t = 1,2,3,\dots$ is obtained by gradient descent on the 
ridge-penalized squared loss sequence. If each $|y_t - f_t(x_t)| \leq b$ and
$\|z_t\|_2 \leq c$, then the result in \eqref{eq:squared_ridge_avg_grad_bound}
says:     
\begin{equation}
\label{eq:squared_ridge_covariance_bound}
\bigg\| \frac{1}{T} \sum_{t=1}^T (f_t(x_t) + z_t^T \theta_t - y_t) z_t \bigg\|_2
\leq \frac{bc}{T} + \frac{\sqrt{C(\lambda)}}{\sqrt\lambda \eta T} +
\sqrt{\lambda C(\lambda)} + \lambda \eta bc,  
\end{equation}
where $C(\lambda) = [((1 - \lambda \eta - \eta c^2)^2 b^2 / (4 (1 - \lambda \eta
- \eta c^2 / 2)) + \eta c^2 b^2 / 2) / (1 - \lambda \eta / 2)] \cdot (\eta (c^2
+  \lambda) + 1)$, and we set $\theta_1 = 0$ for simplicity. The bound in 
\eqref{eq:squared_ridge_covariance_bound} can be made small by taking $\lambda$ 
to be small and $T$ to be large. For example, taking $\lambda = 1/T$ to
approximately balance the third and fourth terms gives   
\[
\bigg\| \frac{1}{T} \sum_{t=1}^T (f_t(x_t) + z_t^T \theta_t - y_t) z_t \bigg\|_2 
\leq O\bigg( \frac{1}{\sqrt{T}} \bigg),
\]
where we have simplified the bound to highlight the dependence on $T$. 

As in the logistic setting, we note that we can apply such a post-processing
scheme to achieve groupwise debiased predictions in a setting with overlapping
groups. This is important as it expands our current set of available tools---the
features will not be orthogonal (due to overlap), so the previous theory
(Corollary \ref{cor:glm_loss_ortho_features} part (a)) cannot be applied.
However, the result in \eqref{eq:squared_ridge_covariance_bound} is considerably
broader, because the secondary feature $z_t$ can be specified arbitrarily. Next,
we discuss implications for multiaccuracy.

\subsection{Implications for multiaccuracy}
\label{sec:multiaccuracy}

We show how the decorrelation results 
\eqref{eq:logistic_lasso_covariance_bound},
\eqref{eq:squared_ridge_covariance_bound} can be used to derive an online method
with multiaccuracy guarantees. Let $\cF$ be a linear, finite-dimensional class
of real-valued functions on the space of features. We can write any function $F
\in \cF$ as  
\[
F(x) = \sum_{j=1}^d \alpha_j b_j(x), 
\]
for coefficients $\alpha_j \in \R$, $j = 1,\dots,d$ and basis functions $b_j$, 
$j = 1,\dots,d$. Abbreviate the basis expansion at $x$ by $F(x) = b(x)^\T
\alpha$. With $z_t = b(x_t)$ for $t = 1,2,3,\dots$, consider the logistic 
post-processing scheme following Proposition \ref{prop:logistic_lasso_pos_curv},
where we adjust the predictions from a black-box probabilistic classifier $p_t$
as in \eqref{eq:logistic_loss_adj_pred}. Observe  
\begin{align*}
\bigg| \frac{1}{T} \sum_{t=1}^T \bigg( p_t(x_t) + \frac{2 e^{z_t^\T \theta_t}}{1
  + e^{z_t^\T \theta_t}} - 1 - y_t \bigg) F(x_t) \bigg| 
&= \bigg| \frac{1}{T} \sum_{t=1}^T \bigg( p_t(x_t) + 
\frac{2 e^{z_t^\T \theta_t}}{1 + e^{z_t^\T \theta_t}} - 1 - y_t \bigg)
z_t^\T \alpha \bigg| \\
&\leq \bigg\| \frac{1}{T} \sum_{t=1}^T \bigg( p_t(x_t) + 
\frac{2 e^{z_t^\T \theta_t}}{1+ e^{z_t^\T  \theta_t}} - 1 - y_t \bigg) z_t
\bigg\|_2 \|\alpha\|_2, 
\end{align*}
where the first line uses $F(x_t) = z_t^\T \alpha$ for each $t$, and the second
line uses Cauchy-Schwarz. Thus, introducing the norm $\|F\| = \|\alpha\|_2$ on
$\cF$, and assuming $\|b(x)\|_2 \leq c$ for all $x$, the result in
\eqref{eq:logistic_lasso_covariance_bound} implies              
\begin{equation}
\label{eq:logistic_lasso_multiaccuracy_bound}
\sup_{F \in \cF, \, \|F\| \leq 1} \, \bigg| \frac{1}{T} \sum_{t=1}^T \bigg(
p_t(x_t) + \frac{2 e^{z_t^\T \theta_t}}{1 + e^{z_t^\T \theta_t}} - 1 - y_t
\bigg) F(x_t) \bigg| \leq \frac{2c + \lambda \sqrt{d}}{T} + \frac{1.116 + \eta
  (2c + \lambda \sqrt{d})^2}{2 \lambda \eta T} + \lambda \sqrt{d}.
\end{equation}
Similarly, if we apply the squared loss post-processing scheme following
Proposition \ref{prop:squared_ridge_quad_curv}, where we adjust the predictions
from a black box-model $f_t$ as in \eqref{eq:squared_loss_adj_pred}, then the
result in \eqref{eq:squared_ridge_covariance_bound} implies
\begin{equation}
\label{eq:squared_ridge_multiaccuracy_bound}
\sup_{F \in \cF, \, \|F\| \leq 1} \, \bigg| \frac{1}{T} \sum_{t=1}^T
(f_t(x_t) + z_t^T \theta_t - y_t) F(x_t) \bigg| \leq \frac{bc}{T} +
\frac{\sqrt{C(\lambda)}}{\sqrt\lambda \eta T} + \sqrt{\lambda C(\lambda)} + 
\lambda \eta bc,    
\end{equation}
where as before $C(\lambda) = [((1 - \lambda \eta - \eta c^2)^2 b^2 / (4 (1 -
\lambda \eta - \eta c^2 / 2)) + \eta c^2 b^2 / 2) / (1 - \lambda \eta / 2)]
\cdot (\eta (c^2 + \lambda) + 1)$.

Each of \eqref{eq:logistic_lasso_multiaccuracy_bound}, 
\eqref{eq:squared_ridge_multiaccuracy_bound} are online multiaccuracy bounds.
They apply to a relatively restricted function class $\cF$ (linear and 
finite-dimensional), but they make no assumptions about the data or the
original predictors $p_t$ or $f_t$ whatsoever (other than boundedness), i.e.,
they apply deterministically. The post-processing scheme based on gradient
descent is also very simple to implement practically. See Section
\ref{sec:related_work} for a discussion of related work in the context of
multiaccuracy.   

\section{Arbitrary step sizes}
\label{sec:arbitrary_steps}

We extend our theory to the case of decaying step sizes. As in the last two
sections, we assume that each loss $\ell_t$ is finite and subdifferentiable on
all of $\R^d$. Core to our extension is a modification of Proposition
\ref{prop:gd_simple}.  

\begin{proposition}
\label{prop:gd_simple_lrt}
Consider gradient descent \eqref{eq:grad_descent}, with arbitrary initialization
$\theta_1 \in \R^d$, and abitrary positive step sizes $\eta_t > 0$, $t =
1,2,3,\dots$. For convenience, let $\eta_0^{-1} = 0$, and define 
\[
\Delta_t = \eta_t^{-1}-\eta_{t-1}^{-1}, \quad t = 1,2,3,\dots,
\]    
The average gradient satisfies
\begin{equation}
\label{eq:gd_avg_grad_identity_lrt}
\frac{1}{T} \sum_{t=1}^T g_t(\theta_t) =  \frac{1}{T} \sum_{t=1}^T (\theta_t - 
\theta_{T+1}) \Delta_t,
\end{equation}
and therefore
\begin{equation}
\label{eq:gd_avg_grad_bound_lrt}
\bigg\| \frac{1}{T} \sum_{t=1}^T g_t(\theta_t) \bigg\|_2 \leq \frac{2}{T} \Big(
\max_{t \leq T+1} \, \|\theta_t \|_2 \Big) \sum_{t=1}^T |\Delta_t|.
\end{equation}
When the step sizes $\eta_t$, $t = 1,2,3,\dots$ are nonincreasing,  
\begin{equation}
\label{eq:gd_avg_grad_bound_dec}
\bigg\| \frac{1}{T} \sum_{t=1}^T g_t(\theta_t) \bigg\|_2 \leq \frac{2}{\eta_T T}
\Big( \max_{t \leq T+1} \, \|\theta_t \|_2 \Big).
\end{equation}
\end{proposition}

\begin{proof}
Rewrite the iteration \eqref{eq:grad_descent} as $\theta_t - \theta_{t+1} =
\eta_t g_t(\theta_t)$. Adding this up over $t = s,\dots,T$, the left-hand side
telescopes, yielding
\[
\theta_s - \theta_{T+1} = \sum_{t = s}^T \eta_t g_t(\theta_t)
\]
Next, for all $t$, we have 
\[
\sum_{s=1}^t \Delta_s = \eta_t^{-1}.
\]
This allows us to write 
\begin{align*}
\frac{1}{T} \sum_{t=1}^T g_t(\theta_t) 
&= \frac{1}{T} \sum_{t=1}^T \sum_{s=1}^t \Delta_s \eta_t g_t(\theta_t) \\ 
&= \frac{1}{T} \sum_{s=1}^T \Delta_s  \sum_{t=s}^T \eta_t g_t(\theta_t) \\ 
&= \frac{1}{T} \sum_{s=1}^T \Delta_s (\theta_s - \theta_{T+1}).
\end{align*}
This proves \eqref{eq:gd_avg_grad_identity_lrt}. Furthermore, writing $\Delta = 
(\Delta_1,\dots,\Delta_T) \in \R^d$, 
\begin{align*}
\bigg\| \frac{1}{T} \sum_{t=1}^T g_t(\theta_t) \bigg\|_2
&\leq \frac{1}{T} \sum_{s=1}^T |\Delta_s| \|\theta_s - \theta_{T+1}\|_2 \\
&\leq \frac{1}{T} \|\Delta\|_1 \Big( \max_{s \leq T} \, \|\theta_s -
  \theta_{T+1}\|_2 \Big) \\
&\leq \frac{2}{T} \|\Delta\|_1 \Big( \max_{s \leq T+1} \, \|\theta_s \|_2
  \Big),
\end{align*}
which proves \eqref{eq:gd_avg_grad_bound_lrt}. Lastly, when the step sizes are 
nonincreasing, 
\[
\|\Delta\|_1 = \eta_1^{-1} + \sum_{t=1}^T (\eta_t^{-1} - \eta_{t-1}^{-1}) =
\eta_T^{-1},
\]
which verifies \eqref{eq:gd_avg_grad_bound_dec}, and completes the proof. 
\end{proof}

Observe that the key term in either of \eqref{eq:gd_avg_grad_bound_lrt},
\eqref{eq:gd_avg_grad_bound_dec} is the maximum iterate norm \smash{$\max_{t 
\leq T+1} \|\theta_t\|_2$}. The next result summarizes when these bounds lead to
gradient equilibrium.

\begin{proposition}
\label{prop:gd_sublinear_lrt}
For gradient descent \eqref{eq:grad_descent} with arbitrary positive
sizes, gradient equilibrium \eqref{eq:grad_eq} holds if the iterates satisfy
\smash{$\max_{t \leq T+1} \|\theta_t\|_2 = o(T / \sum_{t=1}^T |\Delta_t|)$}. For
nonincreasing step sizes, this condition simplifies to \smash{$\max_{t \leq T+1} 
  \|\theta_t\|_2 = o(\eta_T T)$}. When the step sizes decrease slowly enough in
order for $\eta_t t$, $t = 1,2,3,\dots$ to be nonincreasing, it further
simplifies to $\|\theta_{t+1}\|_2 = o(\eta_t t)$. 
\end{proposition}

Fortunately, the restorative theory established in Section
\ref{sec:restorative}, which provides control over iterate norms, carries over
to arbitrary step sizes. The next result extends Propositions
\ref{prop:restorative_zero_curv_1d}--\ref{prop:restorative_quad_curv}.
Its proof is basically identical to that of these propositions, and is hence 
omitted. 

\begin{proposition}
\label{prop:restorative_lrt}
Consider gradient descent \eqref{eq:grad_descent}, with arbitrary initialization 
$\theta_1 \in \R^d$, and abitrary positive step sizes $\eta_t > 0$, $t =
1,2,3,\dots$. Abbreviate \smash{$N_t = \max_{s \leq t} \eta_s$}, and let $c > 0$
and $\alpha \in (0,1)$ be arbitrary.

\begin{enumerate}[label=(\alph*)]
\item If $d = 1$, and each $\ell_t$ is $L$-Lipschitz and $(h_t, 0)$-restorative,
  for nondecreasing $h_t$, then
  \begin{equation}
  \label{eq:gd_iterate_bound_zero_curv_1d_lrt}
  |\theta_{T+1}| \leq \max\{ |\theta_1|, h_T \} + N_T L. 
  \end{equation}
  Thus if we set each $\eta_t = c t^{-\alpha}$, and $h_t = o(t^{1-\alpha})$,
  then \eqref{eq:gd_avg_grad_bound_dec} implies gradient equilibrium:    
  \begin{equation}
  \label{eq:gd_avg_grad_bound_zero_curv_1d_lrt}
  \bigg| \frac{1}{T} \sum_{t=1}^T g_t(\theta_t) \bigg| \leq \frac{2
    |\theta_1|/c}{T^{1-\alpha}} + \frac{2(L + h_T/c)}{T^{1-\alpha}} \to 0, 
  \quad \text{as $T \to \infty$}.    
  \end{equation}

\item For general $d$, if each $\ell_t$ is $L$-Lipschitz and $(h_t,
  0)$-restorative, then   
  \begin{equation}
  \label{eq:gd_iterate_bound_zero_curv_lrt}
  \|\theta_{T+1}\|_2 \leq \sqrt{\|\theta_1\|_2^2 + \sum_{t=1}^T (\eta_t^2 L^2 +
    2 \eta_t h_t L)}.   
  \end{equation}
  Thus if we set each $\eta_t = c t^{-\alpha}$, and $h_t = o(t^{1-\alpha})$,
  then \eqref{eq:gd_avg_grad_bound_dec} implies gradient equilibrium:              
  \begin{equation}
  \label{eq:gd_avg_grad_bound_zero_curv_lrt}
  \bigg\| \frac{1}{T} \sum_{t=1}^T g_t(\theta_t) \bigg\|_2 \leq \frac{2
    \|\theta_1\|_2/c}{T^{1-\alpha}} + \sqrt{\sum_{t=1}^T \frac{t^{-2 \alpha} L^2
    + 2 t^{-\alpha} h_t L/c}{T^{2(1-\alpha)}}} \to 0, \quad \text{as $T \to
  \infty$}.        
  \end{equation}

\item If each $\ell_t$ is $L$-Lipschitz and $(h_t,\phi_t)$-restorative with
  positive curvature: $\phi_t(\theta) \geq \eta_t L^2 / 2$, for nondecreasing 
  $h_t$, then 
  \begin{equation}
  \label{eq:gd_iterate_bound_pos_curv_lrt}
  \|\theta_{T+1}\|_2 \leq \max\{ \|\theta_1\|_2, h_T \} + N_T L.    
  \end{equation}
  Thus if we set each $\eta_t = c t^{-\alpha}$, and $h_t = o(t^{1-\alpha})$,
  then \eqref{eq:gd_avg_grad_bound_dec} implies gradient equilibrium:    
  \begin{equation}
  \label{eq:gd_avg_grad_bound_pos_curv_lrt}
  \bigg\| \frac{1}{T} \sum_{t=1}^T g_t(\theta_t) \bigg\|_2 \leq \frac{2
    \|\theta_1\|_2/c}{T^{1-\alpha}} + \frac{2(L + h_T/c)}{T^{1-\alpha}} \to 0,
  \quad \text{as $T \to \infty$}.   
  \end{equation}

\item If each $\ell_t$ is $L_t$-Lipschitz on the set $\{ \theta \in \R^d :
  \|\theta\|_2 \leq h_t \}$, and also $(h_t,\phi_t)$-restorative with quadratic
  curvature: $\phi_t(\theta) \geq \eta_t^2 \|g_t(\theta)\|_2^2 / 2$, for
  nondecreasing $h_t,L_t$, then   
  \begin{equation}
  \label{eq:gd_iterate_bound_quad_curv_lrt}
  \|\theta_{T+1}\|_2 \leq \max\{ \|\theta_1\|_2, h_T \} + N_T L_T.   
  \end{equation}
  Thus if we set each $\eta_t = c t^{-\alpha}$, and $h_t = o(t^{1-\alpha})$,
  $L_t = o(t^{1-\alpha})$, then \eqref{eq:gd_avg_grad_bound_dec} implies
  gradient equilibrium:     
  \begin{equation}
  \label{eq:gd_avg_grad_bound_quad_curv_lrt}
  \bigg\| \frac{1}{T} \sum_{t=1}^T g_t(\theta_t) \bigg\|_2 \leq \frac{2
    \|\theta_1\|_2/c}{T^{1-\alpha}} + \frac{2(L_T + h_T/c)}{T^{1-\alpha}} \to 0,
  \quad \text{as $T \to \infty$}.     
  \end{equation}
\end{enumerate}
\end{proposition}

\begin{remark}
In the gradient equilibrium results
\eqref{eq:gd_avg_grad_bound_zero_curv_1d_lrt}, 
\eqref{eq:gd_avg_grad_bound_zero_curv_lrt},
\eqref{eq:gd_avg_grad_bound_pos_curv_lrt}, 
\eqref{eq:gd_avg_grad_bound_quad_curv_lrt} in the proposition, we have focused
on decaying step sizes of the form $\eta_t = c t^{-\alpha}$, in order to make
the conclusions salient. However, we note that the iterate bounds
\eqref{eq:gd_iterate_bound_zero_curv_1d_lrt},
\eqref{eq:gd_iterate_bound_zero_curv_lrt}, 
\eqref{eq:gd_iterate_bound_pos_curv_lrt}, 
\eqref{eq:gd_iterate_bound_quad_curv_lrt} combined with
\eqref{eq:gd_avg_grad_bound_lrt} will lead to gradient equilibrium guarantees
for a broader class of step size schedules, including adaptive ones. 
\end{remark}

\begin{remark}
Another reason to focus on decaying step sizes is that it allows us to fluidly
draw in standard, parallel results on regret. For example, by Theorem 3.1 in
\cite{hazan2016introduction}, we know that online projected gradient descent for
$L$-Lipschitz losses $\ell_t$, $t = 1,2,3,\dots$, with step sizes $\eta = D / (L
\sqrt{t})$, $t = 1,2,3,\dots$, where $D$ is the diameter of the constraint set,
leads to an average regret bound
\[
\frac{\Regret_T}{T} \leq \frac{3}{2} \frac{DL}{\sqrt{T}} \to 0, \quad \text{as
  $T \to \infty$}.    
\]
Our theory does not require bounded constraint sets, but it does extend to
constraints and projections; see Appendix \ref{app:constraints}. The topline
conclusion is that gradient equilibrium and no regret are not at odds, and 
gradient descent with decaying step sizes can achieve them both simultaneously. 
A secondary conclusion is that our restorative theory may be itself helpful for 
regret analysis in unbounded domains, by virtue of its control over the iterate 
norms in gradient descent. 
\end{remark}

\section{Experiments}
\label{sec:experiments}

In this section, we present a set of worked examples. These work together to
showcase the generality of the gradient equilibrium framework laid out in
previous sections. 
See \url{https://github.com/aangelopoulos/gradient-equilibrium} for code to reproduce these experiments.

Generally, our examples will follow a common structure: we
define a base model (which may be a pre-trained or trained online), and a
parameter sequence $\theta_t$, $t = 1,2,3,\dots$ that we will learn
online. Then, we define a loss function over these parameters and state a
resulting gradient equilibrium guarantee. Finally, we examine a given dataset
and present application-specific plots and interpretations to showcase the
utility of the method.

Henceforth we will denote by $f_t$, $t = 1,2,3,\dots$ a sequence of predictions
of a ground-truth label sequence $y_t$, $t = 1,2,3,\dots$. Also, $z_t$, $t =
1,2,3,\dots$ denote features of interest, which are not necessarily the same as
the ones used to make the predictions. They could be, for example,
societally-meaningful covariates such as age, race, sex, etc. Also, we
note that in many of our experiments, we test the following ad hoc settings of
the learning rate: $\{0, 0.001, 0.01, 0.05\}$. While there is a substantial 
literature on choosing learning rates in optimization, we do not pursue such
methods here.      

\subsection{Datasets}

\begin{figure}[htb]
\centering
\includegraphics[width=\linewidth]{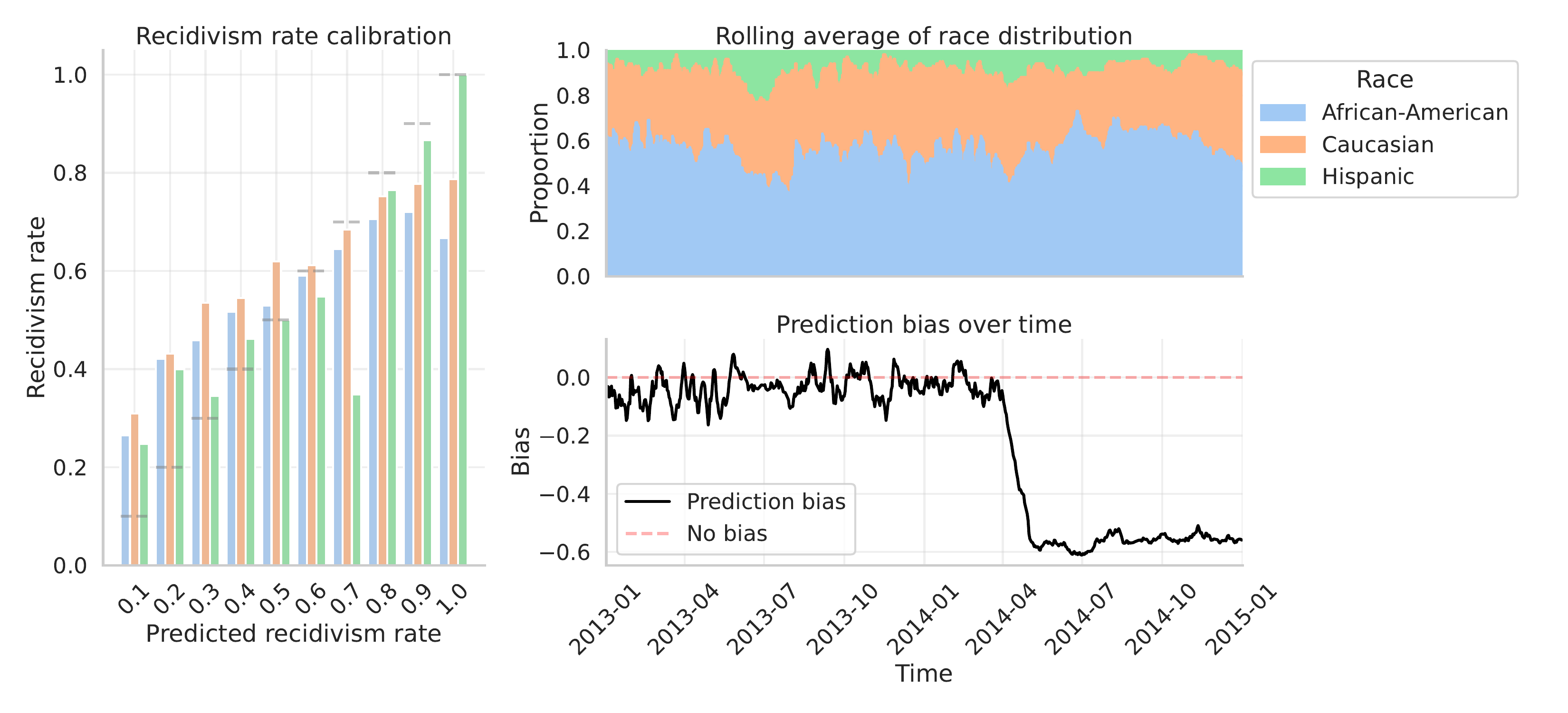}
\caption{\it Statistics of the COMPAS dataset. On the left-hand side is a 
  calibration plot showing the predicted recidivism rate on the horizontal axis
  and the true recidivism rate---conditionally on the predicted one---on the
  vertical axis. The COMPAS algorithm tends to overpredict recidivism for
  African-Americans as the predicted recidivism rate grows, as compared to
  Caucasians and Hispanics. On the top right, we show a rolling average of the
  race distribution over time, with window size 100 (in terms of individuals 
  screened).  On the bottom right, we show a rolling average of the prediction
  bias in time, again with window size 100. A positive bias indicates an
  overprediction of recidivism. Thus, in general, the algorithm is
  underpredicting recidivism. The drastic drop in bias towards the end is an
  artifact of the dataset---the distribution changes drastically so that almost  
  all individuals screened are recidivist. } 
\label{fig:compas-statistics}
\end{figure}

\paragraph{COMPAS dataset.}

The Correctional Offender Management Profiling for Alternative Sanctions
(COMPAS) Recidivism and Racial Bias dataset is a widely-used benchmark for 
evaluating algorithmic fairness and bias in predictive models. It contains
information about individuals arrested in Broward County, Florida, and it
includes features such as criminal history, demographics, and a recidivism risk
score, which is calculated by the COMPAS algorithm for the purpose of making
pretrial bail release decisions. One of the key challenges in this dataset is
the presence of significant racial disparities in the COMPAS risk scores,
which has led to concerns about the fairness and bias of the COMPAS predictions.  

Figure \ref{fig:compas-statistics} plots the racial distribution of the COMPAS
data, along with the predictive bias. The predicted recidivism rate is obtained
by taking the COMPAS recidivism score, which lies on an integer scale from 1 to
10, and dividing it by 10. We preprocess the dataset to remove rare groups, and
only consider the Caucasian, African-American, and Hispanic subgroups. Together,
these groups make up over 90\% of the dataset. 

\paragraph{HelpSteer2 dataset.} 

The HelpSteer2 dataset is an open-source dataset specifically designed for
training and fine-tuning reward models in alignment tasks, particularly for
language models. This dataset contains around 10,000 response pairs for training
and 1,000 for validation, each annotated for helpfulness, correctness,
coherence, complexity, and verbosity on a Likert scale. Figure
\ref{fig:helpsteer-examples} provides two examples from this dataset. We
post-train a Gemma2-2B \cite{team2024gemma} model to predict helpfulness on the
training set, and use the validation set for the purposes of our
experiments. The goal is to debias the Gemma2-2B reward model using this data,
in an online fashion.

\begin{figure}[htb]
\centering
\includegraphics[width=\linewidth]{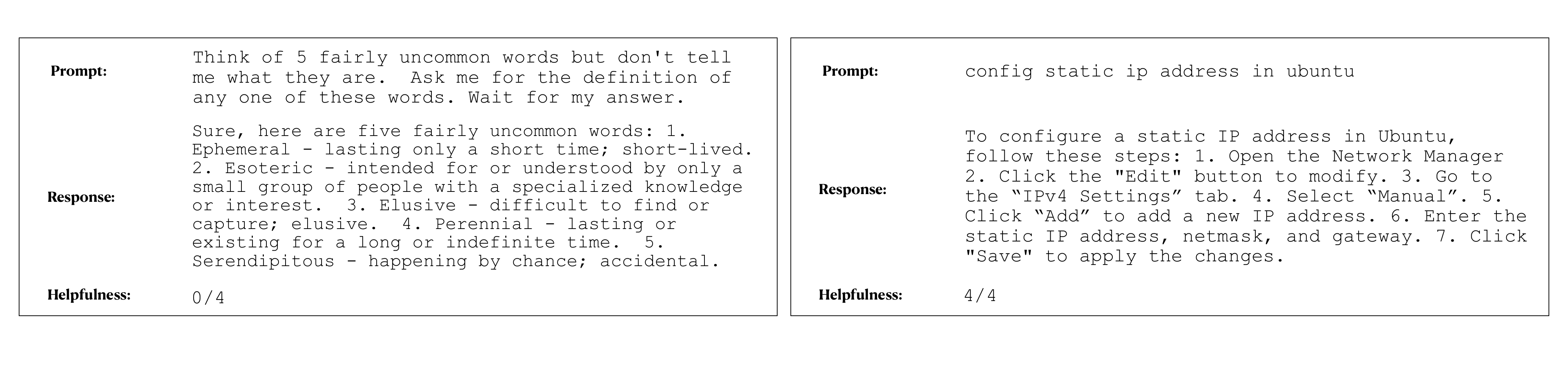}
\vspace{-25pt} % Spacing hack because somehow the caption is too far away
\caption{\it Two examples from the HelpSteer2 dataset. We display a prompt, a   
  response, and the helpfulness as judged by an expert human.} 
\label{fig:helpsteer-examples}

\bigskip
\includegraphics[width=0.6\linewidth]{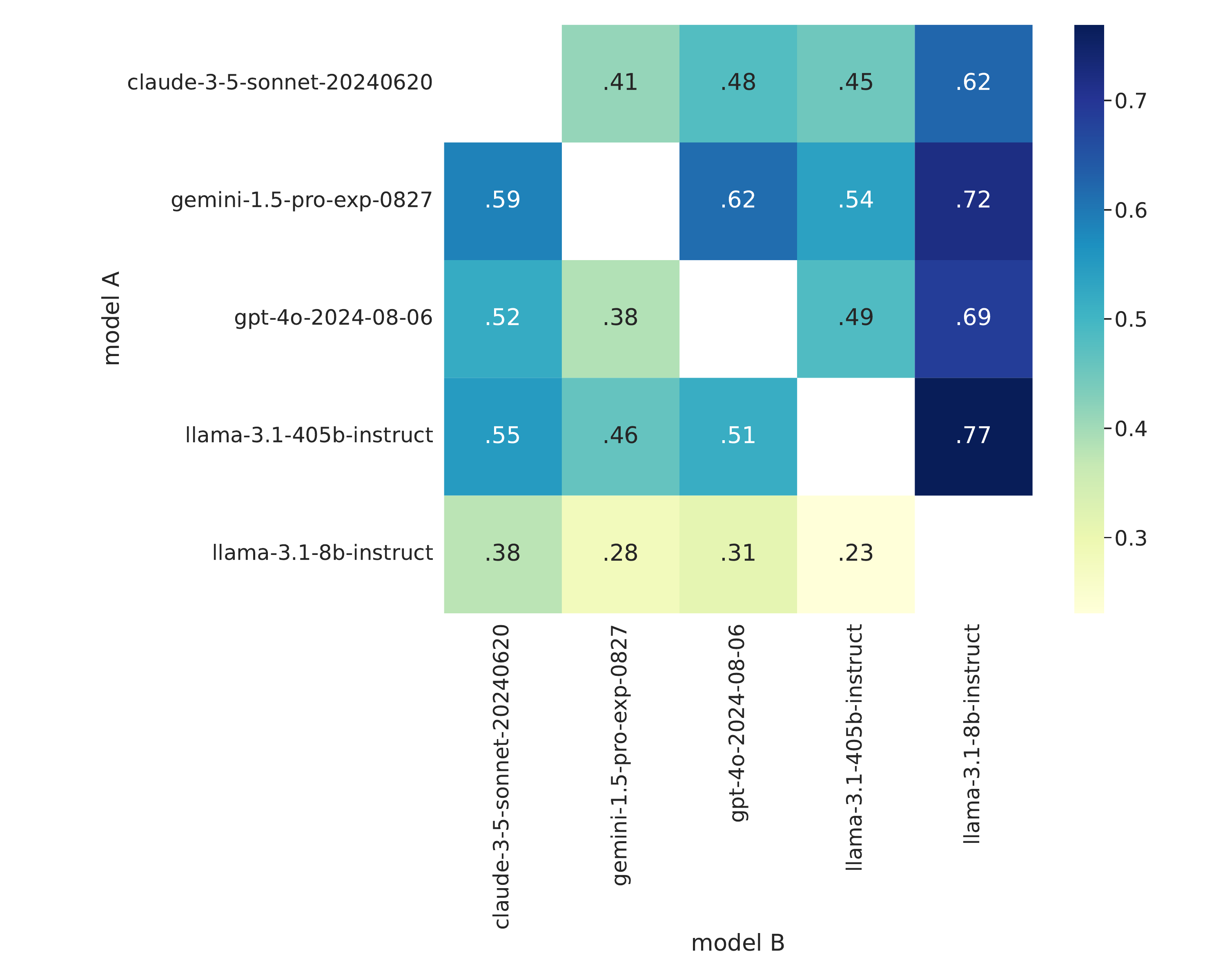}
\caption{\it Chatbot Arena win matrix. We display the win rates between five 
  major models.} 
\label{fig:arena-win-matrix}
\vspace{-5pt} % Analogous spacing hack
\end{figure}

\paragraph{Chatbot Arena dataset.} 

Chatbot Arena \cite{chiang2024chatbot} is an open-source dataset of
conversations between humans and anonymized pairs of large language models
(LLMs). At the end of the conversation, the human votes for which response they
prefer, encoded as a binary preference. This binary preference data is used
to perform a Bradley-Terry regression \cite{bradley1952rank}, whose coefficients
estimate each AI model's relative ``strength'' or capability. This allows for a
leaderboard to be created which reflects how models perform against each other
in direct comparisons, and predicts future preferences.

Bradley-Terry regression can be reframed as a particular form of logistic
regression with binary features. Therefore, our online gradient descent theory
immediately applies, and we can get a guarantee on an online variant of this
regression. The online form of Bradley-Terry regression produces what are also
known as Elo scores \cite{elo1967proposed}. Unlike the previous examples, here
we are not debiasing a model post hoc. Instead, we are simply giving a guarantee
on the Elo score, as it is normally computed. 

\paragraph{MIMIC dataset.} 

The Medical Information Mart for Intensive Care (MIMIC-III) database is a 
publicly available, de-identified clinical dataset encompassing health
information from over 40,000 patients admitted to the intensive care units of
Beth Israel Deaconess Medical Center between 2001 and 2012. It includes data
types such as demographics, vital signs, lab results, medications, caregiver
notes, and mortality data. Our goal is length-of-stay estimation: predicting a
patient's length of stay in the hospital from their covariates at the time of
initial hospital admission. We use a gradient boosting tree as our base model. 

\subsection{Simple debiasing}
\label{sec:simple_debiasing}

We consider the task of debiasing a sequence of predictions, so their average is
approximately equal to the average of responses. The structure differs slightly 
when handling regression and classification problems.  

\paragraph{Regression.}

In the regression setting, given real-valued predictions $f_t$, $t =
1,2,3,\dots$, the goal is to produce a sequence of real-valued parameters
$\theta_t$, $t = 1,2,3,\dots$, with each $\theta_t$ depending only on past data, 
such that (using $a_T \asymp b_T$ to mean $a_T - b_T \to 0$ as $T \to \infty$,
as before):   
\begin{equation}
\label{eq:simple-debias-squared}
\frac{1}{T} \sum_{t=1}^T (f_t+\theta_t) \asymp \frac{1}{T} \sum_{t=1}^Ty_t.
\end{equation}
In other words, the adjusted predictions $f_t + \theta_t$ should be unbiased
for $y_t$, in the long run along the sequence. As discussed in
Section \ref{sec:squared_losses}, online gradient descent with the squared loss,  
$\ell_t(\theta) = \frac{1}{2}(y_t - f_t - \theta)^2$, achieves this goal under 
mild conditions. See Algorithm \ref{alg:simple-debias-squared} below for
pseudocode. 

\begin{algorithm}[H]
\caption{Simple debiasing (regression)}
\begin{algorithmic}[1]
\REQUIRE Predictions $f_1,\dots,f_T$; responses $y_1,\dots,y_T$; learning rate
$\eta > 0$ 
\STATE Initialize $\theta_1 = 0$
\FOR{$t = 1,\dots,T$}
\STATE Compute the gradient: $g_t = f_t + \theta_t - y_t$
\STATE Update the parameter: $\theta_{t+1} = \theta_t - \eta g_t$
\ENDFOR
\ENSURE Adjusted predictions $f_1+\theta_1, \dots, f_T+\theta_T$ 
\end{algorithmic}
\label{alg:simple-debias-squared}
\end{algorithm}

Figure \ref{fig:simple-debiasing-helpsteer} shows the results of this method
on the HelpSteer2 dataset. As expected, gradient equilibrium holds, and thus the 
predictions are unbiased in the long run. Note also that the loss is not
strongly impacted by the choice of learning rate.

\paragraph{Classification.} 

In the classification setting, given $[0,1]$-valued probabilistic predictions
$p_t$, $t = 1,2,3,\dots$ of binary labels, the goal is to produce a sequence of 
real-valued parameters $\theta_t$, $t = 1,2,3,\dots$, where each $\theta_t$
depends only on past data, such that:
\begin{equation}
\label{eq:simple-debias-logistic}
\frac{1}{T} \sum_{t=1}^T \bigg( p_t + \frac{2e^{\theta_t}}{1+e^{\theta_t}} - 1 
\bigg) \asymp \frac{1}{T} \sum_{t=1}^T y_t.
\end{equation}
That is, the adjusted predictions $p_t + 2e^{\theta_t} / (1+e^{\theta_t}) - 1$
should be unbiased for $y_t$, in the long run. As covered in Section
\ref{sec:logistic_losses}, online gradient descent on the generalized logistic
loss, $\ell_t(\theta) = -(y_t - p_t) \theta + 2\log(1 + e^\theta) - \theta$, 
achieves this goal under mild conditions. See Algorithm
\ref{alg:simple-debias-logistic} below for pseudocode.  

\begin{algorithm}[H]
\caption{Simple debiasing (classification)}
\begin{algorithmic}[1]
\REQUIRE Probabilistic predictions $p_t,\dots,p_T$; binary labels
$y_1,\dots,y_T$; learning rate $\eta > 0$
\STATE Initialize $\theta_1 = 0$
\FOR{$t = 1,\dots,T$}
\STATE Compute the gradient: $g_t = p_y - y_t + 2\frac{e^{\theta_t}}
{1+e^{\theta_t}} - 1$    
\STATE Update the parameter: $\theta_{t+1} = \theta_t - \eta g_t$
\ENDFOR
\ENSURE Adjusted predictions $p_1 + \frac{2e^{\theta_1}}{1+e^{\theta_1}} - 1,
\dots, p_T + \frac{2e^{\theta_T}}{1+e^{\theta_T}} - 1$ 
\end{algorithmic}
\label{alg:simple-debias-logistic}
\end{algorithm}

Figure \ref{fig:simple-debiasing-compas} displays the results of this method on
the COMPAS dataset. As expected, gradient equilibrium holds, thus the
predictions are long-run unbiased. Again the loss is not significantly impacted
by the choice of learning rate. We repeat this experiment for several choices of
a decaying learning rate in Figure \ref{fig:simple-debiasing-decaying-compas}. 
The larger decay rates lead to slower convergence to gradient equilibrium
(compared to one another) and smooth out the predictions.

\begin{figure}[p]
\centering
\includegraphics[width=\linewidth]{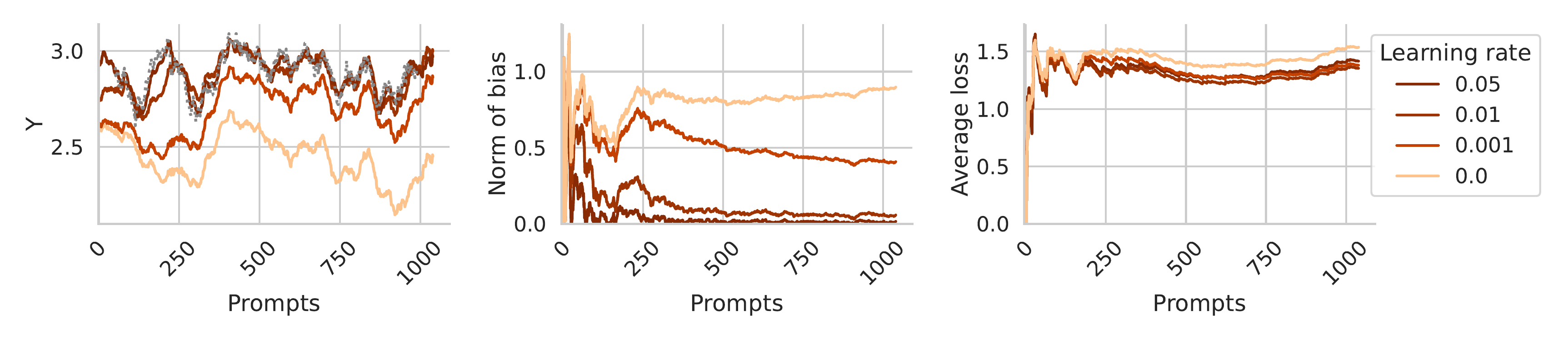}
\caption{\it Simple debiasing results on the HelpSteer2 dataset. On the left, we
  display a rolling average of the response as a gray dotted line along with the
  predictions in different colors, corresponding to the learning rate. In the
  middle, we show the norm of the bias of the adjusted predictor. It diminishes
  more quickly for larger learning rates. On the right, we plot the average loss
  of the adjusted predictor for different learning rates.
  \jupyter{https://github.com/aangelopoulos/gradient-equilibrium/blob/main/helpsteer/debias.ipynb}}  
\label{fig:simple-debiasing-helpsteer}

\bigskip\bigskip
\includegraphics[width=\linewidth]{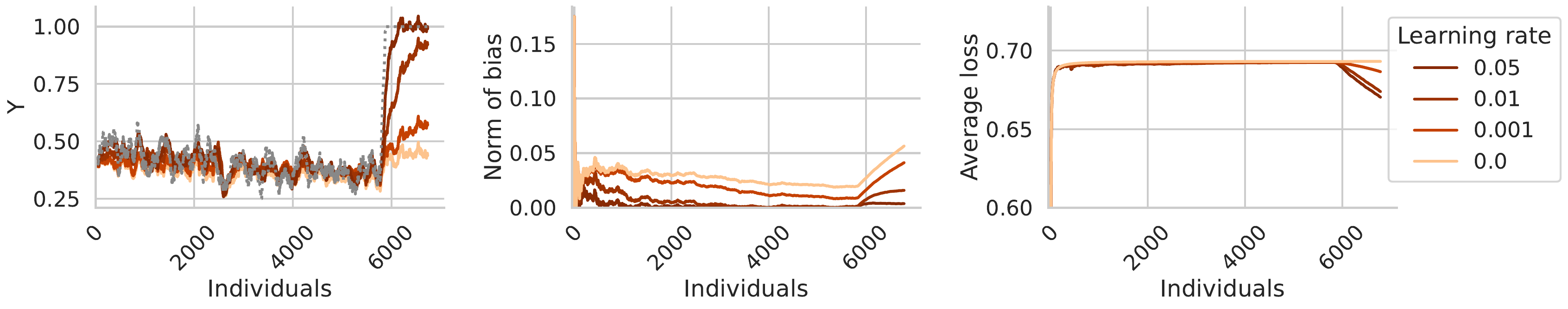}
\caption{\it Simple debiasing results on the COMPAS dataset. Same layout as
  Figure \ref{fig:simple-debiasing-helpsteer}. 
  \jupyter{https://github.com/aangelopoulos/gradient-equilibrium/blob/main/compas/debias.ipynb}}  
\label{fig:simple-debiasing-compas}  

\bigskip\bigskip
\includegraphics[width=\linewidth]{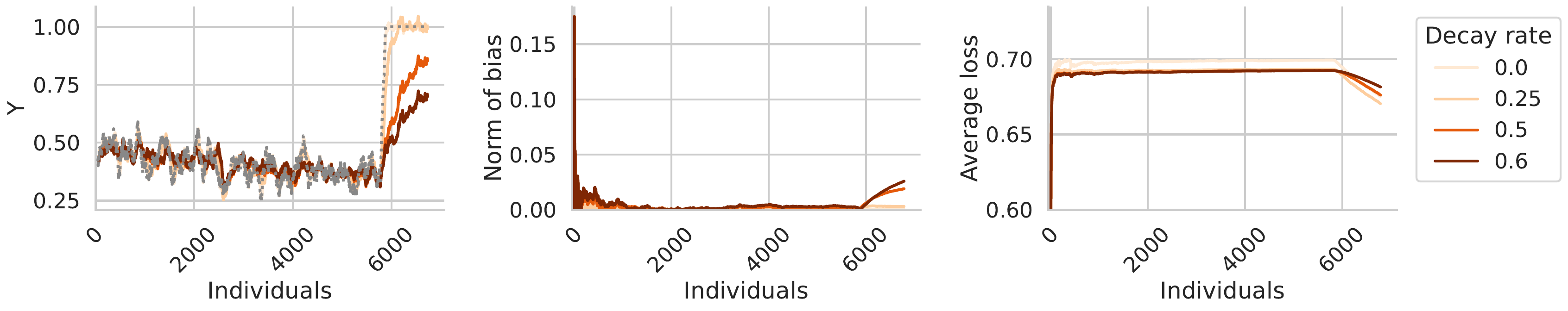}
\caption{\it Simple debiasing results on the COMPAS dataset, now with decaying
  learning rates. 
  \jupyter{https://github.com/aangelopoulos/gradient-equilibrium/blob/main/compas/debias_decaying.ipynb}}
\label{fig:simple-debiasing-decaying-compas} 

\bigskip\bigskip
\centering
\includegraphics[width=\linewidth]{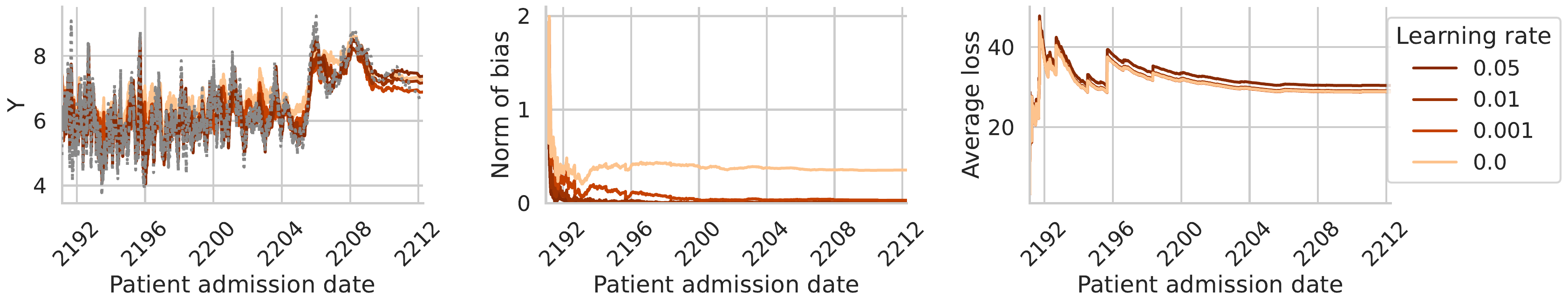}
\caption{\it Multigroup debiasing results on the MIMIC dataset. Same layout
  as Figure \ref{fig:simple-debiasing-helpsteer}. 
  \jupyter{https://github.com/aangelopoulos/gradient-equilibrium/blob/main/mimic_stay/multigroup.ipynb}}
\label{fig:multigroup-debiasing-mimic-series}
\end{figure}

\subsection{Multigroup debiasing}
\label{sec:multigroup_debiasing}

We now turn to the debiasing of a predictor simultaneously over multiple,
possibly overlapping groups. Let $z_t \in \{0,1\}^d$ be a vector of group
indicators, that is, every element of $z_t$ corresponds to some attribute of the 
data point observed at $t$, such as ethnicity, sex, marital status, and so on. 
Our goal is to achieve unbiasedness simultaneously for each of these attributes.
As before, we will have slightly different approches to regression and
classification.   

\paragraph{Regression.} 

In the regression setting, given real-valued predictions $f_t$, $t =
1,2,3,\dots$ and group indicator vectors $z_t \in \{0,1\}^d$, $t = 1,2,3,\dots$,
the goal is produce a sequence of parameter vectors $\theta_t \in \R^d$, $t = 
1,2,3,\dots$, where each $\theta_t$ depends only on past data, such that:
\begin{equation}
\label{eq:multigroup-debias-squared}
\frac{1}{|I_j|} \sum_{t \in I_j} (f_t+z_t^\T \theta_t) \asymp \frac{1}{|I_j|}
\sum_{t \in I_j} y_t, \quad \text{where} \; I_j = \{ t \leq T : z_{tj} = 1 \},
\end{equation}
for each $j = 1,\dots,d$ such that $|I_j|$ grows linearly in $T$. Observe that 
\eqref{eq:multigroup-debias-squared} is the same as 
\eqref{eq:simple-debias-squared}, but with $z_t^\T\theta_t$ in 
place of $\theta_t$, and the sum being taken only over indexes where a
particular group indicator is active. That is, the adjusted predictions $f_t +
z_t^\T \theta_t$ should be long-run unbiased for $y_t$, for each $j$ appearing
often enough. As discussed in Section \ref{sec:glms_ortho_features}, when the
groups are nonoverlapping, online gradient descent on the squared loss,
$\ell_t(\theta) = \frac{1}{2}(y_t - f_t - z_t^\T \theta)^2$, achieves this 
goal under mild conditions; as discussed in Section \ref{sec:squared_ridge},
for arbitrary groups, a ridge-regularized version can be used. See Algorithm
\ref{alg:multigroup-debias-squared} below (for the unregularized version). 

\begin{algorithm}[H]
\caption{Multigroup debiasing (regression)}
\begin{algorithmic}[1]
\REQUIRE Predictions $f_1,\dots,f_T$; group indicators $z_1,\dots,z_T$;
responses $y_1,\dots,y_T$; learning rate $\eta > 0$
\STATE Initialize $\theta_1 = 0$
\FOR{$t = 1,\dots,T$}
\STATE Compute the gradient: $g_t = z_t (f_t + z_t^\T \theta_t - y_t)$
\STATE Update the parameter: $\theta_{t+1} = \theta_t - \eta g_t$
\ENDFOR
\ENSURE Adjusted predictions $f_1+z_1^\T \theta_1, \dots, f_T+z_T^\T \theta_T$
\end{algorithmic}
\label{alg:multigroup-debias-squared}
\end{algorithm}

Figure \ref{fig:multigroup-debiasing-mimic-series} shows marginal
results of this procedure on the MIMIC dataset. Notice that the marginal bias
also goes to zero, and the average loss remains essentially the same regardless
of the learning rate. Figure \ref{fig:multigroup_debiasing_mimic} (in the 
introduction) displays the bias stratified by group. It can be seen that as
compared to the base model, applying our method drives the group-stratified
bias to zero within all groups. (It is also interesting to note that ethnicity 
and marital status form overlapping groups, yet in this case, gradient
equilibrium is achieved even without regularization.)  

\paragraph{Classification.}

In the classification setting, given $[0,1]$-valued predictions $f_t$, $t =
1,2,3,\dots$ of binary labels, and group indicator vectors $z_t \in \{0,1\}^d$, 
$t = 1,2,3,\dots$, the goal is produce a sequence of parameter vectors $\theta_t
\in \R^d$, $t = 1,2,3,\dots$, where each $\theta_t$ depends only on past data,
such that: 
\begin{equation}
\label{eq:multigroup-debias-logistic}
\frac{1}{|I_j|} \sum_{t \in I_j} \bigg( p_t + \frac{2e^{z_t^\T \theta_t}}
{1+e^{z_t^\T \theta_t}} - 1 \bigg) \asymp \frac{1}{|I_j|} \sum_{t \in I_j} y_t,
\quad \text{where} \; I_j = \{ t \leq T : z_{tj} = 1 \},  
\end{equation}
for each $j = 1,\dots,d$ such that $|I_j|$ grows linearly in $T$. Note
\eqref{eq:multigroup-debias-logistic} is the same as
\eqref{eq:simple-debias-logistic}, but with $z_t^\T\theta_t$ in  
place of $\theta_t$, and the sum being taken only over indexes where a
particular group indicator is active. As before, the adjusted predictions
$p_t + 2e^{\theta_t} / (1+e^{\theta_t}) - 1$ should be unbiased for
$y_t$ should be long-run unbiased for $y_t$, for each $j$ appearing 
often enough. As established in Section \ref{sec:glms_ortho_features}, when the   
groups are nonoverlapping, online gradient descent on the generalized logistic
loss, $\ell_t(\theta) = -(y_t - p_t) z_t^\T \theta + 2\log(1 + e^{z_t^T \theta})
- z_t^\T \theta$, achieves this goal under mild conditions; as shown in Section
\ref{sec:logistic_lasso}, for arbitrary groups, a lasso-regularized version can  
be used. See Algorithm \ref{alg:multigroup-debias-logistic} below (for the
unregularized version).  

\begin{algorithm}[H]
\caption{Multigroup debiasing (classification)}
\begin{algorithmic}[1]
\REQUIRE Probabilistic predictions $p_1,\dots,p_T$; group indicators
$z_1,\dots,z_T$; binary labels $y_1,\dots,y_T$; \\ learning rate $\eta > 0$ 
\STATE Initialize $\theta_1 = 0$
\FOR{$t = 1,\dots,T$}
\STATE Compute the gradient: $g_t = z_t \big( p_y - y_t + 2\frac{e^{\theta_t}} 
{1+e^{\theta_t}} - 1 \big)$
\STATE Update the parameter: $\theta_{t+1} = \theta_t - \eta g_t$
\ENDFOR
\ENSURE Adjusted predictions $p_1 + \frac{2e^{z_1^\T\theta_1}}{1+e^{z_1^\T 
    \theta_1}} - 1, \dots, p_T + \frac{2e^{z_T^\T \theta_T}}{1+e^{z_T^\T
    \theta_T}} - 1$
\end{algorithmic}
\label{alg:multigroup-debias-logistic}
\end{algorithm}

Figures \ref{fig:multigroup-debiasing-compas-series} and
\ref{fig:multigroup-debiasing-compas-bias} show results of this procedure on the 
COMPAS dataset. Like previously, the gradient equilibrium conditions manifest,
leading to zero long-run bias per group (improving on the COMPAS base model
considerably). Meanwhile, the loss does not change much as a function of
learning rate.   

\begin{figure}[ht]
\centering
\includegraphics[width=\linewidth]{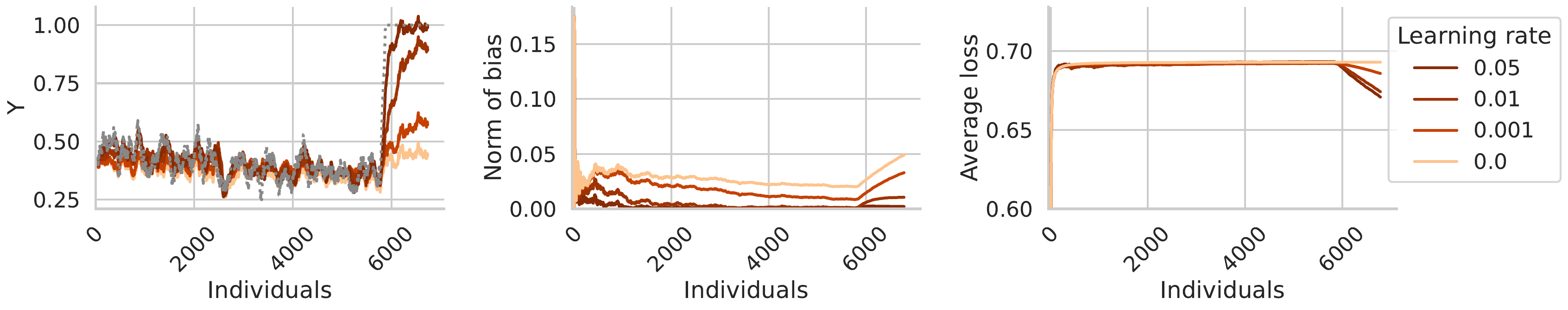}
\caption{\it Multigroup debiasing results on the COMPAS dataset. Same layout as
  Figure \ref{fig:simple-debiasing-helpsteer}. 
  \jupyter{https://github.com/aangelopoulos/gradient-equilibrium/blob/main/compas/multigroup.ipynb}} 
\label{fig:multigroup-debiasing-compas-series}

\bigskip\bigskip
\includegraphics[width=\linewidth]{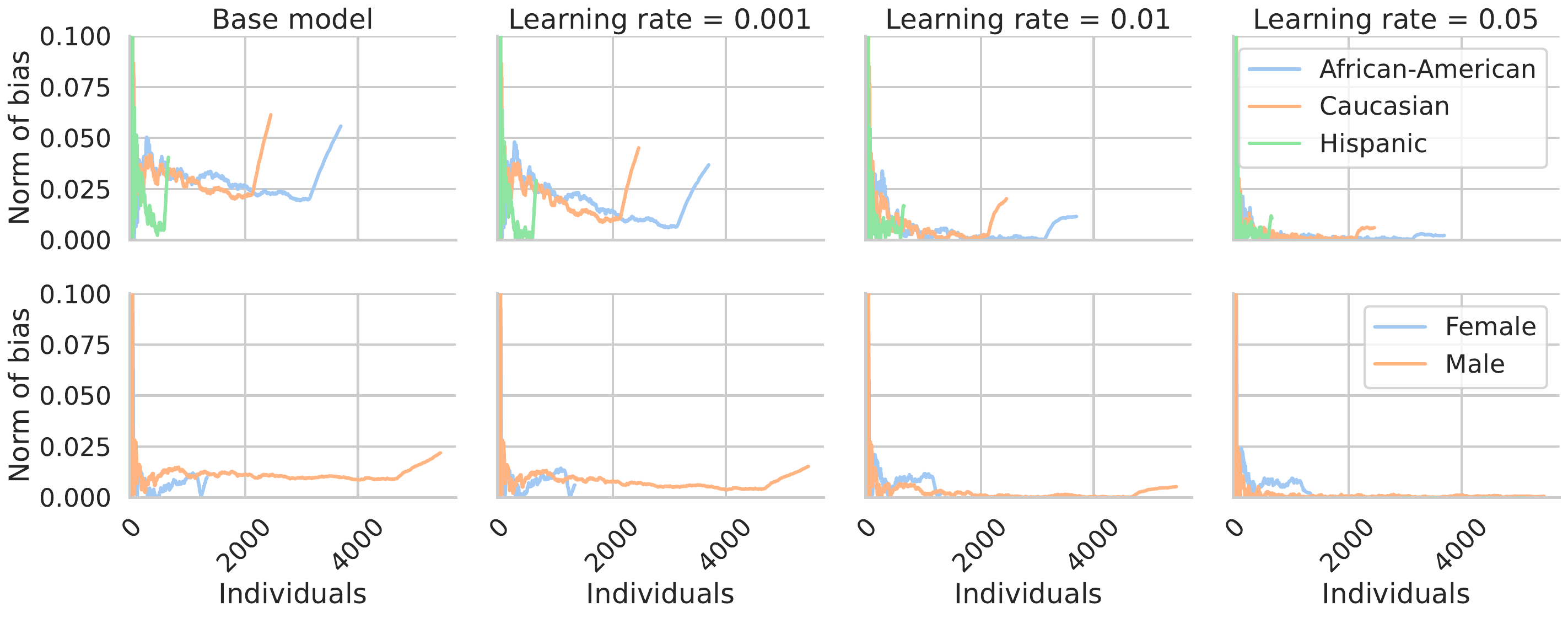}
\caption{\it Multigroup debiasing results on the COMPAS dataset. Same layout as
  Figure \ref{fig:multigroup_debiasing_mimic}.
  \jupyter{https://github.com/aangelopoulos/gradient-equilibrium/blob/main/compas/multigroup.ipynb}}  
\label{fig:multigroup-debiasing-compas-bias}
\end{figure}

\subsection{Quantile tracking and ensembling}
\label{sec:quantile_tracking}

Another application of our framework is to choose the weights of an
ensemble. Unlike the subsections above, this one does not involve debiasing.  
Instead, we focus on sequential quantile estimation, first recalling how to
provide coverage using online gradient updates, then showing how to form an
ensemble for most robust performance. We adopt the setup of online conformal
prediction \cite{gibbs2021adaptive}, and specifically focus on the quantile
tracking algorithm studied in \cite{feldman2022achieving, bhatnagar2023improved,
  angelopoulos2023conformal}. Given real-valued predictions $f_t$, $t =
1,2,3,\dots$ and responses $y_t$, $t = 1,2,3,\dots$, we seek to construct
real-valued parameters $\theta_t$, $t =1,2,3,\dots$, where each $\theta_t$
depends only on past data, such that:  
\begin{equation}
\label{eq:quantile-coverage}
\frac{1}{T} \sum_{t=1}^T 1 \{ y_t \leq f_t + \theta_t \} \asymp \tau.
\end{equation}
In other words, $f_t + \theta_t$ lies above $y_t$ with frequency $\tau$. We
refer to this property as ``coverage,'' insofar as it represents the ability to
contain $y_t$ with a one-sided interval. As discussed in Section
\ref{sec:quantile_losses}, online gradient descent with the quantile loss,   
$\ell_t(\theta) = \rho_\tau(y_t - f_t - \theta)$, achieves this goal under 
mild conditions. To be clear, this result is not original to the current paper,
and is well-known in the online conformal literature. Also, these online
gradient descent updates are sometimes referred to as quantile tracking.  

As with nearly all online learning methods, practical deployments of quantile
tracking require a choice of learning rate $\eta$. In some problems, the best
choice of $\eta$ may actually be nonstationary (time-varying). Initial efforts
have been made to set the learning rate via mixture-of-experts schemes, as in
\cite{zaffran2022adaptive, gibbs2024conformal}. Here, we propose a different
strategy: learning a mixture of quantile tracking iterates via online mirror
descent. To fix notation, consider $K$ learning rates, $\nu_1,\dots,\nu_K$,
and let $\theta^k_t$ denote the quantile tracking iterate under step size
$\nu_k$, at iteration $t$. We produce a mixed parameter estimate   
\[
\theta_t = w_t^\T (\theta^1_t, \dots \theta^K_t),
\]
where $w_t \in \R^K$ is a weight vector living in the standard $K$-dimensional
probability simplex. Further, we learn $w_t$ $t = 1,2,3,\dots$ by online mirror
descent, applied to the loss sequence $\ell_t(w) = \rho_t(y_t - f_t - w^\T
(\theta^1_t, \dots \theta^K_t))$, $t = 1,2,3,\dots$. Algorithm
\ref{alg:quantile-ensembling} gives pseudocode. Mirror descent is formally
studied in Appendix \ref{app:constraints}, where it is shown that it satisfies a
form of gradient equilibrium. How this translates into a rigorous coverage
statement, as in \eqref{eq:quantile-coverage} but at the ensemble level, is left 
to future work.   

\begin{algorithm}[H]
\caption{Quantile ensembling}
\begin{algorithmic}[1]
\REQUIRE Predictions $f_1,\dots,f_T$; responses $y_1,\dots,y_T$; quantile level
$\tau$; base learning rates $\nu_1,\dots,\nu_K$; \\ ensemble learning rate $\nu$  
\STATE Initialize $w_1$ to be uniform on the simplex 
\STATE Initialize $\theta^k_1 = 0$, for $k = 1,\dots,K$ 
\STATE Initialize $\theta_1 = 0$
\FOR{$t = 1,\dots,T$}
\STATE Compute signed error: $\sigma_t = 1\{y_t \leq f_t + \theta^k_t\} - \tau$ 
\STATE Compute quantile updates: $\theta^k_{t+1} = \theta_t - \nu_k \sigma_t$,
$k = 1,\dots,K$    
\STATE Compute ensemble update: 
\begin{align*}
z_{t+1,k} &= w_{tk} \exp(-\eta \theta^k_t \sigma_t), \quad k = 1,\dots,K \\ 
w_{t+1} &= z_{t+1} / \|z_{t+1}\|_1
\end{align*}
\STATE Compute ensemble prediction: $\theta_{t+1} = w_{t+1}^\T (\theta^1_{t+1}, 
\dots, \theta^K_{t+1})$
\ENDFOR
\ENSURE Adjusted quantile predictions $f_1+\theta_1, \dots, f_T+\theta_T$
\end{algorithmic}
\label{alg:quantile-ensembling}
\end{algorithm}

\begin{figure}[ht]
\centering
\includegraphics[width=\linewidth]{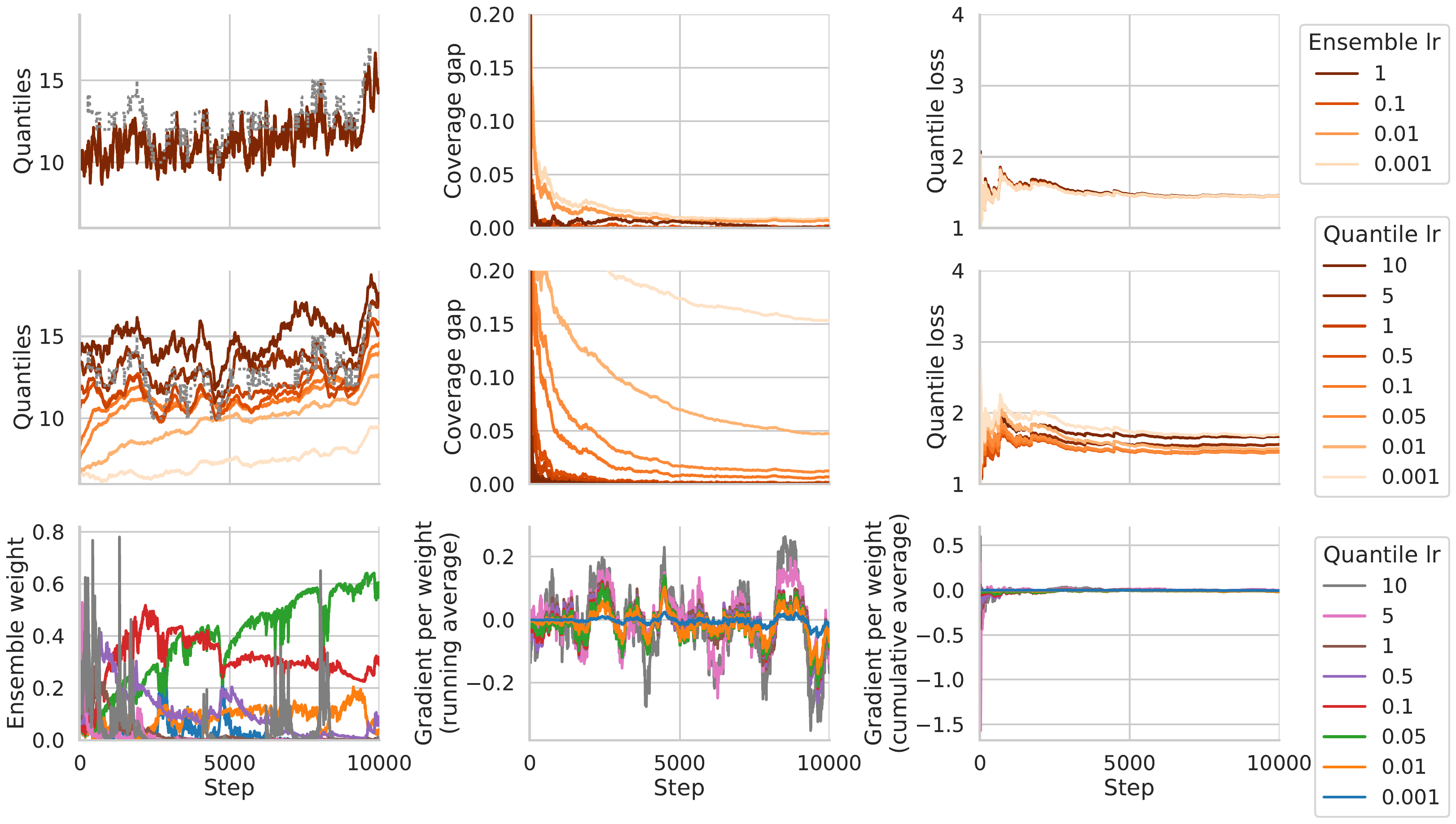}
\caption{\it Quantile ensembling results on the MIMIC dataset. We show results
  of our procedure learning to balance (via weights) different choices of
  learning rates, in quantile tracking. The top row follows the format of
  Figure \ref{fig:simple-debiasing-helpsteer}. The middle row also follows this 
  format, but this time with respect to different quantile tracker learning
  rates as opposed to different ensemble learning rates. The last row shows the 
  behavior and properties of the weights for the quantile learning rates, along 
  the sequence.
  \jupyter{https://github.com/aangelopoulos/gradient-equilibrium/blob/main/mimic_stay/quantile_ensembling.ipynb}}
\label{fig:quantile-ensembling-compas}
\end{figure}

In Figure \ref{fig:quantile-ensembling-compas}, we present results of this
procedure for predicting quantiles of the length-of-stay variable on the MIMIC
dataset. In the middle row, we see that the coverage gap diminishes quickly (as 
expected by the gradient equilibrium guarantees) for quantile tracking as we 
increase the learning rate, but this comes at the expense of increasing the
quantile loss. The top row shows that the ensemble achieves coverage faster than
many of the individual quantile trackers, while matching essentially the best
individual quantile loss. This behavior is almost invariant to the ensemble
learning rate. Finally, the third row reveals how the ensemble upweights or
downweights individual quantile trackers over the sequence.  

\subsection{Pairwise preference prediction}
\label{sec:pairwise_preference}

\begin{figure}[p]
\centering
\includegraphics[width=\linewidth]{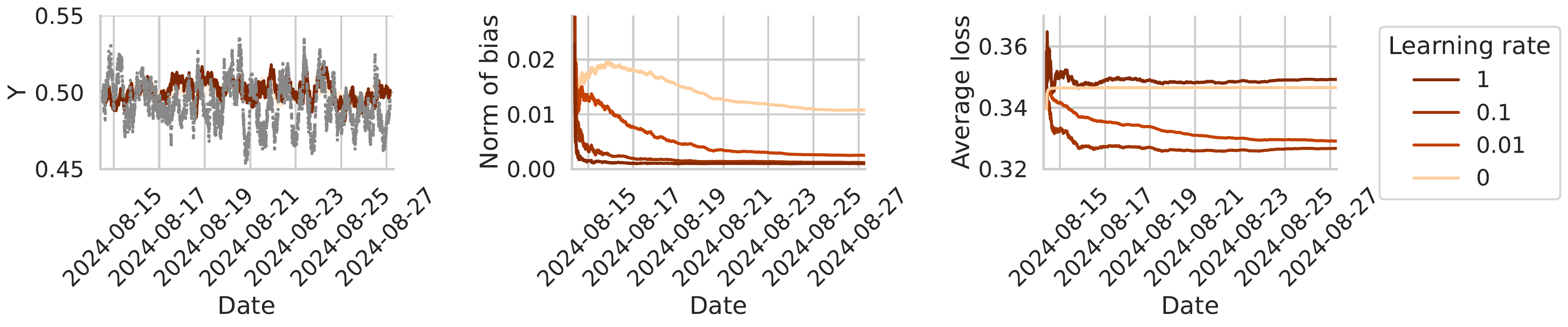}
\includegraphics[width=\linewidth]{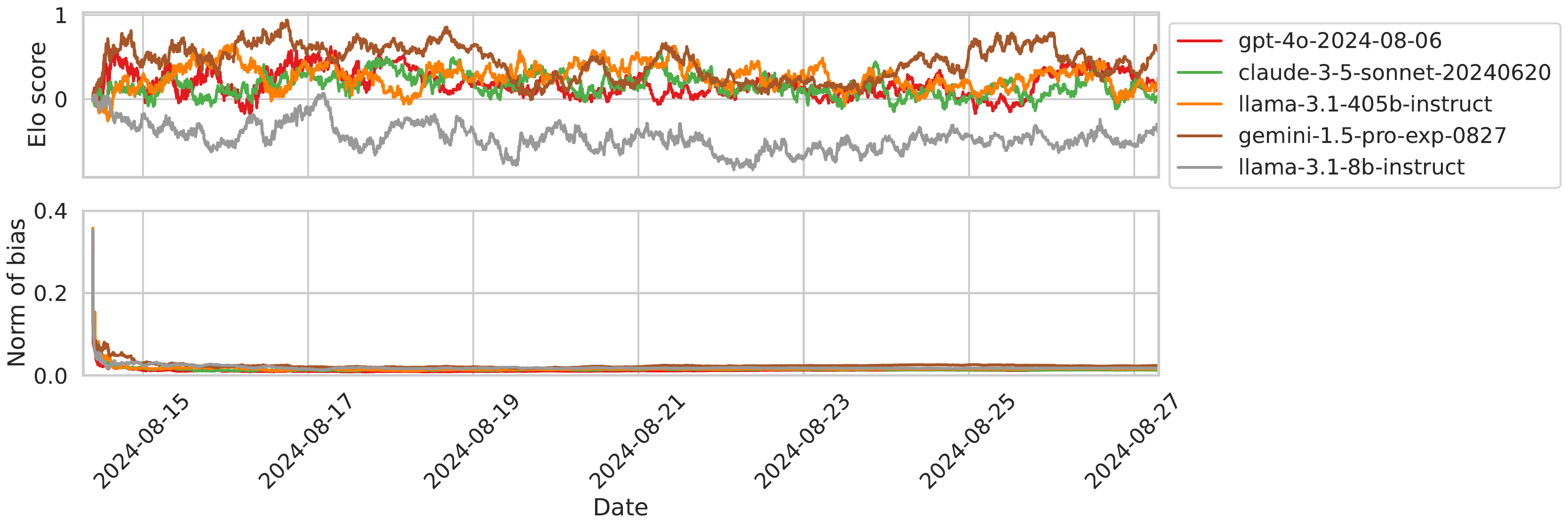}
\caption{\it Results on Chatbot Arena with regularization. Same format as 
  Figure \ref{fig:elo_scores_chatbot}. 
  \jupyter{https://github.com/aangelopoulos/gradient-equilibrium/blob/main/arena/online_elo_reg.ipynb}}
\label{fig:cba-reg}    

\bigskip\bigskip
\includegraphics[width=\linewidth]{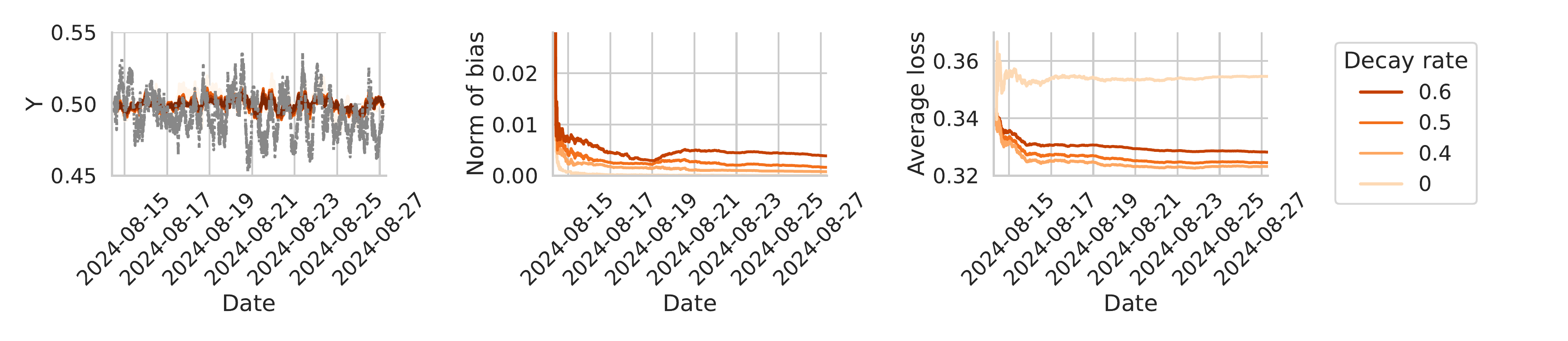}
\includegraphics[width=\linewidth]{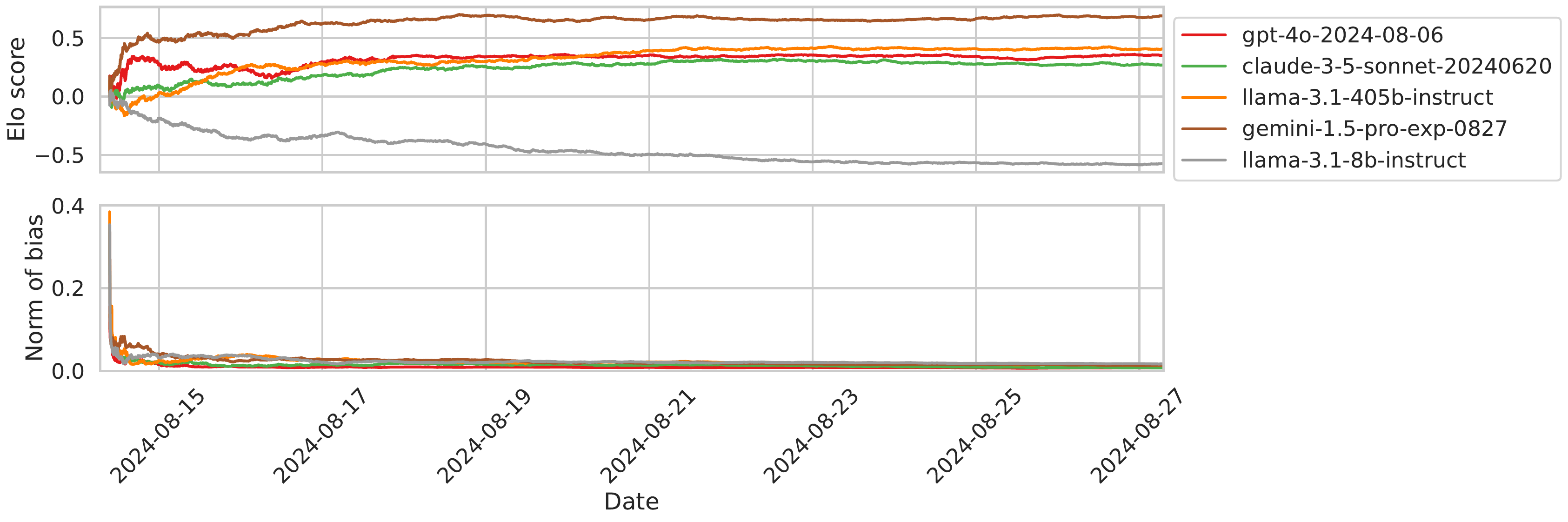}
\caption{\it Results on Chatbot Arena with decaying learning rates. Same format as 
  Figure \ref{fig:elo_scores_chatbot}. For the bottom plots, we set a learning
  rate decay of $\alpha=0.5$. 
  \jupyter{https://github.com/aangelopoulos/gradient-equilibrium/blob/main/arena/online_elo_decaying.ipynb}}  
\label{fig:cba-decay}
\end{figure}

As a last example, we use our gradient equilibrium framework to examine the
well-known Elo rating system \cite{elo1967proposed}, often used in chess and
other two-player competitions, as previewed in Figure
\ref{fig:elo_scores_chatbot} in the introduction. We use the Chatbot Arena
dataset for this example, where we have access to $M$ models and a sequence of
$T$ battles between pairs of models $a_t,b_t \in \{1,\dots,M\}$, and we define  
$y_t = 1$ when a human rater votes for the response of model $b_t$ over $a_t$,
and $y_t = 0$ otherwise. 

The Elo rating system assigns a strength to each model at each $t$, encapsulated
by a parameter $\theta_t \in \R^M$. The Elo updates can be recast as
online gradient descent with respect to the logistic loss function on a certain
set of features. In particular, define a feature vector $z_t \in \R^M$ with
$z_{tj} = -1$ if $j = a_t$, $z_{tj} = 1$ if $j = b_t$, and $z_{tj} = 0$
otherwise. Then the Elo updates are 
\[
\theta_{t+1} = \theta_t - \eta z_t (\sigma(z_t^\T \theta) - y_t), \quad t =
1,2,3,\dots,
\]
where $\sigma(x) = e^x/(1+e^x)$ is the sigmoid function. Gradient equilibrium
for this problem translates into the statement that 
\[
\frac{1}{|I_m|}\sum_{t \in I_m} p_t - \frac{1}{|I_m|}\sum_{t \in I_m} y_t \to 0, 
\quad \text{as $T \to \infty$, where} \; I_m = \{ t \leq T : \text{$a_t = m$ or 
      $b_t = m$} \},
\]
for each model $m$ such that $|I_m|$ grows linearly in $T$. This guarantee says
that the Elo score gives unbiased win-rate predictions for all players appearing
often enough, even in an adversarial online setting. 

Recall, our most basic theory in Proposition \ref{prop:gd_simple} says that the
gradient equilibrium holds whenever the Elo iterates remain bounded or slowly
growing. As shown in Section  \ref{sec:logistic_lasso}, this can be achieved by
adding in lasso regularization to the updates, but it is possible a more refined
analysis would show that the original updates themselves lead to bounded or
slowly growing iterates in general. 

Figure \ref{fig:elo_scores_chatbot} (in the introduction) shows results of the
Elo procedure on Chatbot Arena, confirming that the gradient equilibrium
guarantee can hold in practice, without any modification to the updates.  
However, this plot also underscores the importance of choosing a good learning
rate for these sorts of problems; the largest learning rate, which reaches
equilibrium the fastest, also has the highest loss. Meanwhile, the second
largest learning rate present a much more favorable tradeoff. In Figures 
\ref{fig:cba-reg} and \ref{fig:cba-decay}, we run two further experiments, with
regularization and decaying learning rates, respectively. Regularization (a
lasso penalty with $\lambda = 0.001$) leads to a slight smoothing of the
sequence of predictions, but it comes at a cost: the bias does not go to
zero. Instead, it is driven to a small positive value. With decaying learning
rates we now get much smoother Elo scores, and the bias is still driven to
zero, albeit more slowly than with constant learning rates.   

\section{Discussion}
\label{sec:discussion}

This paper has introduced gradient equilibrium as a property of interest for
online learning algorithms in the standard adversarial sequence model. Online
gradient descent with constant step sizes (Section \ref{sec:grad_descent}) will
lead to gradient equilibrium whenever its iterates remain bounded or grow
slowly, and we have developed theory under regularization (Section
\ref{sec:regularization}), and arbitrary step sizes (Section
\ref{sec:arbitrary_steps}) as well. The analysis can be generalized further to
handle constraints (Appendix \ref{app:constraints}), and to cover online
proximal mirror descent (a class of methods which includes mirror descent,
projected gradient descent, and proximal gradient descent).  Finally, our study
has practical implications for machine learning (Section \ref{sec:experiments}):
gradient equilibrium can help us debias black-box predictions, calibrate
quantiles, learn ensemble weights, prove guarantees on the Elo score, and more.   

Directions for further exploration with gradient equilibrium might include
probabilistic calibration---it is possible this condition can be expressed as
the gradient equilibrium condition of a sufficiently regular loss, or that the
theory herein can otherwise be used to recover calibration guarantees. Another
exciting frontier is control theory. A critical point worth making is that our
core theory does not depend on $g_t$ being a gradient (or subgradient) at 
all, and simply analyzes its action as a map from vectors to vectors (recall the
use of the restorative negative gradient field condition). This means that
gradient descent updates could be applied and analyzed well outside of a setting 
motivated by loss minimization---for example, when we seek to hit a fixed point 
(on average) of some operator, which is a common goal in control and other
fields.    

\subsection*{Acknowledgments}

We thank Sivaraman Balakrishnan, Emmanuel Cand{\`e}s, Isaac Gibbs, and Eric Zhao for
several thoughtful discussions, and Isaac and Eric for helping us identify related work.  
We thank Evan Frick for setting up the HelpSteer2 experiment, and in particular,
for writing the initial version of the script at the following URL:  
\url{https://github.com/aangelopoulos/gradient-equilibrium/blob/main/helpsteer/train_and_generate_rewards.py}. 

\subsection*{Funding}

MIJ was funded by the European Union, ERC-2022-SYG-OCEAN-101071601. Views and 
opinions expressed are however those of the author(s) only and do not 
necessarily reflect those of the European Union or the European Research Council
Executive Agency. Neither the European Union nor the granting authority can be
held responsible for them. RJT was funded by the Office of Naval Research, ONR
grant N00014-20-1-2787.  

\clearpage

{\RaggedRight\small
\bibliographystyle{alpha}
\bibliography{gradeq, general}}

\clearpage
\appendix

\section{Proofs}

\subsection{Proof of Proposition \ref{prop:squared_loss_nr_lrs}}
\label{app:squared_loss_nr_lrs}

Define $\theta_t = y_t + s_T$, $t = 1,\dots,T$. Then
\[
\frac{1}{T} \sum_{t=1}^T (y_t - \theta_t)^2 = s_T^2,
\]
and
\[
\frac{1}{T} \sum_{t=1}^T \theta_t = \bar{y}_T + s_T,
\]
which proves the claim.

\subsection{Proof of Proposition \ref{prop:restorative_zero_curv_1d}}  
\label{app:restorative_zero_curv_1d} 

For convenience, redefine each $h_t = \max\{ |\theta_1|, h_t \}$, and let $h_0 =  
|\theta_1|$. We will use induction to prove
\eqref{eq:gd_iterate_bound_zero_curv_1d}. The base case, for $T = 0$, is 
trivial. Assume that the result is true up through $T$. We break up the
argument for $T+1$ into cases. If $|\theta_T| \leq h_T$, then note that by the
triangle inequality
\begin{align*}
|\theta_{T+1}| 
&\leq |\theta_T| + \eta |g_T(\theta_T)| \\ 
&\leq h_T + \eta L,
\end{align*}
where the second line uses $L$-Lipschitzness of $\ell_t$. If instead
$|\theta_{T+1}| > h_T$, then we further divide this case into two subcases. If
$\theta_{T+1} > h_T$, then by \eqref{eq:restorative_zero_curv_1d}, we have 
$\theta_{T+1} \leq \theta_T$, and moreover, $\theta_{T+1} \geq h_T - \eta
g_T(\theta_T) \geq -\eta L$, using $L$-Lipschitzness (and $h_T \geq 0$). Thus,
putting these bounds together, we have $\theta_{T+1} \in [- \eta L,
\theta_T]$. In the case $\theta_{T+1} < -h_T$, similar arguments lead to
$\theta_{T+1} \in [\theta_T, \eta L]$. Thus, when $|\theta_{T+1}| > h_T$, we
have shown   
\[
\theta_{T+1} \in \big[ \min\{ \theta_T, -\eta L\}, \max \{\theta_T, \eta L\}
\big]. 
\]
By the inductive hypothesis, this gives $|\theta_{T+1}| \leq h_{T-1} + \eta L
\leq h_T + \eta L$, using the nondecreasing property of $h_T$. This completes
the inductive step and proves \eqref{eq:gd_iterate_bound_zero_curv_1d}. The
average gradient bound \eqref{eq:gd_avg_grad_bound_zero_curv_1d} now follows
from \eqref{eq:gd_avg_grad_bound}, and then using the simple inequality
$\max\{a, b\} \leq a + b$ to write the bound more cleanly.       

\subsection{Proof of Proposition \ref{prop:restorative_zero_curv}}  
\label{app:restorative_zero_curv} 

Taking the squared norm on both sides of the update \eqref{eq:grad_descent} and
expanding yields  
\begin{align}
\nonumber
\|\theta_{T+1}\|_2^2 
&= \|\theta_T\|_2^2 + \eta^2 \|g_T(\theta_T)\|_2^2 - 2 \eta 
  g_T(\theta_T)^\T \theta_T \\ 
\nonumber
&\leq \|\theta_T\|_2^2 + \eta^2 L^2 - 2 \eta g_T(\theta_T)^\T \theta_T \\ 
\label{eq:intermediate_bound}
&\leq \|\theta_1\|_2^2 + \eta^2 L^2 T - 2 \eta \sum_{t=1}^T g_t(\theta_t)^\T 
  \theta_t,   
\end{align}
where the second line uses $L$-Lipschitzness of the loss (or equivalently,
$\|g_t\|_2 \leq L$), and the third unravels the iteration over $t =
1,\dots,T$. Next we show each summand on last line above is bounded by    
\begin{equation}
\label{eq:summand_bound}
- 2 \eta g_t(\theta_t)^\T \theta_t \leq 2 \eta L h_t.  
\end{equation}
We split the argument into cases. If $\|\theta_t\|_2 \leq h_t$, then the upper
bound \eqref{eq:summand_bound} follows from the Cauchy-Schwarz inequality and
$L$-Lipschitzness. If $\|\theta_t\|_2 > h_t$, then the restorative condition
\eqref{eq:restorative} provides the stronger upper bound $-\ell_t(\theta_t)^\T
\theta_t \leq 0$. This completes the proof of \eqref{eq:summand_bound}, and
combining this with \eqref{eq:intermediate_bound}, then taking a square root, 
proves \eqref{eq:gd_iterate_bound_zero_curv}. The second 
result \eqref{eq:gd_avg_grad_bound_zero_curv} follows from \smash{$\sum_{t=1}^T
  h_T \leq T h_T$}, as $h_T$ is nondecreasing, applying
\eqref{eq:gd_avg_grad_bound}, and then using \smash{$\sqrt{a + b} \leq
  \sqrt{a} + \sqrt{b}$} to write the bound more cleanly.    

\subsection{Proof of Proposition \ref{prop:restorative_pos_curv}}  
\label{app:restorative_pos_curv} 

We use induction, as in the proof of Proposition
\ref{prop:restorative_zero_curv_1d}. As in that proof, redefine each $h_t =
\max\{ \|\theta_1\|_2, h_t \}$ and let $h_0 = \|\theta_1\|_2$. The base case,  
for $T = 0$, is trivial. Assume that the result is true up through $T$. We break
up the argument for $T+1$ into cases. If $\|\theta_T\|_2 \leq h_T$, then note
that by the triangle inequality     
\begin{align*}
\|\theta_{T+1}\|_2 
&\leq \|\theta_T\|_2 + \eta \|g_T(\theta_T)\|_2 \\ 
&\leq h_T + \eta L,
\end{align*}
where the second line uses $L$-Lipschitzness of $\ell_T$. If instead
$\|\theta_T\|_2 > h_T$, then   
\begin{align*}
\|\theta_{T+1}\|_2^2 
&= \|\theta_T\|_2^2 + \eta^2 \|g_T(\theta_T)\|_2^2 - 2 \eta g_T(\theta_T)^\T
  \theta_T \\ 
&\leq \|\theta_T\|_2^2 + \eta^2 L^2 - 2 \eta g_T(\theta_T)^\T \theta_T \\  
&\leq (h_{T-1} + \eta L)^2 + \eta^2 L^2 - 2 \eta \phi_T(\theta) \\
&\leq (h_T + \eta L)^2,
\end{align*}
where the second line uses $L$-Lipschitzness of $\ell_T$, the third uses the
inductive hypothesis and the restorative condition \eqref{eq:restorative} on
$\ell_T$, and the fourth uses the nondecreasing property of $h_T$ and condition
\eqref{eq:pos_curvature} on $\phi_T(\theta)$ (which implies that $\eta^2 L - 2
\eta \phi_T(\theta) \leq 0$). Taking a square root proves
\eqref{eq:gd_iterate_bound_pos_curv}. The second result 
\eqref{eq:gd_avg_grad_bound_pos_curv} is implied by
\eqref{eq:gd_avg_grad_bound}, and using $\max\{a, b\} \leq a + b$ to write the
bound more cleanly.   

\subsection{Proof of Proposition \ref{prop:restorative_quad_curv}}   
\label{app:restorative_quad_curv} 

This proof is similar to that of Proposition \ref{prop:restorative_quad_curv}, 
but it differs slightly in the inductive argument for $T+1$. We again divide 
into two cases. If $\|\theta_T\|_2 \leq h_T$, then the triangle inequality and
\emph{local} $L_T$-Lipschitzness of $\ell_T$, as in
\eqref{eq:local_lipschitz_bound}, implies that $\|\theta_{T+1}\|_2 \leq h_T +
\eta L_T$. If instead $\|\theta_T\|_2 > h_T$, then expanding the squared norm of
the gradient update, and using the inductive hypothesis and nondecreasingness of
$h_T,L_T$, gives
\[
\|\theta_{T+1}\|_2^2 \leq (h_T + \eta L_T)^2 + \eta^2 \|g_T
(\theta_T)\|_2^2 - 2 \eta g_T(\theta_T)^\T \theta_T. 
\]
The restorative condition \eqref{eq:restorative} with quadratic curvature
\eqref{eq:quad_curvature} implies
\[
\eta^2 \|g_T(\theta_T)\|_2^2 - 2 \eta g_T(\theta_T)^\T \theta_T \leq \eta^2
\|g_T(\theta_T)\|_2^2 - 2 \eta \phi_T(\theta) \leq 0,
\]
and thus $\|\theta_{T+1}\|_2 \leq h_T + \eta L_T$, completing the proof of
\eqref{eq:gd_iterate_bound_quad_curv}.

\subsection{Proof of Corollary \ref{cor:quantile_loss_zero_curv}}    
\label{app:quantile_loss_zero_curv}

Observe 
\begin{align*}
\frac{1}{T} \sum_{t=1}^T g_t(\theta_t) 
&= \frac{1}{T} \bigg( -\tau \sum_{t=1}^T1\{ y_t > \theta_t \} + (1-\tau)
  \sum_{t=1}^T 1\{ y_t \leq \theta_t \} \bigg) \\
&= \frac{1}{T} \sum_{t=1}^T 1\{ y_t \leq \theta_t \} -\tau,
\end{align*}
where the second line uses the assumption that $g_t(\theta_t) = -\tau$ whenever
$\theta_t = y_t$. The result follows directly from
\eqref{eq:gd_avg_grad_bound_zero_curv_1d}.       

\subsection{Proof of Corollary \ref{cor:squared_loss_quad_curv}}    
\label{app:squared_loss_quad_curv}

Fix any $|\theta| > h_t$. Let $\mu = \eta/2$, and $b_t = |y_t|$. As highlighted 
in the display before the proposition, we seek to prove that  
\[
\theta (\theta - y_t) \geq \mu (y_t - \theta)^2.
\]
It suffices to show that 
\[
|\theta| (|\theta| - b_t) \geq \mu (b_t + |\theta|)^2,
\]
or 
\[
\mu \leq \frac{|\theta| (|\theta| - b_t)}{(b_t + |\theta|)^2} =
\underbrace{\frac{|\theta|^2}{(b_t + |\theta|)^2}}_{f_+(\theta)} - 
\underbrace{\frac{b_t |\theta|}{(b_t + |\theta|)^2}}_{f_-(\theta)}.
\]
Now we will minimize \smash{$f_+(\theta)$} and maximize \smash{$f_-(\theta)$}
over $|\theta| \geq h_t$. Clearly, the minimum of \smash{$f_+(\theta)$} occurs
at the boundary of this range, $|\theta| = h_t$, and the minimum value is $h_t^2
/ (b_t + h_t)^2$. For \smash{$f_-(\theta)$}, rewrite this as  
\[
f_-(\theta) = \frac{b_t}{(b_t / \sqrt{|\theta|} + \sqrt{|\theta|})^2}. 
\]
Observe that its unconstrained maximum is obtained by globally minimizing the
denominator, which occurs at $|\theta| = b_t$. As \smash{$f_-(\theta)$} is
monotone increasing in $|\theta|$ to the right of $b_t$ (and also $h_t > b_t$),
its maximum over the range $|\theta| \geq h_t$ again occurs at $|\theta| = h_t$,
and the maximum value is therefore $b_t h_t / (b_t + h_t)^2$. Putting the last
two results together (on the extrema of \smash{$f_+$} and \smash{$f_-$}), we see
that it suffices to have         
\[
\mu \leq \frac{h_t (h_t - b_t)}{(b_t + h_t)^2}.
\]
Using $b_t \leq \delta h_t$, the right-hand side above is lower bounded by  
$h_t^2 (1-\delta) / (\delta h_t + h_t)^2 = (1-\delta)/(1+\delta)^2$, which
completes the proof of the first part of the proposition. For the second part,
we simply note that the gradient of the squared loss is bounded by $|\theta_t -
y_t| \leq h_t + b_t$ for $|\theta| \leq h_t$ and $|y_t| \leq b_t$. The rest is a
direct application of \eqref{eq:gd_avg_grad_bound_quad_curv}.   

\subsection{Proof of Corollary \ref{cor:logistic_loss_zero_curv}}    
\label{app:logistic_loss_zero_curv}

We must verify that
\begin{align*}
(b-a) \frac{e^\theta}{1 + e^\theta} + a 
&\leq y_t, \quad \text{for all $\theta < -h$}, \\    
(b-a) \frac{e^\theta}{1 + e^\theta} + a 
&\geq y_t, \quad \text{for all $\theta > h$}.  
\end{align*}
Note that the function $\theta \mapsto (b-a) e^\theta / (1 + e^\theta) + a$ is
an increasing bijection from $\R$ to $[a,b]$, whose inverse is $u \mapsto  
\log((u-a)/(b-u))$, so the first line above holds for $-h = \log((y_t-a) /
(b-y_t))$, while the second line holds for $h = \log((y_t-a) / (b-y_t))$. Taking
the max of these values two proves the first claim in the corollary. The 
remaining claims simply follow from noting $y_t \in [a+\epsilon_t,
b-\epsilon_t]$, where $\epsilon_t \in (0, (b-a)/2)$, implies   
\[
\max\bigg\{ \log \frac{y_t-a}{b-y_t}, \, \log \frac{b-y_t}{y_t-a} \bigg\} \leq  
\log \frac{b-a}{2\epsilon_t},
\]
and the rest is a direct application of
\eqref{eq:gd_avg_grad_bound_zero_curv_1d}.      

\subsection{Proof of Corollary \ref{cor:glm_loss_ortho_features}}
\label{app:glm_loss_ortho_features}

Recalling \eqref{eq:grad_descent_decoupled}, we can rewrite this as parallel
streams of GD processes, one for each $j = 1,\dots,d$:  
\[
\vartheta_{i_j(t+1), j} = \vartheta_{i_j(t), j} + \eta (y_{i_j(t)} -
\psi'(\vartheta_{i_j(t), j})), \quad t = 1,2,3,\dots,    
\]
where $i_j(t)$ is the index at which $j\th$ feature is visited for the $t\th$
time, in the sense that $|I_j(i_j(t))| = t$. We can then apply Corollaries
\ref{cor:squared_loss_quad_curv} or \ref{cor:logistic_loss_zero_curv} to the
above GD iterations (which are the standard GD iterations studied in these
corollaries, but we have simply reparametrized the index from $t = 1,2,3,\dots$
to $i_j(t)$, $t = 1,2,3,\dots$) to establish the desired results.   

\subsection{Proof of Proposition \ref{prop:restricted_strongly_monotone}} 
\label{app:restricted_strongly_monotone}

We will prove the result in parts (a) and (b) separately from that in part (c).  

\paragraph{Parts (a) and (b).} 

Fix any $\|\theta\|_2 > h_t$. Rewrite the restricted $\alpha_t$-strong
monotonicity condition \eqref{eq:restricted_strongly_monotone} as  
\begin{equation}
\label{eq:restricted_strongly_monotone2}
g_t(\theta)^\T \theta \geq g_t(0)^\T \theta + \alpha_t \|\theta\|_2^2.  
\end{equation}
It suffices to prove that 
\[
g_t(0)^\T \theta + \alpha_t \|\theta\|_2^2 \geq \kappa,
\]
where we will set $\kappa = 0$ for part (a) and $\kappa = \eta L^2 / 2$ for  
part (b). By the Cauchy-Schwarz inequality and the assumed bound on the gradient
at the origin, $g_t(0)^\T \theta \geq - \|g_t(0)\|_2 \|\theta\|_2 \geq - b_t
\|\theta\|_2$, so it suffices to prove 
\[
\alpha_t \|\theta\|_2^2 - b_t \|\theta\|_2 - \kappa \geq 0.
\]
We can treat this as a univariate quadratic inequality $q(x) \geq 0$ in $x =
\|\theta\|_2$. The quadratic $q$ has roots  
\[
x_- = \frac{b_t - \sqrt{b_t^2 + 4\kappa}}{2\alpha_t}, \quad \text{and} \quad  
x_+ = \frac{b_t + \sqrt{b_t^2 + 4\kappa}}{2\alpha_t} \leq \frac{b_t +
  \sqrt\kappa}{\alpha_t}, 
\]
where the upper bound uses the simple inequality \smash{$\sqrt{a+b} \leq
  \sqrt{a} + \sqrt{b}$}. As $q$ is convex, it will be nonnegative to the right
of its larger root $x_+$. Thus we can set $h_t$ to be larger or equal to the
right-hand side in the above display, which proves parts (a) and (b), after
plugging in for $\kappa$ appropriately.  

\paragraph{Part (c).}

Fix any $\|\theta\|_2 > h_t$. Expand the right-hand side of the restricted
$\beta_t$-co-coercivity condition \eqref{eq:restricted_co_coercive} and
rearrange to yield
\[
\|g_t(\theta)\|_2^2 \leq \beta_t (g_t(\theta) - g_t(0))^\T \theta -
\|g_t(0)\|_2^2 + 2 g_t(\theta)^\T g_t(0). 
\]
It suffices to prove that
\[
g_t(\theta)^\T \theta \geq \frac{\eta}{2} \Big( \beta_t (g_t(\theta) -
g_t(0))^\T \theta - \|g_t(0)\|_2^2 + 2 g_t(\theta)^\T g_t(0) \Big), 
\]
or equivalently 
\[
\bigg(1 - \frac{\eta\beta_t}{2} \bigg) g_t(\theta)^\T \theta +
\frac{\eta\beta_t}{2} g_t(0)^\T \theta + \frac{\eta}{2} \|g_t(0)\|_2^2 - \eta
g_t(\theta)^\T g_t(0) \geq 0.
\]
Recall \eqref{eq:restricted_strongly_monotone2} from the restricted
$\alpha_t$-strong monotonicity condition. Assuming $\eta \leq 2/\beta_t$ so that
$1 - \eta\beta_t/2 \geq 0$, we can multiply both sides of
\eqref{eq:restricted_strongly_monotone2} by $1 - \eta\beta_t/2$ to show that
the above display is implied by 
\[
g_t(0)^\T \theta + \alpha_t \bigg(1 - \frac{\eta\beta_t}{2} \bigg)
\|\theta\|_2^2 + \frac{\eta}{2} \|g_t(0)\|_2^2 - \eta g_t(\theta)^\T g_t(0) \geq
0, 
\]
or equivalently,
\[
g_t(0)^\T \theta + \alpha_t \bigg(1 - \frac{\eta\beta_t}{2} \bigg)
\|\theta\|_2^2 - \frac{\eta}{2} \|g_t(0)\|_2^2 + \eta (g_t(0) - g_t(\theta))^\T
g_t(0) \geq 0.
\] 
By the Cauchy-Schwarz inequality and the assumed bound on the gradient at the
origin, it suffices to prove 
\[
\alpha_t \bigg(1 - \frac{\eta\beta_t}{2} \bigg) \|\theta\|_2^2 - b_t
\|\theta\|_2 - \frac{\eta b_t^2}{2} + \eta (g_t(0) - g_t(\theta))^\T g_t(0) \geq
0. 
\]
For the last term, we use Cauchy-Schwarz, the restricted $\beta_t$-Lipschitz
condition on the gradient \eqref{eq:restricted_grad_lipschitz}, and the gradient
bound at the origin, to lower bound $(g_t(0) - g_t(\theta))^\T g_t(0) \geq -
\|g_t(0) - g_t(\theta)\|_2 \|g_t(0)\|_2 \geq - \beta_t b_t \|\theta\|_2$. Thus
it suffices to prove   
\[
\alpha_t \bigg(1 - \frac{\eta\beta_t}{2} \bigg) \|\theta\|_2^2 - b_t (1 +
\eta\beta_t) \|\theta\|_2 - \frac{\eta b_t^2}{2} \geq 0.
\]
As before, we can treat this as a univariate convex quadratic inequality $q(x)
\geq 0$ in $x = \|\theta\|_2$, and the larger of its two roots is
\[
x_+ = \frac{b_t (1 + \eta\beta_t) + \sqrt{(b_t (1 + \eta\beta_t))^2 + 2 \eta
    b_t^2}} {2\alpha_t (1 - \eta\beta_t/2)} \leq 
\frac{b_t (1 + \eta\beta_t) + \sqrt{\eta/2} \cdot b_t}{\alpha_t (1 -
  \eta\beta_t/2)}, 
\]
Hence we can set $h_t$ to be larger or equal to the right-hand side above, which
completes the proof of part (c). 

\subsection{Proof of Lemma \ref{lem:logistic_grad_inner_prod}}
\label{app:logistic_grad_inner_prod}

Consider the standard logistic case with $a = 0$ and $b = 1$. Let $f(u; y) =
(e^u / (1 + e^u) - y) u$. For any given $u$, this is linear in $y$, and hence
the minimum over $y$ must occur at an endpoint of the range $[0,1]$. Therefore  
\[
f(u; y) \geq \min \{ f(u; 0), f(u; 1) \}.
\]
Each of $f(\cdot; 0), f(\cdot; 1)$ are smooth quasi-convex functions. Indeed 
$f(u; 0) = f(-u; 1)$, and thus it suffices to minimize, say, $f(\cdot; 1)$. This
can be done by finding the root of its derivative, i.e., solving the
transcendental equation $u = 1 + e^{-u}$. Applying the bisection method gives 
\[
\inf_u \, f(u; 1) = -0.2784645... \geq -0.279.
\]
The result for general $a < b$ follows similarly.

\subsection{Proof of Proposition \ref{prop:logistic_lasso_pos_curv}}
\label{app:logistic_lasso_pos_curv}

To verify that \smash{$\tilde\ell_t$} is Lipschitz, we compute its subgradient  
\[
\tilde{g}_t(\theta) = \bigg( (b-a) \frac{e^{x_t^\T \theta}}{1 + e^{x_t^\T
    \theta}} + a - y_t \bigg) x_t + \lambda s, 
\]
where $s \in \partial \|\theta\|_1$. Thus we see that
\[
\| \tilde{g}_t(\theta) \|_2 \leq \bigg| (b-a) \frac{e^{x_t^\T \theta}}{1 +
  e^{x_t^\T \theta}} + a - y_t \bigg| \|x_t\|_2 + \lambda \|s\|_2 \leq c (b-a) + 
\lambda \sqrt{d},  
\]
which proves the claim. Next, to check that \smash{$\tilde\ell_t$} satisfies
the restorative property, we compute
\[
\tilde{g}_t(\theta)^\T \theta = \bigg( (b-a) \frac{e^{x_t^\T \theta}}{1 +
  e^{x_t^\T \theta}} +a - y_t \bigg) x_t^\T \theta + \lambda \|\theta\|_1.  
\]
By Lemma \ref{lem:logistic_grad_inner_prod}, the first term above can be
globally lower bounded:
\[
\tilde{g}_t(\theta)^\T \theta \geq -0.279 (b-a) + \lambda \|\theta\|_1.
\]
To verify the restorative property with positive curvature, it suffices to show that
for $\|\theta\|_2 \geq h_t$, 
\[
-0.279 (b-a) + \lambda \|\theta\|_1 \geq \frac{\eta (c (b-a) + \lambda
  \sqrt{d})^2}{2}.    
\]
Since $\|\theta\|_1 \geq \|\theta\|_2$, it suffices to have 
\[
-0.279 (b-a) + \lambda h_t \geq \frac{\eta (c (b-a) + \lambda \sqrt{d})^2}{2},   
\]
and rearranging gives the desired lower bound on $h_t$. Finally, setting $h_t$
to this lower bound, we can apply \eqref{eq:gd_avg_grad_bound_pos_curv} to yield   
\[
\bigg\| \frac{1}{T} \sum_{t=1}^T \bigg[ \bigg( (b-a) \frac{e^{x_t^\T
    \theta_t}}{1 + e^{x_t^\T \theta_t}} + a - y_t \bigg) x_t + \lambda s_t
\bigg] \bigg\|_2 \leq \frac{2 \|\theta_1\|_2}{\eta T} + \frac{L}{T} +
\frac{0.279 (b-a) + \eta L^2/2}{\lambda \eta T},  
\]
where $s_t \in \partial \|\theta_t\|_1$. To verify
\eqref{eq:logistic_lasso_avg_grad_bound} we use
\eqref{eq:original_avg_grad_bound} and \smash{$\|s_t\|_2 \leq \sqrt{d}$} to
yield  
\[
\bigg\| \frac{1}{T} \sum_{t=1}^T \bigg( (b-a) \frac{e^{x_t^\T \theta_t}}{1 +
  e^{x_t^\T \theta_t}} + a - y_t \bigg) x_t \bigg\|_2 \leq \frac{2
  \|\theta_1\|_2}{\eta T} + \frac{L}{T} + \frac{0.279 (b-a) + \eta L^2/2}
  {\lambda \eta T} + \lambda \sqrt{d}.  
\]
This completes the proof. 

\subsection{Proof of Lemma \ref{lem:squared_grad_inner_prod}}
\label{app:squared_grad_inner_prod}

Let $f(u; y) = a_1 (u - y) u - a_2 (u - y)^2$. Observe that 
\[
f(u; y) = (a_1 - a_2) u^2 - (a_1 - 2 a_2) y u - a_2 y^2,
\]
which is a convex quadratic in $u$, provided that $a_1 > a_2$. Its minimizer is
\[
\bar{u} = \frac{(a_1 - 2 a_2) y}{2 (a_1 - a_2)},
\]
and its minimum is 
\[
f(\bar{u}; y) = - \frac{(a_1 - 2 a_2)^2 y^2}{4 (a_1 - a_2)} - a_2 y^2, 
\]
and using $|y| \leq b$ to lower bound this further gives the claimed result.  

\subsection{Proof of Proposition \ref{prop:squared_ridge_quad_curv}} 
\label{app:squared_ridge_quad_curv}

To verify that \smash{$\tilde\ell_t$} is locally Lipschitz, we compute its
gradient   
\[
\tilde{g}_t(\theta) = (x_t^T \theta - y_t) x_t + \lambda \theta. 
\]
Thus we see that for $\|\theta\|_2 \leq h_t$, 
\[
\|\tilde{g}_t(\theta)\|_2 \leq |x_t^T \theta - y_t| \|x_t\|_2 + \lambda
\|\theta\|_2 \leq (c_t h_t + b_t) c_t + \lambda h_t,
\]
which proves the claim. Next, to check that \smash{$\tilde\ell_t$} satisfies
the restorative property, we compute
\[
\tilde{g}_t(\theta)^\T \theta = (x_t^T \theta - y_t) x_t^\T \theta + \lambda 
\|\theta\|_2^2,
\]
and
\[
\frac{\eta}{2} \|\tilde{g}_t(\theta)\|_2^2 = \frac{\eta}{2} (x_t^T \theta -
y_t)^2 \|x_t\|_2^2 + \lambda \eta (x_t^T \theta - y_t) x_t^\T \theta +  
\frac{\lambda^2 \eta}{2} \|\theta\|_2^2.
\]
For the restorative property with quadratic curvature it suffices to show that
for $\|\theta\|_2 \geq h_t$,  
\[
\underbrace{(1 - \lambda \eta) (x_t^T \theta - y_t) x_t^\T \theta - \frac{\eta
  c_t^2}{2} (x_t^T \theta - y_t)^2}_{f(x_t^\T \theta; \, y_t)} +\, \lambda
\bigg(1 - \frac{\lambda \eta}{2}\bigg) \|\theta\|_2^2 \geq 0,
\]
where $f(u; y) = (1 - \lambda \eta) (u - y) u - (\eta c_t^2 / 2) (u-y)^2$. By
Lemma \ref{lem:squared_grad_inner_prod}, under the assumption $\eta < 1 /
(\lambda + c_t^2 / 2)$, we can globally lower bound $f$ (first term), and we can 
also simply lower bound $\|\theta\|_2$ by $h_t$ (second term), thus it suffices
to have     
\[
-\frac{(1 - \lambda \eta - \eta c_t^2)^2 b_t^2}{4 (1 - \lambda \eta - \eta c_t^2
  / 2)} - \frac{\eta c_t^2 b_t^2}{2} + \lambda \bigg(1 - \frac{\lambda \eta}{2} 
\bigg) h_t^2 \geq 0.        
\]
Rearranging gives the desired lower bound on $h_t$. Setting $h_t$ to this lower
bound, we can apply \eqref{eq:gd_avg_grad_bound_quad_curv} to yield       
\[
\bigg\| \frac{1}{T} \sum_{t=1}^T \Big[ (x_t^T \theta_t - y_t) x_t + \lambda
\theta_t \Big] \bigg\|_2 \leq \frac{2 \|\theta_1\|_2}{\eta T} + \frac{b_T
  c_T}{T} + \frac{\sqrt{C_{0T}(\lambda)}(\eta (c_T^2 + \lambda) +
  1)}{\sqrt{\lambda} \eta T}. 
\]
where $C_{0T}(\lambda)$ is as in the proposition statement. Now we use
\eqref{eq:original_avg_grad_bound} to yield  
\begin{align*}
\bigg\| \frac{1}{T} \sum_{t=1}^T (x_t^T \theta_t - y_t) x_t \bigg\|_2 
&\leq \frac{2 \|\theta_1\|_2}{\eta T} + \frac{b_T c_T}{T} +
\frac{\sqrt{C_{0T}(\lambda)}(\eta (c_T^2 + \lambda) + 1)}  
{\sqrt{\lambda} \eta T} + \frac{\lambda }{T} \sum_{t=1}^T \|\theta_t\|_2 \\   
&\leq \frac{2 \|\theta_1\|_2}{\eta T} + \frac{b_T c_T}{T} +
\frac{\sqrt{C_{0T}(\lambda)}(\eta (c_T^2 + \lambda) + 1)} 
{\sqrt{\lambda} \eta T} + \lambda (\max\{ \|\theta_1\|_2, h_T \} + \eta L_T), 
\end{align*} 
where the last line uses \eqref{eq:gd_iterate_bound_quad_curv}. Applying the
upper bound 
\[
\lambda (\max\{ \|\theta_1\|_2, h_T \} + \eta L_T) \leq \lambda (\|\theta_1\|_2
+ \eta b_T c_T) + \sqrt{\lambda C_{0T}(\lambda)}(\eta (c_T^2 + \lambda) +
1),  
\]
leads to \eqref{eq:squared_ridge_avg_grad_bound}, and completes the proof.   

\section{Constraints}
\label{app:constraints}

We extend the perspective on gradient equilibrium from the main text by 
considering losses $\ell_t$, $t = 1,2,3,\dots$ subject to a constraint $\theta
\in C \subseteq \R^d$. In this section, we establish theory for online projected
gradient descent and online mirror descent analogous that in Section
\ref{sec:grad_descent} for online gradient descent. Our first step is to revisit
and interpret the gradient  equilibrium condition under constraints.   

\subsection{Gradient equilibrium revisited}

Let \smash{$\tilde\ell_t = \ell_t + I_C$}, $t = 1,2,3,\dots$, where $I_C$ is the
characteristic function of $C \subseteq \R^d$,   
\[
I_C(\theta) = 
\begin{cases} 
0 & \text{if $\theta \in C$}, \\
\infty & \text{if $\theta \notin C$}.
\end{cases}
\]
Assume that each $\ell_t$ is finite and subdifferentiable on $C$. As noted in
the regularization section (and verified in Appendix \ref{app:gen_subgrad}), our
notion of subgradients allows us to take, as a subgradient
\smash{$\tilde{g}_t(\theta)$} of \smash{$\tilde\ell_t$} at $\theta$,              
\begin{equation}
\label{eq:subgrad_decomp_constr}
\tilde{g}_t(\theta) = g_t(\theta) + g_{I_C}(\theta),
\end{equation}
where as usual $g_t(\theta)$ is a subgradient of $\ell_t$ at $\theta$, and
\smash{$g_{I_C}(\theta)$} denotes a subgradient of $I_C$ at
$\theta$. Subgradients of $I_C$ at $\theta$ are elements of the normal cone to
$C$ at $\theta$, a fact which will be leveraged shortly. For a sequence of 
subgradients \smash{$\tilde{g}_t(\theta)$}, $t = 1,2,3\dots$ which are chosen to
satisfy \eqref{eq:subgrad_decomp_constr}, we have 
\[
\frac{1}{T} \sum_{t=1}^T \tilde{g}_t(\theta_t) \asymp \frac{1}{T} \sum_{t=1}^T
g_t(\theta_t) + \frac{1}{T} \sum_{t=1}^T g_{I_C}(\theta_t), 
\]
and therefore (using $a_T \asymp b_T$ to mean $a_T - b_T \to 0$ as $T \to
\infty$, as before) 
\begin{equation}
\label{eq:grad_eq_constr1}
\frac{1}{T} \sum_{t=1}^T \tilde{g}_t(\theta_t) \asymp 0 \iff
\frac{1}{T} \sum_{t=1}^T g_t(\theta_t) \asymp - \frac{1}{T} \sum_{t=1}^T
g_{I_C}(\theta_t).
\end{equation}
That is, the gradient equilibrium condition for the constrained losses
can be recast as a modified equilibrium condition for the original losses.        

We can further develop the right-hand side in \eqref{eq:grad_eq_constr1}, to
help elucidate this condition. As mentioned above, subgradients 
\smash{$g_{I_C}(\theta)$} of the characteristic function $I_C$ at $\theta$ are
elements of the normal cone $N_C(\theta)$ to the set $C$ at $\theta$. The
consists of vectors $v \in \R^d$ such that $v^\T \delta \leq o(\|\delta\|_2)$,
for all directions $\delta \in \R^d$ such that $\theta + \delta \in C$. Appendix
\ref{app:gen_subgrad} gives a more formal characterization. The second condition
in \eqref{eq:grad_eq_constr1} can hence be rewritten as
\begin{equation}
\label{eq:grad_eq_constr2}
\frac{1}{T} \sum_{t=1}^T g_t(\theta_t) \asymp \frac{1}{T} \sum_{t=1}^T v_t, 
\quad \text{where $v_t \in -N_C(\theta_t)$, $t = 1,2,3,\dots$}.
\end{equation}
From this, we can form the following interpretation. If $g_t(\theta_t) = v_t$
for an individual $t$, then by definition of the normal cone, $g_t(\theta_t)^\T
\delta \geq o(\|\delta\|_2)$, for all directions $\delta \in \R^d$ such that
$\theta_t + \delta \in C$. Intuitively, any infinitesimal move away from
$\theta_t$ which remains feasible (remains within $C$) cannot decrease $\ell_t$,
according to its first-order Taylor expansion. In a similar vein, we can view
\eqref{eq:grad_eq_constr2} as saying that the average gradient
\smash{$\frac{1}{T} \sum_{t=1}^T g_t(\theta_t)$} will have nonnegative inner
product with any infinitesimal direction of ``average feasibility'' (by this we
mean an average of directions which maintain feasibility), as $T$ grows. This
can be seen as a sequential analog of the standard first-order condition for 
optimality in constrained problems.

\begin{remark}
When $\theta$ lies in the interior of $C$, the normal cone is trivially
$N_C(\theta) = \{0\}$. Only when $\theta$ lies on the boundary $\partial C$ of
$C$ is the normal cone nontrivial. This allows us to rewrite 
\eqref{eq:grad_eq_constr2} once more as   
\begin{equation}
\label{eq:grad_eq_constr3}
\frac{1}{T} \sum_{t=1}^T g_t(\theta_t) \asymp \frac{1}{T} \sum_{t : \theta_t 
  \in \partial C} v_t, \quad \text{where $v_t \in -N_C(\theta_t)$, $t =
  1,2,3,\dots$}.
\end{equation}
Thus if the number $t$ for which $\theta_t \in \partial C$ grows sublinearly,
$|\{t \leq T : \theta_t \in \partial C\}| = o(T)$, and $v_t$, $t = 1,2,3,\dots$
are uniformly bounded, then \eqref{eq:grad_eq_constr3} reduces to the usual
gradient equilibrium condition \eqref{eq:grad_eq}. Unfortunately, this
observation is of no consequence to us in the theory that follows, as we analyze
projected descent methods, which produce iterates that generically lie on the 
boundary at all iterations. However, it might be useful for analyzing
alternative (interior point) methods.
\end{remark}

\subsection{Examples of constrained equilibrium}

We work through examples of the equilibrium condition \eqref{eq:grad_eq_constr2}
for constrained losses. 

\subsubsection{Linear equalities}

Consider a set of linear equality constraints, $C = \{ \theta \in \R^d : A
\theta = b \}$, for an arbitrary matrix $A \in \R^{k \times d}$ and vector $b
\in \R^k$. Fix $\theta \in C$. Since $C$ is convex, the normal cone
$N_C(\theta)$ has a more explicit form, recalling \eqref{eq:normal_cone3}: it 
consists of all vectors $v \in \R^d$ such that $v^\T (z - \theta) \leq 0$, for
all $z \in C$. Note that for our given $C$,
\[
z \in C \iff \delta = z - \theta \in \nul(A),
\]
where $\nul(A)$ is the null space of $A$. Also, if $v^\T \delta \leq 0$ for 
all $\delta \in \nul(A)$, then we must have $v^\T \delta = 0$ for all $\delta
\in \nul(A)$. In other words, we have shown that the normal cone contains all 
vectors orthogonal to $\nul(A)$. Hence, if $U \in \R^{d \times p}$ is a matrix
whose columns form a basis for $\nul(A)$, then $U^\T v = 0$ for all $v \in
N_C(\theta)$. We can then multiply by $U^\T$ on both sides of the gradient 
equilibrium condition \eqref{eq:grad_eq_constr2}, which gives
\begin{equation}
\label{eq:grad_eq_linear1}
\frac{1}{T} \sum_{t=1}^T U^\T g_t(\theta_t) \asymp 0.
\end{equation}
This is in fact equivalent to the usual (unconstrained) gradient equilibrium
condition had we reparametrized our losses to satisfy the linear equality  
constraints. To see this, define $f_t(\beta) = \ell_t(U \beta + \theta_0)$,
where $A \theta_0 = b$. The usual gradient equilibrium condition for $f_t$, $t =
1,2,3,\dots$, invoking the chain rule, is   
\begin{equation}
\label{eq:grad_eq_linear2}
\frac{1}{T} \sum_{t=1}^T U^\T g_t(U \beta_t + \theta_0) \asymp 0.
\end{equation}
(To check that the chain rule still holds for our generalized definition of
subgradients, see Theorem 10.6 and Exercise 10.7 of
\cite{rockafellar2009variational}.) Observe that \eqref{eq:grad_eq_linear2} 
reduces to \eqref{eq:grad_eq_linear1} once we identify $\theta_t = U \beta_t +
\theta_0$.   

\subsubsection{Simplex}

Consider the standard $d$-dimensional probability simplex, $C = \{ \theta \in 
\R^d : \sum_{i=1} \theta_i = 1, \; \theta_i \geq 0, \; i = 1,\dots,d\}$. Let
$\theta \in \partial C$, and let $F$ be the smallest face of $C$ containing
$\theta$. This is the unique face such that $\theta$ lies in its relative
interior. Write $\aff(F) = L + a$ for the affine span of $F$, where $L \subseteq
\R^d$ is a linear subspace and $a \in \R^d$ is an offset. Again by convexity of
$C$, the normal cone has an explicit form, recalling \eqref{eq:normal_cone3}: it
consists of all vectors $v \in \R^d$ such that $v^\T (z - \theta) \leq 0$, for
all $z \in C$. Note that, for vectors $z$ lying in the face $F \subseteq C$ in
particular,  
\[
z \in F \implies \delta = z - \theta \in L \cap B_\epsilon(\theta), 
\]
for a ball $B_\epsilon(\theta)$ around $\theta$, with some small radius
$\epsilon > 0$, where we have used the fact that $\theta$ is in the relative
interior of $F$. Further, in order to have $v^\T \delta \leq 0$ for all $\delta
\in L \cap B_\epsilon(\theta)$, we must have $v^\T \delta = 0$ for all $\delta
\in L$. In other words, we have shown that the normal cone contains contains all
vectors orthogonal to $L$. 

Note that the linear part $L$ of a face $F$ of the standard probability simplex
generally takes the form   
\[
L = \bigg\{ v \in \R^d : \sum_{i \in S} v_i = 0, \; v_i = 0, \; i \notin S
\bigg\},
\]
for some subset $S \subseteq \{1,\dots,d\}$ of coordinates. Hence, vectors
$u \in \R^d$ orthogonal to $L$, written $u \in L^\perp$, have the form $u_i =
\alpha$ for all $i \in S$, where $\alpha \in \R$.  

Now assume that $F$ contains $v_t$ in its relative interior, for all $t =
1,2,3,\dots$. (This assumption could be relaxed to $F$ containing all but an
$o(T)$ number of elements, with a boundedness assumption on those left out.)
Taking an inner product with $u \in L^\perp$ on both sides of the gradient
equilibrium condition \eqref{eq:grad_eq_constr2} (where we choose $\alpha \not=
0$, and $u_i \not= 0$, $i \notin S$) gives   
\begin{align}
\label{eq:grad_eq_simplex1}
\frac{1}{T} \sum_{t=1}^T \sum_{i \in S} [g_t(\theta_t)]_i &\asymp 0, \\
\label{eq:grad_eq_simplex2}
\frac{1}{T} \sum_{t=1}^T [g_t(\theta_t)]_i &\asymp 0, \quad i \in S.
\end{align}
That is, the constrained equilibrium condition for the simplex in this case
implies the usual notion of gradient equilibrium \eqref{eq:grad_eq_simplex2} for
a subset of the coordinates, and an aggregate equilibrium condition
\eqref{eq:grad_eq_simplex1} for the rest.  

\subsection{Proximal mirror descent}
\label{app:prox_mirror_descent}

We turn to analyzing \emph{online proximal mirror descent}. As we will
see, this algorithm generalizes both online mirror descent and online projected 
gradient descent. Let $r : \R^d \to (-\infty, \infty]$ be a 
convex function, finite and subdifferentiable on a set $D \in \R^d$. Assume that
each loss $\ell_t$ is also finite and subdifferentiable on $D$ (but not
necessarily convex). Finally, let $\Phi: \R^d \to \R$ be a strictly convex and
differentiable function. Given an initial point $\theta_1 \in D$, online
proximal mirror descent produces iterates according to:         
\begin{equation}
\label{eq:prox_mirror_descent}
\begin{rcases*}
\nabla \Phi(z_{t+1}) = \nabla \Phi(\theta_t) - \eta_t g_t(\theta_t) \\  
\displaystyle \theta_{t+1} = \argmin_\theta \, \Big\{ D_\Phi(\theta, z_{t+1}) +
\eta_t r(\theta) \Big\} \; 
\end{rcases*} \quad t = 1,2,3,\dots,
\end{equation}
where $D_\Phi(\theta, z) = \Phi(\theta) - \Phi(z) - \nabla \Phi(z)^\T (\theta - 
z)$ is the Bregman divergence with respect to $\Phi$. The quantity $\nabla \Phi$
in this setting is often referred to as the \emph{mirror map}. Henceforth, in
referring to \eqref{eq:prox_mirror_descent}, we will drop the qualifier
``online'' and simply call this proximal mirror descent. Similarly, we will drop
the qualifier ``online'' when discussing all algorithms in this section.

Proximal mirror descent \eqref{eq:prox_mirror_descent} has many notable special
cases: 
\begin{itemize}
\item when $r = I_C$, the characteristic function of a set $C$, it reduces to
  mirror descent;
\item when $\nabla \Phi = \Id$ (the identity map), it reduces to proximal
  gradient descent; 
\item when $r = I_C$ and $\nabla \Phi = \Id$, it reduces to projected gradient
  descent; 
\item when $r = 0$ and $\nabla \Phi = \Id$, it reduces to gradient descent. 
\end{itemize}

In what follows, our main interest will be in the cases where $r = I_C$, i.e.,
mirror descent and projected gradient descent, as we are concerned with studying
gradient equilibrium under constraints. But since it is fluid to work with the
framework offered by proximal mirror descent, we will begin our analysis in
greater generality. Throughout, we let \smash{$\tilde\ell_t = \ell_t + r$}
denote the modified loss and write \smash{$\tilde{g}_t(\theta)$} for its
subgradient at $\theta$. Recall that we can take this to be of the form 
\smash{$\tilde{g}_t(\theta) = g_t(\theta) + g_r(\theta)$}, where $g_r$ is a
subgradient of $r$ at $\theta$. We focus on gradient equilibrium for the 
regularized loss sequence,
\begin{equation}
\label{eq:grad_eq_reg}
\frac{1}{T} \sum_{t=1}^T \tilde{g}_t(\theta_t) \asymp 0 \iff
\frac{1}{T} \sum_{t=1}^T g_t(\theta_t) \asymp - \frac{1}{T} \sum_{t=1}^T
g_r(\theta_t).
\end{equation}

Next we derive a generalization of the basic result in Proposition
\ref{prop:gd_simple} for proximal mirror descent.  

\begin{proposition}
\label{prop:pmd_simple}
Consider proximal mirror descent \eqref{eq:prox_mirror_descent}, with arbitrary
initialization $\theta_1 \in D$, and constant step sizes $\eta_t = \eta > 0$,
for all $t$. Let $g_r(\theta_{t+1}) = (\nabla \Phi(z_{t+1}) - \nabla
\Phi(\theta_{t+1})) / \eta$, which is indeed a subgradient of $r$ at
$\theta_{t+1}$, for each $t$, and let $g_r(\theta_1)$ be an arbitrary
subgradient of $r$ at $\theta_1$. Then  
\begin{equation}
\label{eq:pmd_avg_grad_identity}
\frac{1}{T} \sum_{t=1}^T \tilde{g}_t(\theta_t) = \frac{\nabla \Phi(\theta_1) +
  \eta g_r(\theta_1) - \nabla \Phi(\theta_{T+1}) - \eta g_r(\theta_{T+1})}{\eta
  T},
\end{equation}
and therefore
\begin{equation}
\label{eq:pmd_avg_grad_bound}
\bigg\| \frac{1}{T} \sum_{t=1}^T \tilde{g}_t(\theta_t) \bigg\|_2 \leq
\frac{\|\nabla \Phi(\theta_1) + \eta g_r(\theta_1)\|_2 + \|\nabla
  \Phi(\theta_{T+1}) + \eta g_r(\theta_{t+1})\|_2}{\eta T}.    
\end{equation}
\end{proposition}

\begin{proof}
Consider first the second step in \eqref{eq:prox_mirror_descent}, with $\eta_t =
\eta$. By the subgradient optimality condition,
\[
0 \in \nabla_\theta D(\theta_{t+1}, z_{t+1}) + \eta \partial r(\theta_{t+1}) 
\iff 0 \in \nabla \Phi(\theta_{t+1}) - \nabla \Phi(z_{t+1}) + \eta \partial 
r(\theta_{t+1}),  
\]
where we have used convexity of $\Phi,r$ to decompose the subgradients. This 
verifies the claim that $g_r(\theta_{t+1}) = (\nabla \Phi(z_{t+1}) - \nabla
\Phi(\theta_{t+1})) / \eta$ is a subgradient of $r$ at $\theta_{t+1}$. Notice,
by the first step in \eqref{eq:prox_mirror_descent}, with $\eta_t = \eta$,  
\begin{align*}
\nabla \Phi(\theta_{t+1}) 
&= \nabla \Phi(z_{t+1}) + \nabla \Phi(\theta_{t+1}) - \nabla \Phi(z_{t+1}) \\
&= \nabla \Phi(\theta_t) - \eta g_t(\theta_t) + \nabla \Phi(\theta_{t+1}) -
  \nabla \Phi(z_{t+1}) \\  
&= \nabla \Phi(\theta_t) - \eta g_t(\theta_t) - \eta g_r(\theta_{t+1}).
\end{align*}
Now rewrite the last line as $\nabla \Phi(\theta_t) - \nabla \Phi(\theta_{t+1})
= \eta g_t(\theta_t) + \eta  g_r(\theta_{t+1})$. Adding this up over $t =
1,\dots,T$, the left-hand side telescopes, yielding   
\[
\nabla \Phi(\theta_1) - \nabla \Phi(\theta_{T+1}) = \eta \sum_{t=1}^T 
(g_t(\theta_t) + g_r(\theta_{t+1})).
\]
Identifying \smash{$\tilde{g}_t(\theta_t) = g_t(\theta_t) + g_r(\theta_t)$},
rearranging, and dividing both sides by $\eta T$ proves
\eqref{eq:pmd_avg_grad_identity}. The bound \eqref{eq:pmd_avg_grad_bound}
follows by taking the norm of both sides, and then applying the triangle
inequality.   
\end{proof}

The simple bound in \eqref{eq:pmd_avg_grad_bound} reveals an important property
of proximal mirror descent (and thus all special cases thereof) with constant
step sizes. The next result summarizes. 

\begin{proposition}
\label{prop:pmd_sublinear}
For proximal mirror descent \eqref{eq:prox_mirror_descent} with constant step
sizes, gradient equilibrium \eqref{eq:grad_eq_reg} holds if the mirror map
$\nabla \Phi$ delivers slowly growing iterates, \smash{$\|\nabla
  \Phi(\theta_t)\|_2 = o(t)$}, and the regularizer $r$ delivers slowly growing
residuals $\|g_r(\theta_t)\|_2 = \| \nabla \Phi(z_t) - \nabla \Phi(\theta_t)\|_2
/ \eta = o(t)$.    
\end{proposition}

Below, we specialize to $r = I_C$ for a convex set $C$, and we investigate
consequences of the basic theory developed above for both bounded and unbounded 
constraint sets.        

\subsubsection{Bounded constraints, general $\nabla \Phi$}

The analysis for bounded $C$ is particularly simple, and it does not require any
of the restorative theory from Section \ref{sec:restorative}. If $\nabla \Phi$
is continuous, then it has a finite maximum on the bounded set $C$. Combined
with a mild (local) Lipschitz condition on $\ell_t$, this gives the desired
boundedness of the right-hand side in \eqref{eq:pmd_avg_grad_identity}.     

\begin{proposition}
\label{prop:finite_mirror}
Let $C \subseteq \R^d$ be a bounded convex set. Assume that $\nabla \Phi$
continuous on $C$, and denote 
\[
M = \sup_{\theta \in C} \, \|\nabla \Phi(\theta)\|_2 < \infty.
\]
Assume also that each $\ell_t$ is $L_t$-Lipschitz on $C$, where $L_t$ is
sublinear. Then mirror descent \eqref{eq:prox_mirror_descent} with $r = I_C$,
arbitrary initialization $\theta_1 \in C$, and constant step sizes $\eta_t =
\eta > 0$, for all $t$, satisfies gradient equilibrium:
\begin{equation}
\label{eq:pmd_avg_grad_bound_finite_mirror}
\bigg\| \frac{1}{T} \sum_{t=1}^T \tilde{g}_t(\theta_t) \bigg\|_2 \leq
\frac{\|\nabla \Phi(\theta_1) + \eta g_r(\theta_1)\|_2}{\eta T} + 
\frac{M}{\eta T} + \frac{L_T}{T} \to 0, \quad \text{as $T \to \infty$}.    
\end{equation}
\end{proposition}

\begin{proof}
Return to \eqref{eq:pmd_avg_grad_bound}, and rewrite the right-hand side as
\begin{align*}
\frac{1}{T} \sum_{t=1}^T \tilde{g}_t(\theta_t) 
&\leq \frac{\|\nabla \Phi(\theta_1) + \eta g_r(\theta_1)\|_2 + \|\nabla 
  \Phi(z_{T+1})\|_2}{\eta T} \\
&= \frac{\|\nabla \Phi(\theta_1) + \eta g_r(\theta_1)\|_2 + \|\nabla 
  \Phi(\theta_T) - \eta g_T(\theta_T)\|_2}{\eta T}. 
\end{align*}
The result \eqref{eq:pmd_avg_grad_bound_finite_mirror} follows by the using
triangle inequality, $\|\nabla \Phi(\theta_T)\|_2 \leq M$, and
$\|g_T(\theta_T)\|_2 \leq L_T$. 
\end{proof}

\subsubsection{Unbounded constraints,  $\nabla \Phi = \Id$}

For unbounded $C$, fortunately, we can leverage much of the restorative theory
from Section \ref{sec:restorative}. We assume, without loss of generality, that
$C$ contains the origin. This tied to the fact that the restorative condition
\eqref{eq:restorative} is based around the origin. (If $0 \notin C$, then we can
translate the parameter space to make it so.) We further make the restriction
that $\nabla \Phi = \Id$, the identity map, and thus focus on projected gradient  
descent. 

Note that $g_r(\theta_t) \in \partial r(\theta_t) = \partial I_C(\theta_t) =  
N_C(\theta_t)$, the normal cone to $C$ at $\theta_t$. As $C$ is convex, the 
normal cone has a special structure \eqref{eq:normal_cone3}, which implies
$g_r(\theta_t)^\T \theta_t \geq g_r(\theta_t)^\T z$ for all $z \in C$, and so, 
taking $z = 0$, 
\[
g_r(\theta_t)^\T \theta_t \geq 0.
\]
% RJT note: we could also investigate whether this gives results for proximal  
% gradient descent in lieu of gradient descent on regularized loss (Section 3).
% Proximal (strong) nonexpansiveness probably relevant here?
This simple fact will be critical in porting over the restorative theory to
projected gradient descent descent with an unbounded set $C$. It implies  
$\|\theta_t\|_2^2 + \eta^2 \|g_r(\theta_t)\|_2^2 \leq \|\theta_t + \eta
g_r(\theta_t)\|_2^2$, and in particular, 
\begin{equation}
\label{eq:subgrad_iterate_ineq}
\|\theta_t\|_2 \leq \|\theta_t + \eta g_r(\theta_t)\|_2. 
\end{equation}
We are ready to state our result for projected gradient descent, which
generalizes Propositions
\ref{prop:restorative_zero_curv_1d}--\ref{prop:restorative_quad_curv}.

\begin{proposition}
\label{prop:restorative_proj}
Let $C \subseteq \R^d$ be a convex set containing the origin, and consider
projected gradient descent \eqref{eq:prox_mirror_descent} with $r = I_C$ and
$\nabla \Phi = \Id$, arbitrary initialization $\theta_1 \in C$, and constant
step sizes $\eta_t = \eta > 0$, for all $t$.

\begin{enumerate}[label=(\alph*)]
\item If $d = 1$, and each $\ell_t$ is $L$-Lipschitz and $(h_t, 0)$-restorative,
  for nondecreasing $h_t$, then
  \begin{equation}
  \label{eq:pgd_iterate_bound_zero_curv_1d}
  |\theta_{T+1} + \eta g_r(\theta_{T+1})| \leq \max\{ |\theta_1 + \eta
  g_r(\theta_1)|, \, h_T \} + \eta L. 
  \end{equation}
  Thus if $h_t$ is sublinear, then \eqref{eq:pmd_avg_grad_bound} implies
  gradient equilibrium:      
  \begin{equation}
  \label{eq:pgd_avg_grad_bound_zero_curv_1d}
  \bigg| \frac{1}{T} \sum_{t=1}^T \tilde{g}_t(\theta_t) \bigg| \leq \frac{2 
    |\theta_1 + \eta g_r(\theta_1)|}{\eta T} + \frac{L}{T} + \frac{h_T}{\eta T}
  \to 0, \quad \text{as $T \to \infty$}.  
  \end{equation}

\item For general $d$, if each $\ell_t$ is $L$-Lipschitz and $(h_t,
  0)$-restorative, then  
  \begin{equation}
  \label{eq:pgd_iterate_bound_zero_curv}
  \|\theta_{T+1} + \eta g_r(\theta_{T+1})\|_2 \leq \sqrt{\|\theta_1 + \eta
    g_r(\theta_1)\|_2^2 + \eta^2 L^2 T + 2 \eta L \sum_{t=1}^T h_t}.  
  \end{equation}
  Thus if $h_t$ is sublinear and nondecreasing, then
  \eqref{eq:pmd_avg_grad_bound} implies gradient equilibrium:      
  \begin{equation}
  \label{eq:pgd_avg_grad_bound_zero_curv}
  \bigg\| \frac{1}{T} \sum_{t=1}^T \tilde{g}_t(\theta_t) \bigg\|_2 \leq \frac{2 
    \|\theta_1 + \eta g_r(\theta_1)\|_2}{\eta T} + \sqrt{\frac{L^2}{T} + \frac{2
      L h_T}{\eta T}} \to 0, \quad \text{as $T \to \infty$}.  
  \end{equation}

\item If each $\ell_t$ is $L$-Lipschitz and $(h_t,\phi_t)$-restorative with
  positive curvature \eqref{eq:pos_curvature}, for nondecreasing $h_t$, then
  \begin{equation}
  \label{eq:pgd_iterate_bound_pos_curv}
  \|\theta_{T+1} + \eta g_r(\theta_{T+1})\|_2 \leq \max\{ \|\theta_1 + \eta
  g_r(\theta_1)\|_2, \, h_T \} + \eta L.   
  \end{equation}
  Thus if $h_t$ is sublinear, then \eqref{eq:pmd_avg_grad_bound} implies
  gradient equilibrium:  
  \begin{equation}
  \label{eq:pgd_avg_grad_bound_pos_curv}
  \bigg\| \frac{1}{T} \sum_{t=1}^T \tilde{g}_t(\theta_t) \bigg\|_2 \leq
  \frac{2 \|\theta_1 + \eta g_r(\theta_1)\|_2}{\eta T} + \frac{L}{T} +
  \frac{h_T}{\eta T} \to 0, \quad \text{as $T \to \infty$}.      
  \end{equation}

\item If each $\ell_t$ is $L_t$-Lipschitz on the set $\{ \theta \in \R^d :
  \|\theta\|_2 \leq h_t \}$, and also $(h_t,\phi_t)$-restorative with quadratic
  curvature \eqref{eq:quad_curvature}, for nondecreasing $h_t,L_t$, then  
  \begin{equation}
  \label{eq:pgd_iterate_bound_quad_curv}
  \|\theta_{T+1} + \eta g_r(\theta_{T+1})\|_2 \leq \max\{ \|\theta_1 + \eta
  g_r(\theta_1)\|_2, \, h_T \} + \eta L_T.   
  \end{equation}
  Thus if $h_t,L_t$ are sublinear, then \eqref{eq:pmd_avg_grad_bound} implies
  gradient equilibrium:  
  \begin{equation}
  \label{eq:pgd_avg_grad_bound_quad_curv}
  \bigg\| \frac{1}{T} \sum_{t=1}^T \tilde{g}_t(\theta_t) \bigg\|_2 \leq
  \frac{2 \|\theta_1 + \eta g_r(\theta_1)\|_2}{\eta T} + \frac{L_T}{T} +
  \frac{h_T}{\eta T} \to 0, \quad \text{as $T \to \infty$}.      
  \end{equation}
\end{enumerate}
\end{proposition}

\begin{proof}
It will be helpful to rewrite the projected gradient descent updates using the
regularizer subgradient $g_r(\theta_{t+1}) = (z_{t+1} - \theta_{t+1}) / \eta$,
introduced in Proposition \ref{prop:pmd_simple}. Note that we can rewrite
\eqref{eq:prox_mirror_descent} (with $\nabla \Phi = \Id$) as 
\begin{equation}
\label{eq:proj_grad_descent} 
\theta_{t+1} + \eta g_r(\theta_{t+1}) = \theta_t - \eta g_t(\theta_t) \quad t =
1,2,3,\dots.  
\end{equation}
We now prove parts (a)--(d) separately.

\paragraph{Part (a).}

We follow the proof of Proposition \ref{prop:restorative_zero_curv_1d} in
Appendix \ref{app:restorative_zero_curv_1d}. For convenience, redefine $h_t =
\max\{ |\theta_1 + \eta g_r(\theta_1)|, \, h_t \}$, and let $h_0 = |\theta_1 +
\eta g_r(\theta_1)|$. We will use induction to prove
\eqref{eq:pgd_iterate_bound_zero_curv_1d}. The base case, for $T = 0$, is  
trivial. Assume that the result is true up through $T$. We break up the
argument for $T+1$ into cases. If $|\theta_T| \leq h_T$, then note that by 
\eqref{eq:proj_grad_descent} and the triangle inequality 
\begin{align*}
|\theta_{T+1} + \eta g_r(\theta_{T+1})| 
&\leq |\theta_T| + \eta |g_T(\theta_T)| \\ 
&\leq h_T + \eta L,
\end{align*}
where the second line uses $L$-Lipschitzness of $\ell_t$. If instead
$|\theta_{T+1}| > h_T$, then the same argument (applying one-sided inequalities)
as in Appendix \ref{app:restorative_zero_curv_1d} leads to  
\[
\theta_{T+1} + \eta g_r(\theta_{t+1}) \in \big[ \min\{ \theta_T, -\eta L\}, \max
\{\theta_T, \eta L\} \big]. 
\]
By \eqref{eq:subgrad_iterate_ineq} and the inductive hypothesis, 
$|\theta_{T+1} + \eta g_r(\theta_{t+1})| \leq h_{T-1} + \eta L \leq h_T + \eta
L$, using the nondecreasing property of $h_T$. This completes the inductive step
and proves \eqref{eq:pgd_iterate_bound_zero_curv_1d}. The average gradient
bound \eqref{eq:pgd_avg_grad_bound_zero_curv_1d} follows from 
\eqref{eq:pmd_avg_grad_bound}, and then using the simple inequality $\max\{a, b\}
\leq a + b$ to write the bound more cleanly.   

\paragraph{Part (b).}

We follow the proof of Proposition \ref{prop:restorative_zero_curv} in Appendix
\ref{app:restorative_zero_curv}. Taking the squared norm on both sides of the
update \eqref{eq:proj_grad_descent} and expanding yields  
\begin{align*}
\|\theta_{T+1} + \eta g_r(\theta_{T+1})\|_2^2 &= \|\theta_T\|_2^2 + \eta^2
  \|g_T(\theta_T)\|_2^2 -  2 \eta g_T(\theta_T)^\T \theta_T \\ 
&\leq \|\theta_T\|_2^2 + \eta^2 L^2 - 2 \eta g_T(\theta_T)^\T \theta_T \\ 
&\leq \|\theta_1\|_2^2 + \eta^2 L^2 T - 2 \eta \sum_{t=1}^T g_t(\theta_t)^\T 
  \theta_t.
\end{align*}
Here the second line uses $L$-Lipschitzness, while the third unravels the
iteration over $t = 1,\dots,T$, using \eqref{eq:subgrad_iterate_ineq} at each
step. The same argument as in Appendix \ref{app:restorative_zero_curv} shows
that each $- 2 \eta g_t(\theta_t)^\T \theta_t \leq 2 \eta L h_t$. Plugging this
in, and taking a square root, proves \eqref{eq:pgd_iterate_bound_zero_curv}. The
second result \eqref{eq:gd_avg_grad_bound_zero_curv} follows from bounding 
\smash{$\sum_{t=1}^T h_T \leq T h_T$}, as $h_T$ is nondecreasing, applying
\eqref{eq:gd_avg_grad_bound}, and then using \smash{$\sqrt{a + b} \leq \sqrt{a}  
  + \sqrt{b}$} to write the bound more cleanly.  

\paragraph{Part (c).}

We follow the proof of Proposition \ref{prop:restorative_pos_curv} in Appendix
\ref{app:restorative_pos_curv}. Redefine $h_t = \max\{ \|\theta_1 + \eta
g_r(\theta_1)\|_2, \, h_t \}$, and let $h_0 = \|\theta_1 + \eta
g_r(\theta_1)\|_2$. We use induction to prove
\eqref{eq:pgd_iterate_bound_pos_curv}. The base case, for $T = 0$, is 
trivial. Assume that the result is true up through $T$. We break up the
argument for $T+1$ into cases. If $\|\theta_T\|_2 \leq h_T$, then note 
that by \eqref{eq:proj_grad_descent} and the triangle inequality
\begin{align*}
\|\theta_{T+1} + \eta_t g_r(\theta_{T+1})\|_2
&\leq \|\theta_T\|_2 + \eta \|g_T(\theta_T)\|_2 \\ 
&\leq h_T + \eta L,
\end{align*}
where the second line uses $L$-Lipschitzness of $\ell_T$. If instead
$\|\theta_T\|_2 > h_T$, then      
\begin{align*}
\|\theta_{T+1} + \eta g_r(\theta_{T+1})\|_2^2 
&= \|\theta_T\|_2^2 + \eta^2 \|g_T(\theta_T)\|_2^2 - 2 \eta 
  g_T(\theta_T)^\T \theta_T \\
&\leq \|\theta_T\|_2^2 + \eta^2 L^2 - 2 \eta g_T(\theta_T)^\T \theta_T \\ 
&\leq (h_{T-1} + \eta L)^2 + \eta^2 L^2 - 2 \eta g_T(\theta_T)^\T \theta_T \\ 
&\leq (h_{T-1} + \eta L)^2 + \eta^2 L^2 - 2 \eta \phi_T(\theta) \\
&\leq (h_T + \eta L)^2.
\end{align*}
Here the second line uses $L$-Lipschitzness of $\ell_T$, the third uses
\eqref{eq:subgrad_iterate_ineq} and the inductive hypothesis, the fourth uses
the restorative condition \eqref{eq:restorative}, and the last uses the
nondecreasing property of $h_T$ and positive curvature
\eqref{eq:pos_curvature}. Taking a square root proves 
\eqref{eq:pgd_iterate_bound_pos_curv}. The second result
\eqref{eq:pgd_avg_grad_bound_pos_curv} follows from
\eqref{eq:pmd_avg_grad_bound}, and using $\max\{a, b\} \leq a + b$ to write    
the bound more cleanly. 

\paragraph{Part (d).}

This follows a similar proof to part (c) (and also to that of Proposition
\ref{prop:restorative_quad_curv} in Appendix \ref{app:restorative_quad_curv}),
but differs slightly in the inductive argument for $T+1$. As before, we
divide into two cases. If $\|\theta_T\|_2 \leq h_T$, then the triangle
inequality and local $L_T$-Lipschitzness of $\ell_T$ implies $\|\theta_{T+1}\|_2
\leq h_T + \eta L_T$. If instead $\|\theta_T\|_2 > h_T$, then expanding the
squared norm of the gradient update, using \eqref{eq:subgrad_iterate_ineq}
together with the inductive hypothesis, and nondecreasingness of $h_T,L_T$,
gives   
\[
\|\theta_{T+1}\|_2^2 \leq (h_T + \eta L_T)^2 + \eta^2 \|g_T
(\theta_T)\|_2^2 - 2 \eta g_T(\theta_T)^\T \theta_T. 
\]
The restorative condition \eqref{eq:restorative} with quadratic curvature
\eqref{eq:quad_curvature} implies
\[
\eta^2 \|g_T (\theta_T)\|_2^2 - 2 \eta g_T(\theta_T)^\T \theta_T \leq \eta^2
\|g_T (\theta_T)\|_2^2 - 2 \eta \phi_T(\theta) \leq 0,
\]
and thus $\|\theta_{T+1}\|_2 \leq h_T + \eta L_T$, completing the proof of
\eqref{eq:pgd_iterate_bound_quad_curv}.
\end{proof}

\begin{remark}
All uses of the restorative property, throughout the proof, were limited to 
examining gradient inner products of the form $g_t(\theta_t)^\T \theta_t$ for
$\theta_t \in C$ (by the nature of projected gradient descent). That is, instead
of \eqref{eq:restorative}, we only really require the following $C$-restricted 
restorative condition:   
\begin{equation}
\label{eq:restorative_constr}
g(\theta)^\T \theta \geq \phi(\theta), \quad \text{for all $\|\theta\|_2 > h$,
  $\theta \in C$, and all generalized subgradients $g(\theta)$ of $\ell$ at
  $\theta$}.    
\end{equation}
\end{remark}

\section{Generalized subgradients}
\label{app:gen_subgrad}

A generalized subgradient of $\ell$ at $\theta \in D$, denoted $g(\theta)$, can
be taken to be any vector $v \in \R^d$ satisfying 
\begin{equation}
\label{eq:gen_subgrad1}
\liminf_{\substack{z \to \theta \\ z \not= \theta}} \frac{\ell(z) -
  \ell(\theta) - v^\T (z-\theta)}{\|z-\theta\|_2} \geq 0.
\end{equation}
We note that \cite{rockafellar2009variational} call this a ``regular''
subgradient, in their Definition 8.3. Two important special cases to highlight
are as follows: 
\begin{itemize}
\item if $\ell$ is convex, then $g(\theta)$ reduces to a classical subgradient
  of $\ell$ at $\theta$, written $g(\theta) \in \partial \ell(\theta)$;  
\item if $\ell$ is differentiable at $\theta$, then $g(\theta)$ reduces
  (uniquely) to the gradient of $\ell$ at $\theta$, written $g(\theta) = \nabla 
  \ell(\theta)$. 
\end{itemize}
See Exercise 8.8(a) and Proposition 8.12 of \cite{rockafellar2009variational}
for proofs of these facts. In general, when a subgradient exists at $\theta$,
the function $\ell$ is said to be subdifferentiable at this point.

Note that we may abbreviate \eqref{eq:gen_subgrad1} as
\begin{equation}
\label{eq:gen_subgrad2}
\ell(z) \geq \ell(\theta) + v^\T (z-\theta) + o(\|z-\theta\|_2), \quad
\text{for all $z$}.
\end{equation}
This allows us to easily verify the following fact. If $\ell = f_1 + f_2$, where 
each $f_1,f_2$ are subdifferentiable at $\theta$ with subgradients $v_1,v_2$ 
respectively, then $\ell$ is subdifferentiable at $\theta$ with subgradent
$v_1+v_2$. To check this, simply note that by adding together  
\begin{align*}
f_1(z) &\geq f_1(\theta) + v_1^\T (z-\theta) + o(\|z-\theta\|_2), \\ 
f_2(z) &\geq f_2(\theta) + v_2^\T (z-\theta) + o(\|z-\theta\|_2),  
\end{align*}
we get 
\[
(f_1+f_2)(z) \geq (f_1+f_2)(\theta) + (v_1+v_2)^\T (z-\theta) +
o(\|z-\theta\|_2),  
\]
as desired.

Another useful fact concerns the characteristic function $I_C$ of a set $C$, 
\[
I_C(\theta) = 
\begin{cases} 
0 & \text{if $\theta \in C$}, \\
\infty & \text{if $\theta \notin C$},
\end{cases}
\]
whose subgradients are intimately connected to the geometry of $C$. It is not 
hard to see that subgradients of $I_C$ are vectors $v \in \R^d$ satisfying  
\begin{equation}
\label{eq:normal_cone1}
\liminf_{\substack{z \overset{C}{\to} \theta \\ z \not= \theta}} \frac{v^\T
  (z-\theta)}{\|z-\theta\|_2} \leq 0. 
\end{equation}
Here \smash{$z \overset{C}{\to} \theta$} means that $z$ converges to $\theta$
from within $C$. A vector $v \in \R^d$ that satisfies the above condition is 
called a normal vector to $C$ a $\theta$. The set of all such vectors form what 
we call the normal cone, denoted $N_C(\theta)$.  We note that
\cite{rockafellar2009variational} call such vectors ``regular'' normal vectors.  
Similar to the notation for subgradients, it is often convenient to abbreviate 
\eqref{eq:normal_cone1} as 
\begin{equation}
\label{eq:normal_cone2}
v^\T (z-\theta) \leq o(\|z-\theta\|_2), \quad \text{for all $z \in C$}. 
\end{equation}
The form \eqref{eq:normal_cone2} is used to describe normal vectors in Appendix
\ref{app:constraints}. An important special case to highlight: if $C$ is
convex, then $N_C(\theta)$ reduces to the usual definition of the normal cone
from convex geometry,  
\begin{equation}
\label{eq:normal_cone3}
N_C(\theta) = \{ v \in \R^d : v^\T (z - \theta) \leq 0, \; \text{for all $z \in  
  C$} \}. 
\end{equation}
See Theorem 6.9 of \cite{rockafellar2009variational} for a proof of this fact. 

\section{No move regret}
\label{app:no_move_regret}

Consider the following definition of a regret-type property. 

\begin{definition}
A sequence of iterates $\theta_t$, $t = 1,2,3,\dots$ is satisfies \emph{no move 
  regret} (NMR) with respect to a sequence of loss functions $\ell_t$, $t =
1,2,3,\dots$ provided that for any bounded set $D \subseteq \R^d$,  
\begin{equation}
\label{eq:no_move_regret}
\liminf_{T \to \infty} \, \inf_{\delta \in D} \, \frac{1}{T} \bigg( \sum_{t=1}^T
\ell_t(\theta_t + \delta) - \sum_{t=1}^T \ell_t(\theta_t) \bigg) \geq 0. 
\end{equation}
\end{definition}

This concept serves as somewhat of a bridge between no regret (NR) and gradient 
equilibrium (GEQ). For convex losses, it is straightforward to show that GEQ
implies NMR. 

\begin{proposition}
\label{prop:geq_implies_nmr}
Assume each $\ell_t$, $t = 1,2,3,\dots$ is convex. If the iterates $\theta_t$,
$t = 1,2,3,\dots$ satisfy gradient equilibrium \eqref{eq:grad_eq}, then they
satisfy no move regret \eqref{eq:no_move_regret}.
\end{proposition}

\begin{proof}
Fix any bounded set $D$, and $\delta \in D$. Since $\ell_t$ is convex, any
subgradient $g_t(\theta_t)$ at $\theta_t$ must satisfy
\[
\ell_t(\theta_t + \delta) - \ell_t(\theta_t) \geq g_t(\theta_t)^\T \delta.
\]
Average this over $t = 1,\dots,T$ and using linearity gives
\[
\frac{1}{T} \bigg( \sum_{t=1}^T \ell_t(\theta_t + \delta) - \sum_{t=1}^T
\ell_t(\theta_t) \bigg) \geq \bigg( \frac{1}{T} \sum_{t=1}^T g_t(\theta_t)
\bigg)^\T \delta. 
\]
Letting $b$ denote the bound on vectors in $D$, by Cauchy-Schwarz,
\[
\inf_{\delta \in D} \, \frac{1}{T} \bigg( \sum_{t=1}^T \ell_t(\theta_t + \delta)
- \sum_{t=1}^T \ell_t(\theta_t) \bigg) \geq \bigg( \frac{1}{T} \sum_{t=1}^T
g_t(\theta_t) \bigg)^\T \delta \geq - b \bigg\| \sum_{t=1}^T g_t(\theta_t)
\bigg\|_2. 
\]
Taking the limit infimum as $T \to \infty$ proves the claim.
\end{proof}

On the other hand, for smooth losses, it is likewise straightforward to show
that NMR implies GEQ.     

\begin{proposition}
\label{prop:nmr_implies_geq}
Assume each $\ell_t$, $t = 1,2,3,\dots$ is $\beta$-smooth (differentiable with
$\beta$-Lipschitz gradient). If the iterates $\theta_t$, $t = 1,2,3,\dots$
satisfy no move regret \eqref{eq:no_move_regret}, then they satisfy gradient
equilibrium \eqref{eq:grad_eq}.    
\end{proposition}

\begin{proof}
Fix any $\delta$. By a well-known fact for $\beta$-smooth functions, 
\[
\ell_t(\theta_t + \delta) - \ell_t(\theta_t) \leq \nabla \ell_t(\theta_t)^\T
\delta + \frac{\beta}{2} \|\delta\|_2^2.
\]
Averaging this over $t = 1,\dots,T$ gives
\[
\frac{1}{T} \bigg( \sum_{t=1}^T \ell_t(\theta_t + \delta) - \sum_{t=1}^T
\ell_t(\theta_t) \bigg) \leq \bigg( \frac{1}{T} \sum_{t=1}^T \nabla
\ell_t(\theta_t) \bigg)^\T \delta + \frac{\beta}{2} \|\delta\|_2^2.
\]
Letting \smash{$G_T = \frac{1}{T} \sum_{t=1}^T \nabla \ell_t(\theta_t)$}, the
right-hand side above is minimized at $\delta = -G_T / \beta$, and it has a
minimum value of $-\|G_T\|_2^2 / (2\beta)$. Hence, if $G_t \not\to 0$ as $T
\to \infty$, then we would have a violation of NMR. This proves the claim via
the contrapositive.
\end{proof}

In the above, we assumed convexity and smoothness for simplicity. It is 
possible that connections could be drawn in greater generality.  

% \section{Average iterate convergence}
% \label{app:avg_iterate_conv}
%
% TODO

\end{document}